\pdfoutput=1
\documentclass[11pt]{article}
\usepackage[margin=1.1in]{geometry}
\usepackage{amsmath,amssymb,amsthm}
\usepackage{booktabs}
\usepackage{xcolor}
\usepackage[round]{natbib}
\usepackage{hyperref}
\usepackage{bookmark}
\usepackage{url}
\hypersetup{
  colorlinks=true,linkcolor=blue!60!black,citecolor=blue!60!black,urlcolor=blue!60!black,
  pdftitle={Exact Risk Ratios for Weighted Data Selection in Linear Regression},
  pdfauthor={Guangjian Zhang}
}

\newcommand{\R}{\mathbb{R}}
\newcommand{\pos}{\operatorname{pos}}
\newcommand{\spn}{\operatorname{span}}
\newcommand{\supp}{\operatorname{supp}}
\newcommand{\rk}{\operatorname{rank}}
\newcommand{\ip}[2]{\langle #1,#2\rangle}
\newcommand{\norm}[1]{\lVert #1\rVert}
\newcommand{\Fw}{F_{\mathrm{w}}}
\newcommand{\Astar}{A^{\star}}
\newcommand{\Ls}{L_D^{\star}}
\newcommand{\Gam}{\Gamma}

\theoremstyle{plain}
\newtheorem{theorem}{Theorem}[section]
\newtheorem{lemma}[theorem]{Lemma}
\newtheorem{proposition}[theorem]{Proposition}
\newtheorem{corollary}[theorem]{Corollary}
\newtheorem{conjecture}[theorem]{Conjecture}
\theoremstyle{definition}
\newtheorem{definition}[theorem]{Definition}
\theoremstyle{remark}
\newtheorem{remark}[theorem]{Remark}
\numberwithin{equation}{section}

\title{Exact Risk Ratios for Weighted\\Data Selection in Linear Regression\thanks{The
vector-valued version of the problem is treated in our companion paper \citep{ZhangVector},
arXiv:2608.30254, by the same author.}}

\author{Guangjian Zhang\\ University of New South Wales\\ \texttt{zgj1226029469@outlook.com}}
\date{October 2, 2026}

\begin{document}

\maketitle

\begin{center}\small
\textbf{arXiv.org perpetual, non-exclusive license}\\[3pt]
\begin{minipage}{0.88\textwidth}\centering
Copyright \textcopyright\ 2026 Guangjian Zhang. Distributed by arXiv.org under
the arXiv.org perpetual, non-exclusive license,
\url{https://arxiv.org/licenses/nonexclusive-distrib/1.0/}.
\end{minipage}\\[5pt]
arXiv:2608.28007\quad DOI: \url{https://doi.org/10.48550/arXiv.2608.28007}
\end{center}

\begin{abstract}
How much data must a fixed learner retain? Hanneke, Moran,
Shlimovich and Yehudayoff (COLT 2025) posed this question for linear
regression with the minimum-norm empirical risk minimizer. A selector sees a
finite dataset $D\subseteq\R^d\times\R$, keeps at most $n$ examples with
nonnegative weights, and $\Fw(d,n)$ is the worst-case ratio between the
full-data loss of the trained predictor and the optimal loss. The value is
$\infty$ for $n<d$, $d+1$ at $n=d$ and $1$ for $n\ge2d$, and the regime
$d<n<2d$ was left open. We settle several cases. For every $d$ we
prove $\Fw(d,2d-1)=1+1/d$, which confirms a claim stated without proof in the
original note. We also prove $\Fw(3,4)=5/3$, $\Fw(4,5)=2$ and $\Fw(4,6)=3/2$,
the three smallest cells not covered by that formula. For every intermediate
budget $n=d+k$ we prove the lower bound $\Fw(d,d+k)\ge1+\Gam_{d,k}$, where
$\Gam_{d,k}$ is an explicit harmonic quantity over balanced partitions of
$d$. This bound is the exact minimax value on the class of
datasets whose whitened systems split into orthogonal circuit blocks.
All proved values equal $1+\Gam_{d,k}$, and we conjecture that this holds
throughout the open regime. Our upper bounds combine a rigidity theorem
for positive spanning configurations of loss gradients with normal forms
of the small positive bases in $\R^3$ and $\R^4$. These forms are special
cases of the classification of Cornaz, Kerleau and Royer; we also control
all gradients outside the basis. A dimension-free extremal-basis argument
converts sign-cone geometry into selections of $d+1$ points. Explicit
counterexamples rule out several shorter routes. Every upper bound is
constructive, with selection procedures polynomial in the number of points
for fixed dimension. Every numerical claim about a specific instance is
an exact rational or algebraic identity, recomputed in exact arithmetic in
the supplementary material.
\end{abstract}

\section{Introduction}\label{sec:intro-main}

Choosing what to train on is a fundamental question in modern learning.
Coresets summarize
large datasets for regression and clustering
\citep{FeldmanLangberg2011,Phillips2017,BDM2013}, data pruning discards
examples before training deep networks \citep{Paul2021,Sorscher2022}, and
active regression pays for only a few labels \citep{ChenPrice2019,MMWY2022}.
Most guarantees in this literature have the form $1+\varepsilon$ with
budgets that grow like $d/\varepsilon$ up to logarithmic factors.
\citet{HMSYerm} asked a
sharper worst-case question. For a fixed learning rule and a fixed budget
$n$, what is the best multiplicative loss achievable on the full dataset,
uniformly over all datasets? The selector sees the whole dataset, labels
included, so the object is the exact compression limit of a fixed learner
at a fixed budget rather than a label-efficient algorithm. \citet{HMSYopen} distilled the open cases into
concrete problems. We address their Question~2, weighted data selection for
linear regression. The aim is to identify the precise worst-case loss that
even an optimal selector cannot avoid. This benchmark separates budget limits from the shortcomings of a selection rule.

The setting is simple. A dataset $D$ is a finite multiset in
$\R^d\times\R$ with squared losses $\ell_{(x,y)}(w)=(\ip{w}{x}-y)^2$. A
selector sees all of $D$, picks at most $n$ examples and a convex
combination $F$ of their losses, and hands $F$ to the min-norm ERM $\Astar$.
The learner returns the minimizer of $F$ of smallest Euclidean norm. We study
the worst-case ratio $\Fw(d,n)$ of the returned predictor's full-data loss
to the optimal full-data loss
(Section~\ref{sec:setting}). \citet{HMSYerm} proved a sharp trichotomy:
$\Fw(d,n)$ equals $\infty$ below budget $d$, equals $d+1$ at budget $d$, and
equals $1$ from budget $2d$ onward, where selection loses nothing at all.
The endpoint arguments use different geometries: volume sampling at
$n=d$ and Steinitz's interior-point theorem at $n=2d$. The whole regime
$d<n<2d$ between them was left open. Its value was known only to lie in
$(1,d+1]$ and to be $O(d)$. This gap leaves the transition from severe
compression to lossless selection unresolved. The note adds that for $n=2d-1$ the authors
believe the value is $1+\frac1d$, ``though the argument is somewhat ad hoc
and does not clearly extend to other values of $n$'' \citep[p.~3]{HMSYopen}.

\paragraph{Contributions.} Write the intermediate budget as $n=d+k$ with
$1\le k\le d-1$.
\begin{itemize}\setlength{\itemsep}{1pt}
\item \textbf{Three exact cells.} $\Fw(3,4)=\frac53$, $\Fw(4,5)=2$
and $\Fw(4,6)=\frac32$ (Theorem~\ref{thm:main-cells}). The first value
refutes the natural interpolation guess $1+(2d-n)^2/d$, which fits every
previously known value and predicts $\frac73$.
\item \textbf{Structural contributions.} We prove a rigidity theorem for gradient
configurations one point below the Steinitz budget. We give normal forms for the
five-vector positive bases of $\R^3$ and the minimal positive bases of
$\R^4$ of sizes $6$ and $7$. These are the cases $\ell=3,4$ of the
characterization of positive bases of sizes $2\ell-1$ and $\ell+2$ by
\citet{CKR2024}. We give self-contained proofs and control all gradients
outside the basis, a step specific to the risk problem. A dimension-free
extremal-basis principle converts sign-cone geometry into selections.
Our finite toolkit comprises Gale coordinates, an adjacent-basis lemma
that makes weight optimization at budget $d+1$ a finite computation, a
composition law for direct sums, and a two-circuit theorem for bases of size $d+2$.
\item \textbf{The top of the open regime, for every dimension.}
$\Fw(d,2d-1)=1+\frac1d$ for all $d\ge1$ (Theorem~\ref{thm:main-endpoint}).
An independent proof appears in the preprint of \citealp{Dewa2026}. This
includes $\Fw(2,3)=\frac32$, the smallest open cell, and proves the value
that \citet{HMSYopen} state they believe but do not prove.
\item \textbf{A lower bound in every cell, and an exact model class.} For
all $1\le k\le d-1$, $\Fw(d,d+k)\ge1+\Gam_{d,k}$ with $\Gam_{d,k}$ the
harmonic quantity \eqref{eq:gamma-main} (Theorem~\ref{thm:main-lb}). On
datasets whose whitened systems split into orthogonal circuit blocks the
worst-case excess at budget $d+k$ is exactly $\Gam_{d,k}$, as the value of an
explicit budget-allocation program (Theorem~\ref{thm:main-block}). Both
results also appear in \citet{Dewa2026}, the second on a closely related class.
All four proved values equal $1+\Gam_{d,k}$, and we conjecture that
equality holds throughout the open regime (Conjecture~\ref{conj:gamma}).
\item \textbf{Obstructions.} Explicit rational counterexamples show that
three shorter routes fail at the cell $(4,5)$, and that in $\R^6$ a
connected coupled system has budget-$(d+1)$ excess above $1$ for every
strictly convex selection, so no mechanism certifying excess at most $1$
with $d+1$ points covers every coupled configuration
(Section~\ref{sec:obstructions-main}).
\item \textbf{Algorithms and verification.} All upper bounds are
constructive: for fixed dimension, a selection attaining the bound, or the
bound plus any $\varepsilon>0$, is computed in polynomially many exact
arithmetic operations. Every numerical claim about a specific instance
is recomputed in exact rational or algebraic arithmetic by scripts in the
supplementary material, kept separate from floating-point search
diagnostics (Section~\ref{sec:algo-main}).
\end{itemize}

\paragraph{Proof architecture.} After whitening, the full-data
optimum sits at $q=0$, the risk ratio is $1+\norm{q}^2$, and the loss
gradients at the optimum sum to zero. If some $2d-1$ gradients positively
span $\R^d$, a strictly positive zero-sum weighting turns them into a
strictly convex objective whose exact minimizer is $q=0$: selection is free.
The whole difficulty is the opposite case. At budget $2d-1$ a rigidity lemma
shows that this case is rigid: all gradients lie on $d$ lines.
The covariance identity forces the lines orthogonal with unit energy each.
The problem reduces to $d$ one-dimensional regressions, where leaving
out one point on the lightest line costs exactly the harmonic factor
$\frac1d$. At smaller budgets the failure of positive spanning no longer
forces an axis configuration. We instead use the classification of small
positive bases. Minimal positively spanning sets of gradients decompose
into few circuits. Minimal cone representations control extra gradients
through low-rank circuits. Each resulting geometry yields an exact
certificate, splits into orthogonal blocks, or admits completion through
quotient volume sampling. The hardest configuration in $\R^4$ consists
of two coupled triangles with two independent cross ears. We call it the
rectangle. It resists the partition arguments available to us, but an
extremal-basis lemma converts its sign-cone geometry into a five-point selection.

\paragraph{Reading guide.} Results: \S\ref{sec:setting}; tools/proofs:
\S\S\ref{sec:prelim-main}--\ref{sec:techniques-main}; related work/discussion:
\S\S\ref{sec:related-main}--\ref{sec:outlook-main}; appendices:
\ref{sec:prelim}--\ref{sec:extended-outlook}.

\section{Setting and main results}\label{sec:setting}\label{sec:results-main}
\bookmark[level=2,dest=\getrefbykeydefault{thm:main-cells}{anchor}{Doc-Start}]{Theorem \ref*{thm:main-cells}}
\bookmark[level=3,dest=\getrefbykeydefault{thm:fw34}{anchor}{Doc-Start}]{Proof location: Theorem \ref*{thm:fw34}}
\bookmark[level=3,dest=\getrefbykeydefault{thm:fw45}{anchor}{Doc-Start}]{Proof location: Theorem \ref*{thm:fw45}}
\bookmark[level=3,dest=\getrefbykeydefault{thm:fw46}{anchor}{Doc-Start}]{Proof location: Theorem \ref*{thm:fw46}}
\bookmark[level=2,dest=\getrefbykeydefault{thm:main-endpoint}{anchor}{Doc-Start}]{Theorem \ref*{thm:main-endpoint}}
\bookmark[level=3,dest=\getrefbykeydefault{cor:endpoint}{anchor}{Doc-Start}]{Corollary \ref*{cor:endpoint}}
\bookmark[level=3,dest=\getrefbykeydefault{thm:wend}{anchor}{Doc-Start}]{Theorem \ref*{thm:wend}}
\bookmark[level=3,dest=\getrefbykeydefault{thm:endpoint-ub}{anchor}{Doc-Start}]{Theorem \ref*{thm:endpoint-ub}}
\bookmark[level=3,dest=\getrefbykeydefault{thm:endpoint-lb}{anchor}{Doc-Start}]{Theorem \ref*{thm:endpoint-lb}}
\bookmark[level=2,dest=\getrefbykeydefault{thm:main-lb}{anchor}{Doc-Start}]{Theorem \ref*{thm:main-lb}}
\bookmark[level=3,dest=\getrefbykeydefault{thm:gammalb}{anchor}{Doc-Start}]{Theorem \ref*{thm:gammalb}}
\bookmark[level=3,dest=\getrefbykeydefault{prop:blocks}{anchor}{Doc-Start}]{Proposition \ref*{prop:blocks}}

A dataset $D=\{z_i=(x_i,y_i)\}_{i=1}^{N}$ is a finite multiset in
$\R^d\times\R$. For $w\in\R^d$ let $\ell_z(w)=(\ip{w}{x}-y)^2$,
$L_D(w)=\frac1N\sum_i\ell_{z_i}(w)$ and $\Ls=\min_wL_D(w)$. The learner
$\Astar$ is the min-norm ERM: for a nonnegatively weighted objective $F$,
$\arg\min F$ is a translate of a linear subspace and $\Astar(F)$ is its
element of smallest Euclidean norm. Weighted selection with budget $n$ means
choosing $z_1,\dots,z_n\in D$ and $F\in\mathrm{conv}(\ell_{z_1},\dots,\ell_{z_n})$, and
\[
\Ls(n;\Astar)=\inf_{z_1,\dots,z_n,\,F}L_D\bigl(\Astar(F)\bigr),\qquad
\Fw(d,n)=\sup_D\frac{\Ls(n;\Astar)}{\Ls},
\]
with $0/0=1$ and $c/0=\infty$ for $c>0$, as in Theorem~8 of
\citet{HMSYerm}. Repetitions and zero weights are allowed, so selections
are counted by effective support. \citet{HMSYerm} proved $\Fw(d,n)=\infty$
for $n<d$, $\Fw(d,d)=d+1$ and $\Fw(d,n)=1$ for $n\ge2d$
(Appendix~\ref{sec:prelim}).

\paragraph{The cells \texorpdfstring{$(3,4)$, $(4,5)$ and $(4,6)$}{(3,4), (4,5) and (4,6)}.}

\begin{theorem}[Theorems~\ref{thm:fw34}, \ref{thm:fw45} and~\ref{thm:fw46}]\label{thm:main-cells}
$\Fw(3,4)=\frac53$, $\Fw(4,5)=2$ and $\Fw(4,6)=\frac32$.
\end{theorem}

The lower bounds are Theorem~\ref{thm:main-lb} at $(d,k)=(3,1)$, $(4,1)$
and $(4,2)$. The upper bounds are the subject of
Section~\ref{sec:techniques-main}. Table~\ref{tab:values} collects the
values for $d\le6$.

\paragraph{The budget \texorpdfstring{$2d-1$}{2d-1}.}

\begin{theorem}[Corollary~\ref{cor:endpoint}]\label{thm:main-endpoint}
For every $d\ge1$, $\Fw(d,2d-1)=1+\frac1d$. In particular $\Fw(2,3)=\frac32$.
\end{theorem}

\paragraph{General lower bound.}

For $s\in\{k+1,\dots,d\}$ let $C_{d,s}$ be the inverse of the minimum of
$\sum_{j=1}^s1/r_j$ over positive integer partitions $r_1+\dots+r_s=d$; the
minimum is attained by the balanced partition, so with $d=q_ss+a_s$,
$0\le a_s<s$,
\begin{equation}\label{eq:gamma-main}
C_{d,s}=\Bigl(\frac{s-a_s}{q_s}+\frac{a_s}{q_s+1}\Bigr)^{-1},\qquad
\Gam_{d,k}=\max_{k+1\le s\le d}(s-k)\,C_{d,s}.
\end{equation}
The endpoints degenerate to $\Gam_{d,0}=d$ and $\Gam_{d,d-1}=\frac1d$,
matching $\Fw(d,d)=d+1$ and Theorem~\ref{thm:main-endpoint}.

\begin{theorem}[Theorem~\ref{thm:gammalb}]\label{thm:main-lb}
For all $1\le k\le d-1$, $\Fw(d,d+k)\ge1+\Gam_{d,k}$. The bound is attained
by explicit datasets: decompose $\R^d$ orthogonally into blocks $E_j$ of
dimensions $r_j$, place in $E_j$ the $r_j+1$ vertices $-\sum_me_{j,m}$,
$e_{j,1},\dots,e_{j,r_j}$ with equal residuals $a_j=\sqrt{NC/(r_j(r_j+1))}$,
$C=(\sum_j1/r_j)^{-1}$, and shift the optimum by a large multiple $M$ of
$(1,2,\dots,r_j)$ in each block. Then $\Ls=1$ and
$\Ls(d+k;\Astar)=1+(s-k)C$ (Proposition~\ref{prop:blocks}).
\end{theorem}

\begin{table}[htbp]
\caption{Values of $1+\Gam_{d,k}$ at budget $n=d+k$ for $d\le6$. Bold
entries are proved values of $\Fw(d,d+k)$ (Theorems~\ref{thm:main-endpoint}
and~\ref{thm:main-cells}); the others are lower bounds and conjectured
values.}
\label{tab:values}
\begin{center}
\begin{tabular}{c|ccccc}
\toprule
$d\backslash k$&1&2&3&4&5\\
\midrule
2&$\mathbf{3/2}$&&&&\\
3&$\mathbf{5/3}$&$\mathbf{4/3}$&&&\\
4&$\mathbf{2}$&$\mathbf{3/2}$&$\mathbf{5/4}$&&\\
5&$11/5$&$8/5$&$7/5$&$\mathbf{6/5}$&\\
6&$5/2$&$5/3$&$3/2$&$4/3$&$\mathbf{7/6}$\\
\bottomrule
\end{tabular}
\end{center}
\end{table}

\section{Preliminaries and tools}\label{sec:prelim-main}

\paragraph{Whitening and gradients.} Suppose the features span $\R^d$ and
$\Ls>0$. Let $H=\frac1N\sum_ix_ix_i^{\top}$, let $w^{\star}$ be the optimum
and $r_i=y_i-\ip{w^{\star}}{x_i}$, and set $\xi_i=H^{-1/2}x_i$,
$\eta_i=r_i/\sqrt{\Ls}$ and $q=H^{1/2}(w-w^{\star})/\sqrt{\Ls}$. Then
(Lemma~\ref{lem:whiten})
\begin{equation}\label{eq:threeid-main}
\tfrac1N\textstyle\sum_i\xi_i\xi_i^{\top}=I_d,\qquad
\tfrac1N\sum_i\eta_i\xi_i=0,\qquad
\tfrac1N\sum_i\eta_i^2=1,\qquad
L_D(w)/\Ls=1+\norm{q}^2 .
\end{equation}
A family satisfying the three identities is a \emph{whitened system}; the
\emph{gradients} $g_i=\eta_i\xi_i$ are, up to the common factor $-2$, the
loss gradients at the optimum, and $\sum_ig_i=0$. Whitening is not orthogonal
and does not commute with the min-norm tie rule. We therefore use whitened
coordinates only to compute the ratio and build \emph{strictly convex}
objectives, whose unique minimizer never invokes that rule. All branches
that could produce rank-deficient objectives are handled in the original
coordinates or through finite penalty approximations
(Lemma~\ref{lem:penalty}). Two facts drive everything.

\begin{lemma}[exact certificate, Lemma~\ref{lem:cert}]\label{lem:cert-main}
If indices $S$ carry weights $\alpha_i>0$ with $\sum_{i\in S}\alpha_ig_i=0$
and $\spn\{\xi_i:i\in S\}=\R^d$, then $q=0$ is the unique minimizer of
$\sum_{i\in S}\alpha_i(\ip{q}{\xi_i}-\eta_i)^2$: the selection has ratio
exactly $1$.
\end{lemma}

\begin{lemma}[volume sampling interface, Lemma~\ref{lem:L2}]\label{lem:L2-main}
Let $\{(\zeta_i,\eta_i)\}\subseteq\R^r\times\R$ have features spanning
$\R^r$, with $\rho(u)=\frac1M\sum_i(\ip{u}{\zeta_i}-\eta_i)^2$ and minimum
$\rho^{\star}$. Then some $r$ indices with independent features have
interpolation point $\hat u$ with $\rho(\hat u)\le(r+1)\rho^{\star}$
\citep{DW17}.
\end{lemma}

\paragraph{The tie rule matters.} A whitened system is itself a legal
dataset: take $x_i=\xi_i$, $y_i=\eta_i$. On this \emph{self-realization},
the Euclidean min-norm rule coincides with the whitened one and the
optimum is the origin. The tie rule then helps the selector: unconstrained
directions are resolved to $w^{\star}=0$ for free. All upper bounds below
use strictly convex selections, after reducing to the feature span if the
features are rank deficient. These bounds hold for every such dataset.
All lower bounds use original-coordinate instances whose large parameter
shift makes rank-deficient selections costly.

\section{Proofs and constructions}\label{sec:constructions-main}
\bookmark[level=2,dest=reader:proof-endpoint]{Proof of Theorem \ref*{thm:main-endpoint}}
\bookmark[level=2,dest=reader:construction-lb]{Construction for Theorem \ref*{thm:main-lb}}
\bookmark[level=2,dest=\getrefbykeydefault{thm:main-block}{anchor}{Doc-Start}]{Theorem \ref*{thm:main-block}}
\bookmark[level=3,dest=\getrefbykeydefault{thm:lp}{anchor}{Doc-Start}]{Theorem \ref*{thm:lp}}
\bookmark[level=3,dest=\getrefbykeydefault{cor:lp-gamma}{anchor}{Doc-Start}]{Corollary \ref*{cor:lp-gamma}}

\paragraph{Positive spanning sets.} $B\subseteq\R^m\setminus\{0\}$
positively spans $\R^m$ if $\pos(B)=\R^m$. It is a positive basis if no
proper subset positively spans. Then $m+1\le|B|\le2m$, with $|B|=2m$
exactly for $B=\{a_j,-\lambda_ja_j\}_{j\le m}$ over a linear basis
\citep{Davis1954,McKinney1962}. A positive circuit is a set carrying a
strictly positive linear dependence, minimal with respect to support. A
positive circuit of size $m+1$ has rank $m$. The nonclassical ingredient is
the following rigidity statement.

\begin{lemma}[rigidity, Lemma~\ref{lem:rig}]\label{lem:rig-main}
Let $G\subseteq\R^m\setminus\{0\}$ positively span $\R^m$ and suppose no
subset of $G$ of size at most $2m-1$ positively spans. Then there is a
linear basis $a_1,\dots,a_m$ such that every element of $G$ lies on one of
the lines $\R a_j$.
\end{lemma}

\begin{proof}
Extract from $G$ an inclusion-minimal positively spanning subset $B$. By
hypothesis $|B|\ge2m$ and by the size bound $|B|\le2m$, so after positive
rescaling $B=\{\pm a_1,\dots,\pm a_m\}$ for a linear basis $(a_j)$. Take
$g\in G$ and write $g=\sum_{j\in S}c_ja_j$ with all $c_j\ne0$. Suppose
$|S|\ge2$. Put $o_j=\mathrm{sgn}(c_j)\,a_j$ and $r_j=-o_j$ for $j\in S$, and
let $T=\{g\}\cup\{r_j:j\in S\}\cup\{\pm a_j:j\notin S\}$, which has
$2m-|S|+1\le2m-1$ elements, all in $G$ up to positive scaling. For each
$j\in S$, $|c_j|\,o_j=g+\sum_{h\in S\setminus\{j\}}|c_h|\,r_h\in\pos(T)$,
so $\pos(T)$ contains every $\pm a_j$ and equals $\R^m$, contradicting the
hypothesis. Hence $|S|=1$.
\end{proof}

\hypertarget{reader:proof-endpoint}{}\paragraph{Proof of Theorem~\ref{thm:main-endpoint}.}

\emph{Upper bound} (Theorems~\ref{thm:wend} and~\ref{thm:endpoint-ub}).
Rank-deficient and realizable datasets have ratio $1$; otherwise whiten. If
the nonzero gradients span a proper subspace $V$ of dimension $r$, Steinitz's
theorem inside $V$ gives at most $2r$ gradients with a strictly positive
zero-sum weighting. A further $d-r$ zero-gradient points complete the rank,
so Lemma~\ref{lem:cert-main} recovers $q=0$ with $d+r\le2d-1$ points. If the
nonzero gradients span $\R^d$, they positively span it. If some $2d-1$ of them positively
span, Lemma~\ref{lem:cert-main} again recovers $q=0$. Otherwise
Lemma~\ref{lem:rig-main} puts every nonzero gradient on $d$ lines. A
zero-gradient point with nonzero feature yields a certificate with $2d-1$
points. If there is none, every nonzero feature also lies on the lines.
The covariance identity $I_d=\sum_jQ_ju_ju_j^{\top}$ then forces the lines
orthonormal with unit feature energy $Q_j=1$. The objective then separates
across coordinates. The residual energies $R_j$ of the lines sum to at most
$1$, so some line has $R_{j_\ast}\le\frac1d$, and on it some point has
squared offset $(\eta_i/t_i)^2\le R_{j_\ast}$ by weighted averaging. Select
that point and a gradient-canceling pair on each other line. This uses
$1+2(d-1)=2d-1$ points and $\norm{\hat q}^2\le\frac1d$.

\emph{Lower bound} (Theorem~\ref{thm:endpoint-lb}). Fix $c\ge1$ and let
$D_{d,c}=\{(e_j,c+1),(-e_j,1-c):j\le d\}$. Then $w^{\star}=c\mathbf1$,
$\Ls=1$ and $L_D(w)=1+\frac1d\norm{w-c\mathbf1}^2$. A selection of at most
$2d-1$ points misses one of the $2d$ point types, say in coordinate $j$. The
objective separates. In coordinate $j$ the output is $c\pm1$ if one
point is kept, or $0$ by the min-norm rule if none is. Thus some coordinate
deviates by at least $1$ and the ratio is at least $1+\frac1d$. Keeping all
points but $(-e_d,1-c)$ attains it.

\hypertarget{reader:construction-lb}{}\paragraph{Construction for Theorem~\ref{thm:main-lb}.}

The shift is what makes the min-norm rule costly: a block that loses two
vertices leaves a direction to the tie rule, and its excess grows like
$M^2$. An optimal selection therefore keeps at least $r_j$ vertices per block.
Each broken block costs exactly $C$. For $(d,n)=(3,4)$ the instance is, with $\gamma=\sqrt5/3$ and
$\delta=\sqrt{5/3}$,
\[
\begin{gathered}
z_1=\bigl((-1,-1,0),-30+\gamma\bigr),\quad
z_2=\bigl((1,0,0),10+\gamma\bigr),\quad
z_3=\bigl((0,1,0),20+\gamma\bigr),\\
z_4=\bigl((0,0,-1),-10+\delta\bigr),\quad
z_5=\bigl((0,0,1),10+\delta\bigr),
\end{gathered}
\]
with $w^{\star}=(10,20,10)$, $\Ls=1$ and $\Ls(4;\Astar)=\frac53$, attained
by $\{z_1,z_2,z_3,z_5\}$ with weights $(\frac16,\frac16,\frac16,\frac12)$.
Two two-dimensional blocks give $\Ls(5;\Astar)=2$ for $(4,5)$, and blocks of
dimensions $2$ and $3$ give $\frac{11}5$ for $(5,6)$; all three are verified
in exact arithmetic (Appendix~\ref{sec:lb}).

\begin{definition}[Definition~\ref{def:block}]
A whitened system has \emph{orthogonal circuit-block structure} with
dimension vector $r=(r_1,\dots,r_s)$ if $\R^d=E_1\oplus^{\perp}\dots
\oplus^{\perp}E_s$ with $\dim E_j=r_j$, every nonzero feature lies in some
$E_j$, and every $E_j$ contains a positive circuit of $r_j+1$ gradients
spanning $E_j$. Let $R_j$ be the residual energy of block $j$, so
$\sum_jR_j\le1$, and let $\mathcal D_r$ be the class of full-rank datasets
with $\Ls>0$ whose whitened system has this structure.
\end{definition}

\begin{theorem}[Theorem~\ref{thm:lp} and Corollary~\ref{cor:lp-gamma}]\label{thm:main-block}
Let $D\in\mathcal D_r$ have block energies $(R_j)$ and let $n=d+k$ with
$0\le k\le s-1$. There is a strictly convex selection of at most $d+k$
points with $\norm{\hat q}^2\le\min_{|K|=k}\sum_{j\notin K}r_jR_j$, and the
supremum of the right-hand side over admissible energies is exactly
\[
\Phi(r,k)=\max_{J\subseteq[s],\,|J|\ge k+1}\ \frac{|J|-k}{\sum_{j\in J}1/r_j}.
\]
Moreover $\max_r\Phi(r,k)=\Gam_{d,k}$ over dimension vectors with $s\ge k+1$
parts, so $\sup_r\sup_{D\in\mathcal D_r}\bigl(\Ls(d+k;\Astar)/\Ls-1\bigr)
=\Gam_{d,k}$, attained by the balanced datasets of
Theorem~\ref{thm:main-lb}.
\end{theorem}

The selection completes the $k$ blocks of largest $r_jR_j$ with their
circuits and applies Lemma~\ref{lem:L2-main} inside each remaining block.
With $x_j=r_jR_j$ the adversary's program is
$\sup\{\min_{|K|=k}\sum_{j\notin K}x_j:x\ge0,\ \sum_jx_j/r_j\le1\}$. Writing
$J_t=\{j:x_j>t\}$, the layer-cake identities
\[
\min_{|K|=k}\sum_{j\notin K}x_j=\int_0^{\infty}\bigl(|J_t|-k\bigr)_{+}\,dt,
\qquad
\sum_j\frac{x_j}{r_j}=\int_0^{\infty}\sum_{j\in J_t}\frac1{r_j}\,dt\le1
\]
and the definition of $\Phi$, which bounds the first integrand by
$\Phi(r,k)$ times the second whenever $|J_t|\ge k+1$, give the upper bound.
For the lower bound, put mass $C_J=(\sum_{j\in J}1/r_j)^{-1}$ on each block of
the maximizing $J$ and zero elsewhere.

\section{Structural results}\label{sec:techniques-main}

\paragraph{Interfaces, anchors and penalty limits.\pdfbookmark[2]{Interfaces, anchors and penalty limits}{reader:interfaces}\bookmark[level=3,dest=\getrefbykeydefault{thm:L13D}{anchor}{Doc-Start}]{Location: Theorem \ref*{thm:L13D}}}

The endpoint bound for whitened systems
(Theorem~\ref{thm:wend}) is used in later cells in dimensions $2$ and
$3$. To handle a point with $\eta_0=0\ne\xi_0$, an \emph{anchor}, project
the system to $\xi_0^{\perp}$. The projected system is again whitened.
Apply the lower-dimensional interface there, then add the anchor's own
loss with a large finite weight. The penalty limit
(Lemma~\ref{lem:penalty}) makes the output converge to the projected output.
Since $\Ls(n;\Astar)$ is an infimum, this suffices. The construction costs one point
and gives ratio $\frac32+\varepsilon$ in $\R^3$ and $\frac53+\varepsilon$ in
$\R^4$ (Lemma~\ref{lem:anchor3}, Corollary~\ref{cor:anchor4}). In every
remaining branch we may assume $g_i=0\Rightarrow\xi_i=0$: nonzero features
are collinear with their gradients, and the nonzero gradients positively
span $\R^d$. Let $B$ be a minimum-cardinality positively spanning subset of
them; $d+1\le|B|\le2d$. The case $|B|=d+1$ is a full-rank circuit and gives
exact recovery. The case $|B|=2d$ gives orthonormal lines by
Lemma~\ref{lem:rig-main}. The work lies between these two cases.

\paragraph{Dimension three.\pdfbookmark[2]{Dimension three}{reader:three}}

\emph{Proof: Theorem~\ref{thm:fw34}; selected references:}
Theorem~\ref{thm:class5}; Propositions~\ref{prop:coupled},
\ref{prop:split}, \ref{prop:sixaxis}.

\begin{theorem}[Theorem~\ref{thm:class5}]\label{thm:class5-main}
Every positive basis $\{b_1,\dots,b_5\}$ of $\R^3$ is, up to renumbering,
positive rescaling and an invertible linear map,
$b_1=e_1$, $b_2=e_2$, $b_3=-e_1-e_2$, $b_4=e_3-\frac\kappa2e_1$,
$b_5=-e_3-\frac\kappa2e_1$ for some $\kappa\ge0$; equivalently
$b_1+b_2+b_3=0$ and $\kappa b_1+b_4+b_5=0$. For $\kappa=0$ the basis splits
into a triangle circuit and an antipodal pair; all $\kappa>0$ form one
orbit, two triangles sharing $b_1$.
\end{theorem}

This is the case $\ell=3$ of the characterization of positive bases of size
$2\ell-1$ by \citet{CKR2024}; we use the explicit constant $\kappa$. The
dependence cone of $B$ is two-dimensional with two extreme rays, each a
circuit of rank at most $2$ by minimality. Their supports cover the five
indices. In the \emph{coupled} case every extra gradient has a minimal cone
representation of support at most $2$ over $B$: a support of size
$3$ would create a rank-$3$ four-point circuit and hence an exact
certificate. Two cross types cannot coexist (Lemma~\ref{lem:onecross}), so
all nonzero gradients lie in at most three planes, each carrying a triangle
circuit (Lemma~\ref{lem:cover}). Some plane then holds residual energy at
least $\frac13$, and pairing its triangle with the cheapest point off the
plane gives $\norm{\hat q}^2\le1-R(P)\le\frac23$ with four points
(Lemma~\ref{lem:plane}). In the \emph{split} case, once every gradient
outside the triangle plane and the pair line has been reduced to an exact
certificate or to the coupled case, the covariance identity forces the pair
line orthogonal to the triangle plane with unit energy each.
Compare two selections: the triangle plus the best point on the line, or
the canceling pair plus a volume-sampling pair in the plane. The better costs at most
$\min\{R_L,2R_V\}\le\frac23$ (Proposition~\ref{prop:split}). The six-axis
case is Lemma~\ref{lem:rig-main} again. This proves $\Fw(3,4)\le\frac53$
(Theorem~\ref{thm:fw34}).

\paragraph{Dimension four.\pdfbookmark[2]{Dimension four}{reader:four}}

\emph{Proof: Theorem~\ref{thm:fw45}; selected references:}
Propositions~\ref{prop:B6}, \ref{prop:B7}, \ref{prop:B8}, \ref{prop:rect}.

Two general lemmas replace the plane-energy argument.

\begin{lemma}[circuit-flat completion, Lemma~\ref{lem:completion}]\label{lem:completion-main}
Let a whitened system in $\R^d$ contain a positive circuit $C$ of gradients
with $|C|=r+1$ and flat $P=\spn\{g_i:i\in C\}$ of dimension $r$, and let
$E(P)=\frac1N\sum_{i:\Pi_{P^{\perp}}\xi_i\ne0}\eta_i^2$. Then a strictly
convex selection of at most $d+1$ points has
$\norm{\hat q}^2\le(d-r)E(P)$.
\end{lemma}

Select $C$ with its zero-sum weights and apply Lemma~\ref{lem:L2-main} to
the projected features; the circuit keeps stationarity at any
$u\in P^{\perp}$. A companion \emph{circuit-star} lemma
(Lemma~\ref{lem:star}) handles the case where all projected features lie on
two lines of $P^{\perp}$ with $\norm{\hat q}^2\le1$.

For $\Fw(4,5)$ the basis $B$ has size $5$, $6$, $7$ or $8$. Size $6$
(Proposition~\ref{prop:B6}, the case $\ell=4$ of Theorem~5.4 of
\citet{CKR2024}) contains a rank-$3$ four-point circuit, finished
by Lemma~\ref{lem:completion-main}, or is two triangles in complementary
planes $U,V$. An extra gradient whose minimal cone support over $B$ has
size $4$ or $3$ gives exact recovery or completion, respectively. The
remaining extra gradients therefore have supports of size at most $2$.
Those inside one triangle stay in $U\cup V$; mixed supports are \emph{cross
gradients} $h+\alpha a_i+\beta c_j=0$. With no cross
gradients the system is an orthogonal block system and
Theorem~\ref{thm:main-block} gives excess at most $\min\{2R_U,2R_V\}\le1$.
If all cross types share an endpoint, projecting along a cross plane puts
everything on two lines and the circuit-star lemma finishes
(Proposition~\ref{prop:shared}). For size $7$, use the case $\ell=4$ of their
Theorem~5.3. The extreme-ray circuits of the
dependence cone of $B$ have size at most $3$ and are not all antipodal
pairs. Hence some triangle $C_0=\{p,q,u\}$ with plane $V$ exists. The
other four vectors project to a positive basis of the two-dimensional
quotient $\R^4/V$, hence to two antipodal pairs $v+\lambda r\in V$,
$w+\mu s\in V$. Minimal representations of these shifts over $C_0$ yield
the common-anchor normal form $q+p+u=0$, $v+\lambda r+\tau q=0$,
$w+\mu s+\sigma q=0$ with $\tau,\sigma\ge0$. According to which of
$\tau,\sigma$ vanish, a triangle plane $P$ is chosen such that every extra
gradient, controlled again by its minimal cone support, projects into one
of two lines of $P^{\perp}$ (Lemma~\ref{lem:sevenlines}). The
circuit-star lemma then finishes. Size $8$ is four orthonormal lines by
Lemma~\ref{lem:rig-main}, with excess at most $\frac34$. The one
remaining configuration is the \emph{rectangle}: two cross types
$h+pa_1+qc_1=0$ and $k+ra_2+sc_2=0$ attached to different rays of both
triangles. Here the partition-based bounds we know of fail
(Section~\ref{sec:obstructions-main}), and we pass to projective
coordinates $v_i=\xi_i/\eta_i$, for which the constraint
$\ip{q}{\xi_i}=\eta_i$ reads $\ip{q}{v_i}=1$.

\paragraph{Extremal basis.\pdfbookmark[2]{Extremal basis}{reader:extremal}\bookmark[level=3,dest=\getrefbykeydefault{lem:extremal}{anchor}{Doc-Start}]{Location: Lemma \ref*{lem:extremal}}}

\begin{lemma}[sign-cone realization, Lemma~\ref{lem:signcone}]\label{lem:signcone-main}
Let $v_0=\sum_{j=1}^d\kappa_jv_j$ with $\kappa_j>0$ and $v_1,\dots,v_d$
linearly independent projective points. Every $q$ with $\ip{q}{v_0}>1$ and
$\ip{q}{v_j}<1$ for all $j$ is the unique minimizer of a strictly convex,
strictly positively weighted objective on these $d+1$ points.
\end{lemma}

Stationarity forces the vector of weighted residuals to be proportional to
the signed dependence $(1,-\kappa_1,\dots,-\kappa_d)$ of
$v_0-\sum_j\kappa_jv_j=0$, and the two sign conditions make every weight
positive. Independence of the generators gives the Hessian.

\begin{lemma}[extremal basis, Lemma~\ref{lem:extremal}]\label{lem:extremal-main}
Let a whitened system in $\R^d$ with $g_i=0\Rightarrow\xi_i=0$ admit $d$
independent projective points $p_1,\dots,p_d$ such that, writing $q_B$ for
the solution of $\ip{q_B}{p_j}=1$ and $z_i=\ip{q_B}{v_i}$, every point with
$|z_i|>1$ has a representation $v_i=\sum_{j\le d}\kappa_jw_j$ with
$\kappa_j>0$ and independent projective points $w_j$ of other data points.
Then the infimum of $\norm{\hat q}^2$ over strictly convex selections of at
most $d+1$ points is at most $1$.
\end{lemma}

The covariance identity gives $\norm{q_B}^2=\frac1N\sum_i\eta_i^2z_i^2\le
L^2$ with $L=\max_i|z_i|$. If $\norm{q_B}\le1$ the basis interpolation
itself works. Otherwise anchor at a point with $|z_i|=L>1$, rescale
$q_\ast=\pm q_B/L$ so that the anchor evaluates to $1$ and all generators to
at most $1$. If $\sum_j\kappa_j=1$, all generators are interpolated. If
$\sum_j\kappa_j>1$, the strict sign cone is nonempty and
Lemma~\ref{lem:signcone-main} realizes points arbitrarily close to $q_\ast$,
with $\norm{q_\ast}^2\le1$. In the rectangle branch the hypothesis holds with
the basis of maximal points on the four side lines $\R a_1,\R a_2,\R c_1,
\R c_2$: side points have $|z|\le1$ automatically, and every off-side point
is strictly $4$-generated by explicit identities such as
$a_0=\frac1ph+\frac1rk+\frac qpc_1+\frac src_2$ (Lemma~\ref{lem:offside}).
This closes $\Fw(4,5)\le2$ (Theorem~\ref{thm:fw45}).

\paragraph{The cell \texorpdfstring{$(4,6)$}{(4,6)} and a finite toolkit.\pdfbookmark[2]{Theorem \ref*{thm:fw46}: overview}{reader:further}}

\emph{Proof: Theorem~\ref{thm:fw46}; selected references:}
Lemmas~\ref{lem:sevenlines}, \ref{lem:lift}.

With budget $d+2$ every basis of size at most $6$ is recovered exactly, so
only $|B|\in\{7,8\}$ remain. A \emph{lifting} lemma (Lemma~\ref{lem:lift})
composes a circuit of rank $r$ with any strictly convex selection on the
quotient $P^{\perp}$: the quotient gradient at the quotient minimizer lies
in $P$ and is absorbed by the circuit weights. A rank-$3$ circuit plus a
canceling pair on the line $P^{\perp}$ then recovers exactly with six points.
Otherwise the size-$7$ normal form yields a triangle plane with two
orthonormal lines in its complement. Select the triangle, a canceling pair
on the heavy line, and the best point on the light line. The cost is at most
$R_2\le\frac12$. Size $8$ costs at most $R_3+R_4\le\frac12$
(Theorem~\ref{thm:fw46}).

\pdfbookmark[2]{Appendix \ref*{sec:toolkit}: overview}{reader:toolkit-overview}\textbf{Four further tools appear in Appendix~\ref{sec:toolkit}.}
\emph{Gale coordinates}
(Lemma~\ref{lem:gale}) show that the full-rank selector profile depends
only on the kernel of the gradient matrix and the residual masses.
\emph{Adjacent bases} (Lemma~\ref{lem:adjacent}): for $d+1$
points with features of full column rank, the closure of the outputs of all
strictly positively weighted objectives is exactly the convex hull of the
adjacent interpolation points $q_i$, obtained by deleting one point at a
time. Stationarity under weights $\alpha$ forces the barycentric
coordinates of the output to be proportional to $h_i^2/\alpha_i$, with $h$
spanning the kernel of the transposed feature matrix. So the weight
optimization at budget $d+1$ is a finite computation. \emph{Composition}
(Theorem~\ref{thm:alloc}): for an orthogonal direct sum of components with
budget--cost profiles $\delta_j(d_j+\ell)$, the optimal strictly convex
excess at budget $d+k$ is $\min_{\sum\ell_j\le k}\sum_jR_j\delta_j(d_j+\ell_j)$,
and generic large shifts make this the true optimum. \emph{Two-circuit
bases} (Theorem~\ref{thm:twocircuit}): if $|B|=d+2$ and every nonzero
feature lies in the union of the two circuit flats $P_1,P_2$, then at
budget $d+1$, $\norm{\hat q}^2\le C_{d,2}=\Gam_{d,1}$. Indeed the two
extreme-ray circuits cover $B$, so with $p=d-r_1$, $q=d-r_2$ one has
$p+q=d-|C_1\cap C_2|$. Lemma~\ref{lem:completion-main} on each circuit
gives $\norm{\hat q}^2\le\min\{pE(P_1),qE(P_2)\}$. Coverage gives
$E(P_1)+E(P_2)\le1$, and $\min\{pE_1,qE_2\}\le pq/(p+q)\le C_{p+q,2}\le
C_{d,2}$. The coverage hypothesis excludes exactly the points whose features
leave both flats.

\section{Related work}\label{sec:related-main}

The problem belongs to the data-selection program of \citet{HMSYerm}, which
treats mean estimation, classification and stochastic convex optimization
in the same worst-case sense. The weighted relaxation is the one studied in
the coreset literature. For least-squares regression, coresets and
subsampling schemes based on leverage scores \citep{DMM2006}, deterministic
sparsification \citep{BSS2012,BDM2013} and volume sampling
\citep{DW17,DWH2018,DerezinskiMahoney2021} give $(1+\varepsilon)$-approximate
solutions from $O(d/\varepsilon)$ rows up to logarithmic factors, and active
regression
\citep{ChenPrice2019,MMWY2022} obtains constant-factor approximations from
$O(d)$ labels. These results use budgets that grow with $d$
and $1/\varepsilon$; we seek the exact worst-case constant at a
fixed budget between $d$ and $2d$, for a fixed learner whose tie-breaking
rule matters. The budget-$d$ endpoint rests on volume sampling \citep{DW17},
and the budget-$2d$ endpoint on Steinitz's theorem \citep{Steinitz1916}.
Our intermediate results connect these two geometries, a task that the
original authors anticipated would require new ideas. The theory of
positive spanning sets and positive bases goes back to
\citet{Davis1954,McKinney1962}; see \citet{Regis2016} for a survey,
\citet{Audet2011} for a short proof of the maximal-size theorem, and
\citet{Schoch2020} for decomposition theorems. \citet{CKR2024} characterize
the positive bases of sizes $2\ell-1$ and $\ell+2$ of an $\ell$-dimensional
space in every dimension, and \citet{HJP2021} study structured intermediate
positive bases. The normal forms we use in $\R^3$ and $\R^4$ are the cases
$\ell=3,4$ of these characterizations. Our contribution there is the
analysis of the gradients outside the basis and the resulting risk bounds.
An unrefereed Zenodo preprint by \citet{Dewa2026} independently proves
the endpoint value $\Fw(d,2d-1)=1+1/d$ through the
same cross dichotomy, the lower bound $1+\Gam_{d,k}$ through the same
simplex-block construction, and the exactness of this bound on a closely
related class, rank-additive circuit decompositions in which every example
belongs to a circuit and zero-energy circuits are allowed. Up to such
zero-energy blocks that class is contained in ours. Our block class
allows arbitrary further points in each block. For example, the one-dimensional whitened
system with features $(1,1,1)$ and residuals $(-2,-1,3)\sqrt{3/14}$ lies in
our class and admits no such decomposition.
The three exact cells, the control of gradients outside a positive
basis, the extremal-basis mechanism, the counterexamples and the toolkit
of this paper do not appear there. We retain self-contained proofs of the
three overlapping results because later sections build on them. A
Zenodo research note by \citet{Yang2026} exhibits, for the cell $(2,3)$, a one-parameter
family of four-point planar datasets on which every weighted selection of
three points has risk ratio at least $3/2$. This is the $d=2$ case of the
lower bound in Theorem~\ref{thm:main-endpoint}. The note gives no upper bound.

\paragraph{Companion paper.} The vector-valued version of the problem,
Question~4 of \citet{HMSYopen}, is the subject of our companion paper
\citep{ZhangVector}, available as arXiv:2608.30254. The two papers are by
the same author.

\section{Obstructions, algorithms and discussion}\label{sec:outlook-main}

\subsection{Obstructions}\label{sec:obstructions-main}

The proof of $\Fw(4,5)=2$ runs through sign-cone geometry rather than
through three shorter routes. Each fails on an explicit rational instance
(Appendix~\ref{sec:obstructions}); all numbers below are exact.

\emph{Low-dimensional circuit covers fail.} The eight vectors $e_1,e_2,
-e_1-e_2,e_3,e_4,-e_3-e_4,-e_1-e_3,-e_2-e_4$ positively span $\R^4$, but no
five of them do. Their positive circuits are exactly four triangles. Any
cover by circuit flats needs four planes of total dimension $8$
(Proposition~\ref{prop:rectobs}), so no block decomposition exists.

\emph{Anchor disjunctions fail.} Two correct lemmas (Propositions~\ref{prop:anchor1}
and~\ref{prop:anchor2}) give ratio $2+\varepsilon$ whenever some point has
interpolation cost $t_i\le\frac15$ or some pair has cost $t_{ij}\le\frac13$.
Together with Lemma~\ref{lem:completion-main}, the statement ``some circuit
plane has residual energy at least $\frac12$, or some pair has
$t_{ij}\le\frac13$'', over all rectangle systems, would close the cell in a
few lines. It is false: a rectangle
system with masses $(399,400,401,399,400,401,100,200)/2700$
and cross directions scaled by $\frac1{100}$ has plane energies
$\frac49,\frac49,\frac13,\frac{167}{450}$, single-, double- and triple-anchor
costs at least $0.2219$, $0.4439$ and $1.1097$, and is nevertheless cheap:
a five-point sign-cone selection has exact ratio
$\frac{2452861}{2000000}$ (Proposition~\ref{prop:Rfail}).

\emph{Cheap points need not be generated.} In a one-parameter family with
$\varepsilon<\frac{\sqrt{33}-3}{12}$, every strictly $4$-generated point has
interpolation cost above $1$ and every cheap point is not generated by the
other rays at all (Proposition~\ref{prop:c2fail}). Fixing an anchor in
advance therefore cannot work. This is why Lemma~\ref{lem:extremal-main}
calibrates first and anchors last.

\emph{Connectivity alone does not bound the excess.} In $\R^6$, eight
gradients forming two circuits of ranks $3$ and $4$ that share one point,
with rational masses, have exact full-rank optima
$\widetilde\Delta_6=\frac{192048003}{80008000}$ and
$\widetilde\Delta_7=\frac{288030000}{240094007}>1$ at budgets $6$ and $7$
(Proposition~\ref{prop:connected}). These values refer to strictly convex
selections, which every mechanism produces. A generic large
shift as in Theorem~\ref{thm:main-lb} makes them the true optima. Thus no
mechanism certifying excess at most $1$ with $d+1$ points covers every
coupled configuration. A general argument must overcome these obstructions.
Block allocation, sign-cone calibration and circuit completion are the
mechanisms known to us that do. The system is no threat to
Conjecture~\ref{conj:gamma}, since Theorem~\ref{thm:twocircuit} bounds its
budget-$7$ excess by $\frac65\le\frac32=\Gam_{6,1}$.

\subsection{Algorithms and exact verification}\label{sec:algo-main}

All upper bounds are constructive: given $D$ and $\varepsilon>0$, the
procedures of Appendix~\ref{sec:algo} output a selection whose ratio is at
most the stated bound plus $\varepsilon$ in polynomially many arithmetic
operations in $N$ for fixed $d$. The $\varepsilon$ enters only through
penalty limits and sign-cone closures. A minimum-cardinality positively
spanning subset is found by testing subsets of size at most $2d$ with one
linear program each. Minimal cone supports are found by deletion, the dependence cone
of a small basis by its extreme rays, and the extremal basis by one
$4\times4$ solve. The extremal
instances of Theorem~\ref{thm:main-lb} with their attaining selections,
the partition enumeration behind Theorem~\ref{thm:main-block} for
$d\le14$, and all fractions of Section~\ref{sec:obstructions-main} are
recomputed in exact rational or algebraic arithmetic by short Python
scripts in the supplementary material. One command runs all of these
certificates and fails on any mismatch. Floating-point searches found
no dataset exceeding $1+\Gam_{d,k}$ and are packaged separately as
diagnostics.

\paragraph{Discussion.\pdfbookmark[2]{Discussion}{reader:discussion}}\label{sec:discussion-main}

\begin{remark}[Remark~\ref{rem:refute}]
The three previously known values $\Fw(d,d)=d+1$, $\Fw(d,2d-1)=1+1/d$ and
$\Fw(d,2d)=1$ are consistent with the guess $\Fw(d,n)=1+(2d-n)^2/d$, which
predicts $\Fw(3,4)=\frac73$. Theorem~\ref{thm:main-cells} refutes it at the
smallest cell where it differs from $1+\Gam_{d,k}$.
\end{remark}

\begin{conjecture}\label{conj:gamma}
$\Fw(d,d+k)=1+\Gam_{d,k}$ for all $1\le k\le d-1$.
\end{conjecture}

The conjecture holds on the full line $k=d-1$, at $(3,4)$, $(4,5)$ and
$(4,6)$, and its value is the exact minimax answer on the block class. The
smallest open cell is $(5,6)$, with conjectured value $\frac{11}5$. A
dataset of five points whose gradients form a coupled five-vector positive
basis has $\Ls(4;\Astar)<\frac32\Ls$ (Proposition~\ref{prop:sharp5}): the
configuration that carries the difficulty of $\Fw(3,4)$ is strictly cheaper
than the axis extremals. This is the simplest instance of a pattern in
every case we analyzed: coupling helps the selector, and the worst case
degenerates back to orthogonal blocks.

The discussion continues in Appendix~\ref{sec:extended-outlook}.

\paragraph{Conclusion.\pdfbookmark[2]{Conclusion}{reader:conclusion}\bookmark[level=3,dest=\getrefbykeydefault{sec:extended-outlook}{anchor}{Doc-Start}]{Location: Appendix \ref*{sec:extended-outlook}}}\label{sec:conclusion-main}

Exact risk ratios benchmark the limits of data reduction.
We determine the worst-case ratio for the min-norm
ERM on the whole line $n=2d-1$ and at the three smallest remaining cells.
We prove the lower bound $1+\Gam_{d,k}$ in every cell of the open regime
and show that it is exact on orthogonal circuit-block datasets. The values
refute the interpolation guess $1+(2d-n)^2/d$ and support
Conjecture~\ref{conj:gamma}. If true, it would make $\Gam_{d,k}$ the complete
answer to Question~2 of \citet{HMSYopen}. Our upper bounds use a rigidity
theorem, normal forms of small positive bases due to \citet{CKR2024}
combined with control of the gradients outside the basis, and a
dimension-free extremal-basis principle. The counterexamples identify
obstructions that a general argument must overcome.
All selections are computable, for fixed dimension, in polynomially many
exact arithmetic operations, and every numerical claim about a specific
instance is an exact rational or algebraic identity. The next steps are
the facet inequality $2\widetilde\Delta_3+3\widetilde\Delta_4\le6$ for the
two-plus-three composition subclass of the cell $(5,6)$, and a domination
or completion reduction of arbitrary systems to such compositions. Together,
these results give a concrete route toward the full budget--risk tradeoff.

\subsection*{Proofs and exact verification}

Every theorem, proposition and lemma in this paper has a complete proof in
the appendix. Appendix~\ref{sec:prelim} fixes conventions and proves the
whitening, certificate and rigidity lemmas. Appendices~\ref{sec:endpoint}
to~\ref{sec:fw46} prove the exact values and the lower bounds.
Appendix~\ref{sec:obstructions} derives every counterexample by hand and
states each numerical claim as an exact rational identity.
Appendix~\ref{sec:algo} records the selection procedures. The supplementary
material contains self-contained Python scripts in three groups:
exact-arithmetic certificates (\texttt{fractions.Fraction} or SymPy) for
every instance, energy, anchor cost and optimum quoted in the paper,
symbolic derivations, and floating-point diagnostics. It maps each claim
to the script that verifies it and provides a single command that runs all
exact certificates and exits with an error on any mismatch. The exact
certificates run in minutes on a laptop; the diagnostics take longer.

\bibliography{refs}
\bibliographystyle{plainnat}

\clearpage
\appendix
\noindent The appendix contains complete statements and proofs of all
results. Its sections are organized by topic; internal cross-references
between appendix sections refer to appendix labels.

\section{Preliminaries}\label{sec:prelim}

\subsection{Setting and conventions}

A dataset $D=\{z_i=(x_i,y_i)\}_{i=1}^{N}$ is a finite multiset in
$\R^d\times\R$. For a linear predictor $w\in\R^d$, define
\[
\ell_z(w)=(\ip{w}{x}-y)^2,\qquad
L_D(w)=\frac1N\sum_{i=1}^{N}\ell_{z_i}(w),\qquad
\Ls=\min_{w}L_D(w).
\]
The learner $\Astar$ is the \emph{min-norm ERM}. Given a nonnegatively
weighted objective $F$, the set $\arg\min F$ is a translate of a linear
subspace, and $\Astar(F)$ is its unique element of smallest Euclidean norm.
Weighted selection with budget $n$ means choosing $z_1,\dots,z_n\in D$ and
$F\in\mathrm{conv}(\ell_{z_1},\dots,\ell_{z_n})$. Define
\[
\Ls(n;\Astar)=\inf_{z_1,\dots,z_n,\;F}L_D\bigl(\Astar(F)\bigr),
\qquad
\Fw(d,n)=\sup_{D}\frac{\Ls(n;\Astar)}{\Ls},
\]
with the conventions $0/0=1$ and $c/0=\infty$ for $c>0$, as in
Theorem~8 of \citet{HMSYerm}. Repetitions and zero weights are allowed, so
we count selections by \emph{effective support}. We merge the weights of a
repeated point and count a zero-weight point as unselected.
An objective supported on fewer than $n$ points fills its remaining slots
with repetitions. Throughout, ``a selection of at most $n$ points'' refers to
effective support size.

\citet{HMSYopen,HMSYerm} determined three regimes. $\Fw(d,n)=\infty$ for
$n<d$, and the lower bound holds for every weakly continuous ERM.
$\Fw(d,d)=d+1$, with the upper bound via size-$d$ volume sampling
\citep{DW17}. $\Fw(d,n)=1$ for $n\ge 2d$, via Steinitz's interior version of
Carath\'eodory's theorem \citep{Steinitz1916}. For $d<n<2d$ the value is
known to lie in $(1,d+1]$ and to be $O(d)$ \citep[Section~4]{HMSYerm}, and
this regime is the subject of Question~2 in \citet{HMSYopen}. The note adds:
``In the case $n=2d-1$, we believe we have an argument showing that the value
is exactly $1+\frac1d$, though the argument is somewhat ad hoc and does not
clearly extend to other values of $n$'' \citep[p.~3]{HMSYopen}.

We use the following form of Steinitz's theorem
\citep{Steinitz1916}. If a point lies in the interior of the convex hull of
a finite set $S\subseteq\R^m$, then it lies in the interior of the convex
hull of some subset of $S$ of size at most $2m$. The interior is taken in
$\R^m$. When applying the theorem inside a subspace $V$, we first check that
the affine hull of $S$ is all of $V$. Relative interior then means
interior in the ambient space $V$. The theorem returns a subset whose
convex hull still has the point as a $V$-interior point. In particular, the
affine hull of the returned subset is again $V$.

\subsection{Whitening}

\begin{lemma}[whitening]\label{lem:whiten}
Let $H=\frac1N\sum_i x_ix_i^{\top}\succ0$ and $\Ls>0$. Let $w^{\star}$ be the
unique full-data optimum and $r_i=y_i-\ip{w^{\star}}{x_i}$. Set
\[
\xi_i=H^{-1/2}x_i,\qquad
\eta_i=\frac{r_i}{\sqrt{\Ls}},\qquad
q=\frac{H^{1/2}(w-w^{\star})}{\sqrt{\Ls}}.
\]
Then
\begin{equation}\label{eq:threeid}
\frac1N\sum_i\xi_i\xi_i^{\top}=I_d,\qquad
\frac1N\sum_i\eta_i\xi_i=0,\qquad
\frac1N\sum_i\eta_i^2=1,
\end{equation}
and for every $w$,
\begin{equation}\label{eq:ratio}
\frac{L_D(w)}{\Ls}=1+\norm{q}^2 .
\end{equation}
\end{lemma}

\begin{proof}
The first identity is $H^{-1/2}HH^{-1/2}=I_d$. The second is the normal
equation $\frac1N\sum_ir_ix_i=0$. The third is the definition of $\Ls$.
For \eqref{eq:ratio}, use $\ip{w}{x_i}-y_i=\sqrt{\Ls}\,(\ip{q}{\xi_i}-\eta_i)$
and expand the square. The cross term vanishes by the second identity.
\end{proof}

We call a finite family $\{(\xi_i,\eta_i)\}_{i=1}^{N}\subseteq\R^d\times\R$
satisfying \eqref{eq:threeid} a \emph{whitened system}. Write
$g_i=\eta_i\xi_i$. Since $\nabla_q(\ip{q}{\xi_i}-\eta_i)^2=-2g_i$ at $q=0$,
the vector $g_i$ is a positive multiple of the negative gradient of $\ell_i$
at the optimum. The common factor $-2$ is immaterial for every positive
spanning and zero-sum statement below. We also have $\sum_ig_i=0$. Whitening is not an orthogonal map, so
it does not commute with the min-norm tie rule. We therefore use whitened
coordinates only to compute the risk ratio \eqref{eq:ratio} and to build
selections whose objectives are strictly convex. These objectives have a
unique minimizer, so the tie rule never acts. All branches that could produce
rank-deficient objectives are handled separately, either in the original
coordinates or through finite penalty approximations.

A whitened system is itself a legal dataset: take $x_i=\xi_i$, $y_i=\eta_i$.
Then $H=I_d$, $w^{\star}=0$ and $\Ls=1$, and the Euclidean min-norm rule in
the original coordinates of this dataset is the min-norm rule in whitened
coordinates. We call this the \emph{self-realization} of the system and use
it to transfer dataset-level statements to whitened systems without any
change of coordinates.

\begin{lemma}[exact certificate]\label{lem:cert}
Let $S$ be a set of indices with weights $\alpha_i>0$ such that
$\sum_{i\in S}\alpha_ig_i=0$ and $\spn\{\xi_i:i\in S\}=\R^d$. Then $q=0$ is
the unique minimizer of $\sum_{i\in S}\alpha_i(\ip{q}{\xi_i}-\eta_i)^2$, and
the corresponding selection has risk ratio $1$.
\end{lemma}

\begin{proof}
The gradient of the objective at $q=0$ is $-2\sum_{i\in S}\alpha_ig_i=0$.
The Hessian is $2\sum_{i\in S}\alpha_i\xi_i\xi_i^{\top}$. It is positive definite
because the weights are positive and the features span. Zero-gradient points
may be added to $S$ with any positive weights: their $g_i=0$ keeps the
stationarity, and the extra features only help the Hessian.
\end{proof}

\begin{lemma}[volume sampling interface]\label{lem:L2}
Let $\{(\zeta_i,\eta_i)\}_{i=1}^{M}\subseteq\R^r\times\R$ have features
spanning $\R^r$, and let
$\rho(u)=\frac1M\sum_i(\ip{u}{\zeta_i}-\eta_i)^2$ with minimum $\rho^{\star}$.
Then there are $r$ indices with linearly independent features whose
interpolation point $\hat u$ (the unique solution of
$\ip{\hat u}{\zeta_i}=\eta_i$ on those indices) satisfies
$\rho(\hat u)\le(r+1)\rho^{\star}$.
\end{lemma}

\begin{proof}
The family is a legal dataset with features $\zeta_i$ and labels $\eta_i$.
Apply Theorem~5 of \citet{DW17} with size-$r$ volume sampling. This
assigns each subset $S$ probability proportional to $\det(Z_SZ_S^{\top})$.
The expected full-data loss of the least squares solution on $S$ is at most
$(r+1)\rho^{\star}$, with equality in general position. The distribution charges only subsets with independent
features. On these subsets, the least squares solution is the
interpolation point, independent of any weights. Some subset has loss no
greater than the expectation. If $\rho^{\star}=0$, any independent subset
interpolates the realizable labels exactly.
\end{proof}

Note that in Lemma~\ref{lem:L2} the interpolation point does not depend on
weights, so the min-norm rule is irrelevant on these full-rank supports.

\subsection{Positive spanning sets}

A finite set $B\subseteq\R^m\setminus\{0\}$ \emph{positively spans} $\R^m$ if
$\pos(B)=\R^m$, where $\pos$ denotes the set of nonnegative combinations.
$B$ is a \emph{positive basis} if in addition no proper subset positively
spans. A \emph{positive circuit} is a finite set carrying a linear dependence
with all coefficients strictly positive, minimal with respect to support.
A positive circuit of size $m+1$ has rank $m$. The next lemma is classical:
the size bounds and the antipodal structure of maximal positive bases are
due to \citet{Davis1954} and \citet{McKinney1962}; see \citet{Audet2011}
for a short proof and \citet{Regis2016,Schoch2020} for general structure
theory.

\begin{lemma}\label{lem:pb}
Every positive basis $B$ of $\R^m$ satisfies $m+1\le|B|\le 2m$. If
$|B|=2m$ then there is a linear basis $a_1,\dots,a_m$ and positive scalars
$\lambda_j$ with $B=\{a_1,-\lambda_1a_1,\dots,a_m,-\lambda_ma_m\}$.
\end{lemma}

\begin{lemma}\label{lem:l345}
Let $V=\{v_1,\dots,v_p\}$ positively span $\R^m$.
\begin{enumerate}
\item There are coefficients $\lambda_i>0$ with $\sum_i\lambda_iv_i=0$.
\item If $p=m+1$ then $V$ spans linearly, its kernel is one-dimensional, and
the kernel is generated by a strictly positive vector; so $V$ is a positive
circuit of rank $m$.
\item Let $K\subseteq\R^p$ be a linear subspace and let $\lambda$ generate an
extreme ray of $K\cap\R^p_{+}$, and let $S=\supp(\lambda)$. Then
$\dim(K\cap\R^S)=1$, where $\R^S$ is the coordinate subspace of vectors
supported in $S$. If $K$ is the kernel of a matrix with columns $v_i$, this
says that the columns indexed by $S$ carry a one-dimensional dependence
space, and $S$ is then a positive circuit.
\end{enumerate}
\end{lemma}

\begin{proof}
(1) For each $i$, write $-v_i=\sum_j\beta_{ij}v_j$ with $\beta_{ij}\ge0$ and
sum over $i$. The resulting identity $\sum_j(1+\sum_i\beta_{ij})v_j=0$ has all
coefficients at least $1$.
(2) The positive cone is contained in the linear span, so $V$ spans and the
kernel has dimension one. Part (1) puts a strictly positive vector in it, so
the unique dependence is strictly positive and supported everywhere.
(3) If the dependence space on $S$ had dimension two, take $\mu$ in $K$
supported in $S$ and not proportional to $\lambda$. Since $\lambda$ is
strictly positive on $S$, both $\lambda\pm\varepsilon\mu$ lie in the cone for
small $\varepsilon>0$, and $\lambda$ is their midpoint, contradicting
extremality.
\end{proof}

We use one nonclassical rigidity statement. It says that failing to
positively span with one point below the Steinitz budget forces an axis
structure.

\begin{lemma}[rigidity]\label{lem:rig}
Let $G\subseteq\R^m\setminus\{0\}$ positively span $\R^m$, and suppose no
subset of $G$ of size at most $2m-1$ positively spans $\R^m$. Then there is a
linear basis $a_1,\dots,a_m$ such that every element of $G$ lies on one of
the lines $\R a_j$.
\end{lemma}

\begin{proof}
Extract from $G$ an inclusion-minimal positively spanning subset $B$. By
hypothesis $|B|\ge 2m$ and by Lemma~\ref{lem:pb} $|B|\le 2m$, so $|B|=2m$ and
after positive rescaling $B=\{\pm a_1,\dots,\pm a_m\}$ for a linear basis
$(a_j)$. Take any $g\in G$ and write $g=\sum_{j\in S}c_ja_j$ with all
$c_j\ne0$. Suppose $|S|\ge2$. Put $o_j=\mathrm{sgn}(c_j)\,a_j$ and
$r_j=-o_j$ for $j\in S$, and set
\[
T=\{g\}\cup\{r_j:j\in S\}\cup\{\pm a_j:j\notin S\},
\qquad
|T|=1+|S|+2(m-|S|)=2m-|S|+1\le 2m-1 .
\]
For each $j\in S$,
$|c_j|\,o_j=g+\sum_{h\in S,\,h\ne j}|c_h|\,r_h\in\pos(T)$,
so $\pos(T)$ contains all $\pm a_j$ and hence equals $\R^m$. All elements of
$T$ lie in $G$ up to positive scaling, so this contradicts the hypothesis.
Hence $|S|=1$.
\end{proof}

Finally, we record a quotient lemma for passing from a positive basis to
a smaller one.

\begin{lemma}[quotient minimality]\label{lem:quot}
Let $B$ be a minimum-cardinality positively spanning subset of a set of
vectors in $\R^m$, let $C\subseteq B$ be a positive circuit with
$V=\spn(C)=\pos(C)$, and let $\pi:\R^m\to\R^m/V$ be the quotient map. Then
$\pi(b)\ne0$ for every $b\in B\setminus C$, and $\pi(B\setminus C)$ is a
positive basis of $\R^m/V$.
\end{lemma}

\begin{proof}
$\pi(B\setminus C)$ positively spans the quotient because $B$ positively
spans $\R^m$. If $\pi(b)=0$, then $b\in V=\pos(C)\subseteq\pos(B\setminus\{b\})$,
since $C\subseteq B\setminus\{b\}$. Thus $B\setminus\{b\}$ still positively
spans, contradicting minimality. If some
$\pi(b)$ were removable in the quotient, then for every $x\in\R^m$ we could
write $\pi(x)$ as a nonnegative combination of
$\pi(B\setminus C\setminus\{b\})$, lift, and absorb the error in
$V=\pos(C)$. This gives $x\in\pos(B\setminus\{b\})$, again contradicting
minimality.
\end{proof}

\section{The budget $2d-1$}\label{sec:endpoint}

This section proves $\Fw(d,2d-1)=1+1/d$ for every $d\ge1$. We state the
upper bound as an interface theorem for whitened systems. Later sections use
it in dimensions $2$ and $3$.

\subsection{The whitened upper bound}

\begin{theorem}[whitened endpoint interface]\label{thm:wend}
Let $\{(\xi_i,\eta_i)\}_{i=1}^{N}$ be a whitened system in $\R^d$. Then there
exist at most $2d-1$ indices and strictly positive weights whose objective
$G(q)=\sum\alpha_i(\ip{q}{\xi_i}-\eta_i)^2$ is strictly convex with unique
minimizer $\hat q$ satisfying
\[
1+\norm{\hat q}^2\;\le\;1+\frac1d .
\]
\end{theorem}

\begin{proof}
Write $g_i=\eta_i\xi_i$ and $V=\spn\{g_i:g_i\ne0\}$, $r=\dim V$. The features
span $\R^d$ by \eqref{eq:threeid}, and $\sum_ig_i=0$.

\emph{Case $r=0$.} All gradients vanish, so every point with $\xi_i\ne0$ has
$\eta_i=0$. Pick $d$ linearly independent features. Any positive weights give
a strictly convex objective with minimizer $q=0$ and value $1$. This uses $d$
points.

\emph{Case $1\le r\le d-1$.} Every nonzero-gradient point has
$\xi_i=g_i/\eta_i\in V$. The nonzero gradients span $V$ by definition of
$V$, and $0=\sum_ig_i$ expresses $0$ as a strictly positive zero-sum
combination of them. Hence their affine hull contains $0$ and equals their
span $V$. Moreover, $0$ lies in the interior of their convex hull in the
ambient space $V$. Steinitz's theorem applied inside $V$ yields at most $2r$
nonzero gradients whose convex hull still contains $0$ as a $V$-interior
point. In particular, their affine hull is again $V$, so these gradients span
$V$. The $V$-interiority provides strictly positive coefficients $\alpha_i$
with $\sum\alpha_ig_i=0$. The features of all nonzero-gradient points lie in $V$,
while all features together span $\R^d$. Hence points with $g_i=0$ and
$\xi_i\notin V$ exist and their features span $\R^d/V$. Choose $d-r$ of them
to complete the rank. These points have $\eta_i=0$, so adding them with any
positive weights preserves stationarity at $q=0$. By Lemma~\ref{lem:cert}
the selection recovers $q=0$ exactly with at most $2r+(d-r)=d+r\le 2d-1$
points.

\emph{Case $r=d$.} The identity $0=\sum_ig_i$ is a zero-sum combination of
the nonzero gradients with all coefficients equal to $1$. Given any
$x\in\R^d$, write $x$ as a linear combination of the nonzero gradients and
add a large multiple of this zero-sum relation. All coefficients become
nonnegative. Hence the nonzero gradients positively span $\R^d$.
If some subset of at most $2d-1$ nonzero gradients positively spans $\R^d$,
then Lemma~\ref{lem:l345}(1) gives a strictly positive zero-sum
weighting. Its features span, so Lemma~\ref{lem:cert} recovers $q=0$ with at
most $2d-1$ points. So assume no such subset exists. By
Lemma~\ref{lem:rig} there is a basis $a_1,\dots,a_d$ with every nonzero
gradient on one of the lines $L_j=\R a_j$.

Suppose some point has $\eta_0=0$ and $\xi_0\ne0$. Fix a size-$2d$ positive
basis of gradients. By Lemma~\ref{lem:pb}, it consists of an antipodal pair on
each line. Expand $\xi_0$ in the basis $(a_j)$ and pick a line $L_j$ where
its coefficient is nonzero. For each $k\ne j$ select the antipodal pair on
$L_k$ with weights that cancel their gradients, then add the point $\xi_0$.
The gradients sum to zero. The features contain $d-1$ of the lines plus a
vector with a nonzero component along $a_j$, so they span $\R^d$.
Lemma~\ref{lem:cert} recovers $q=0$ with $2(d-1)+1=2d-1$ points. So assume
from now on
\begin{equation}\label{eq:nozero}
g_i=0\ \Longrightarrow\ \xi_i=0 .
\end{equation}
Then every nonzero feature also lies on the $d$ lines.

Write $\xi_i=t_iu_j$ for $i\in I_j$, where $u_j$ is a unit vector along
$L_j$, and put $Q_j=\frac1N\sum_{i\in I_j}t_i^2$. The covariance identity
reads $I_d=\sum_{j=1}^dQ_ju_ju_j^{\top}$. With
$A=[\sqrt{Q_1}u_1\;\cdots\;\sqrt{Q_d}u_d]$, this says $AA^{\top}=I_d$. Since $A$ is
square, $A^{\top}A=I_d$ as well. This forces
\begin{equation}\label{eq:orth}
u_1,\dots,u_d\ \text{orthonormal},\qquad Q_j=1\ \text{for all }j .
\end{equation}

Let $R_j=\frac1N\sum_{i\in I_j}\eta_i^2$. Points with $\xi_i=0$ may still
carry residual energy, so $\sum_jR_j\le1$, and some line has
$R_{j_\ast}\le1/d$. On that line, $R_{j_\ast}$ is the
$\{t_i^2/N\}$-weighted average of $(\eta_i/t_i)^2$ and the weights sum to
$Q_{j_\ast}=1$. Hence some $i_\ast\in I_{j_\ast}$ has
$(\eta_{i_\ast}/t_{i_\ast})^2\le R_{j_\ast}\le1/d$.

On every other line $k$, projecting the normal equation onto $u_k$ gives
$\sum_{i\in I_k}t_i\eta_i=0$. By \eqref{eq:nozero}, every term is nonzero, so
both signs occur. Pick $a_k,b_k$ with $t_{a_k}\eta_{a_k}>0>t_{b_k}\eta_{b_k}$
and weights $\beta_{a_k}=-t_{b_k}\eta_{b_k}$, $\beta_{b_k}=t_{a_k}\eta_{a_k}$,
which cancel the two gradients.

Select the single point $i_\ast$ on line $j_\ast$ and the pair on each other
line, assign any strictly positive masses to the lines, and normalize. The lines are
orthogonal by \eqref{eq:orth}, so the objective separates across coordinates
$\ip{u_j}{q}$: each paired line has unique minimizer $0$, and line $j_\ast$
has unique minimizer $\eta_{i_\ast}/t_{i_\ast}$. The selection uses
$1+2(d-1)=2d-1$ points, is strictly convex, and its minimizer satisfies
$\norm{\hat q}^2=(\eta_{i_\ast}/t_{i_\ast})^2\le1/d$.
\end{proof}

\subsection{The endpoint theorem}

\begin{theorem}\label{thm:endpoint-ub}
For every $d\ge1$ and every finite dataset $D$,
$\Ls(2d-1;\Astar)\le\bigl(1+\frac1d\bigr)\Ls$.
\end{theorem}

\begin{proof}
Let $U=\spn\{x_i\}$ and $\rho=\dim U$. If $\rho<d$, predictions and losses
depend only on the projection of $w$ to $U$, and the min-norm ERM output lies
in $U$. The problem is therefore isometric to a $\rho$-dimensional one. Since
$2d-1\ge2\rho$, the case $n\ge2\rho$ of Theorem~8 in \citet{HMSYerm} gives
ratio $1$. If $\Ls=0$, pick $d$ linearly independent features, as we may now
assume full feature rank. The labels are realizable, so these constraints
determine the full-data predictor uniquely. The ratio is $1$ by the
$0/0$ convention. Otherwise whiten by Lemma~\ref{lem:whiten} and apply
Theorem~\ref{thm:wend}. The resulting objective is strictly convex, so the
min-norm rule does not act, and \eqref{eq:ratio} bounds the ratio by
$1+1/d$.
\end{proof}

\begin{theorem}\label{thm:endpoint-lb}
For every $d\ge1$, $\Fw(d,2d-1)\ge1+\frac1d$.
\end{theorem}

\begin{proof}
Fix $c\ge1$ and let
\[
D_{d,c}=\{(e_j,\,c+1),\ (-e_j,\,1-c)\ :\ j=1,\dots,d\}.
\]
With $w^{\star}=c\mathbf1$, both points of coordinate $j$ have residual $1$,
the normal equations hold, and $\Ls=1$. Writing $w=c\mathbf1+\delta$, the two
losses of coordinate $j$ sum to $2\delta_j^2+2$, so
\begin{equation}\label{eq:cross-risk}
L_D(w)=1+\frac1d\norm{w-c\mathbf1}^2 .
\end{equation}
The objective of any selection with effective support at most $2d-1$ misses
one of the $2d$ point types. Say this type involves coordinate $j$. The objective
separates across coordinates. If only $(e_j,c+1)$ is selected in coordinate
$j$, the unique minimizing value there is $w_j=c+1$. If only $(-e_j,1-c)$ is
selected, it is $w_j=c-1$. If neither is selected, the objective does not
depend on $w_j$, and the min-norm rule sets $w_j=0$, at distance $c\ge1$ from
$c$. In every case some coordinate deviates from $c$ by at least $1$, so
\eqref{eq:cross-risk} gives $L_D(\Astar(F))\ge1+1/d$. Conversely, selecting
all points except $(-e_d,1-c)$ with equal weights on each complete pair
outputs $(c,\dots,c,c+1)$, with risk exactly $1+1/d$.
\end{proof}

\begin{corollary}\label{cor:endpoint}
$\Fw(d,2d-1)=1+\frac1d$ for every $d\ge1$. In particular $\Fw(2,3)=\frac32$,
which resolves the smallest open cell of Question~2 in \citet{HMSYopen}.
\end{corollary}

The value confirms the assertion quoted in Section~\ref{sec:prelim}; the
note of \citet{HMSYopen} does not include a proof. For $d=2$, an extremal
instance is
\[
D_{2,2}=\{((1,0),3),\ ((-1,0),-1),\ ((0,1),3),\ ((0,-1),-1)\}:
\]
selecting the first three points with equal weights outputs $(2,3)$ and
attains risk ratio exactly $\frac32$.

\section{A lower bound for every intermediate budget}\label{sec:lb}

Write the intermediate budget as $n=d+k$ with $1\le k\le d-1$. For
$s\in\{k+1,\dots,d\}$ let $d=q_ss+a_s$ with $q_s\ge1$ and $0\le a_s<s$, and
define
\begin{equation}\label{eq:gamma}
C_{d,s}=\Bigl(\frac{s-a_s}{q_s}+\frac{a_s}{q_s+1}\Bigr)^{-1},
\qquad
\Gam_{d,k}=\max_{k+1\le s\le d}\,(s-k)\,C_{d,s}.
\end{equation}
Equivalently, $1/C_{d,s}$ is the minimum of $\sum_{j=1}^{s}1/r_j$ over
positive integer partitions $r_1+\dots+r_s=d$. The balanced partition attains
the minimum: replacing parts $r_i\ge r_j+2$ by $r_i-1,r_j+1$
strictly decreases the sum. The endpoints degenerate to $\Gam_{d,0}=d$ and
$\Gam_{d,d-1}=1/d$, matching $\Fw(d,d)=d+1$ and Corollary~\ref{cor:endpoint}.

\begin{proposition}[block instances]\label{prop:blocks}
Let $s\in\{k+1,\dots,d\}$ and let $r_1+\dots+r_s=d$ be a partition with
$C=(\sum_j1/r_j)^{-1}$. There is an explicit dataset $D$ with $\Ls=1$ and
\[
\Ls(d+k;\Astar)=1+(s-k)\,C .
\]
\end{proposition}

\begin{proof}
Decompose $\R^d=E_1\oplus^{\perp}\dots\oplus^{\perp}E_s$ with
$\dim E_j=r_j$. In $E_j$ place the $r_j+1$ simplex vertices
$v_{j,0}=-\sum_{m}e_{j,m}$ and $v_{j,m}=e_{j,m}$ for $m=1,\dots,r_j$, so
$N=d+s$ points in total. Set
\[
a_j=\sqrt{\frac{NC}{r_j(r_j+1)}},\qquad
h_j=(1,2,\dots,r_j),\qquad
w_j^{\star}=M\,h_j,
\]
with labels $y_{j,m}=\ip{w_j^{\star}}{v_{j,m}}+a_j$ and shift
\begin{equation}\label{eq:Mshift}
M=\sqrt2\,\Bigl[d\,a_{\max}+\sqrt{N\bigl((s-k)C+1\bigr)}\Bigr],
\qquad a_{\max}=\max_ja_j .
\end{equation}
Any larger value of $M$ also works in the argument below. The explicit
instances at the end of the section use round values.

\emph{Optimum.} Since $\sum_{m=0}^{r_j}v_{j,m}=0$, each block satisfies its
normal equations at $w_j^{\star}$ with all residuals equal to $a_j$. Thus
$\Ls=\frac1N\sum_j(r_j+1)a_j^2=\sum_jC/r_j=1$.

\emph{Cost of a block with $r_j$ selected vertices.} Any $r_j$ simplex
vertices form a basis of $E_j$, so every positively weighted objective on
them has the unique zero-training-error interpolant, independent of the
weights. Expressing the missed vertex's constraint using
$\sum_mv_{j,m}=0$ shows that its new residual has absolute value $(r_j+1)a_j$.
The block therefore contributes excess full-data risk
\[
\frac{(r_j+1)^2a_j^2-(r_j+1)a_j^2}{N}=\frac{r_j(r_j+1)a_j^2}{N}=C .
\]

\emph{No block may lose two vertices.} Suppose a selection uses at most
$r_j-1$ vertices of block $j$. Let $S$ be the selected set,
$|S|=p\le r_j-1$. Every minimizer of the selected objective restricted to
block $j$ interpolates $S$, and the min-norm output has the form
$\hat w_j=MP_Sh_j+b_S$ where $P_S$ projects onto the span of $S$ and $b_S$
lies in that span and interpolates the constant $a_j$. We use two facts:
\begin{enumerate}
\item $\norm{(I-P_S)h_j}\ge1/\sqrt2$. If $v_{j,0}\notin S$, some coordinate
direction is entirely absent from $S$ and $h_j$ has a nonzero entry there of
size at least $1$. If $v_{j,0}\in S$, at least two standard directions
$m_1,m_2$ are unselected. Every vector in $\spn(S)$ has equal coordinates at
$m_1,m_2$, while $h_j$ differs there by at least $1$, giving distance at
least $1/\sqrt2$.
\item $\norm{b_S}\le r_ja_j\le d\,a_{\max}$. Explicitly, if $v_{j,0}\notin S$
then $b_S=a_j\sum_{m\in T}e_m$ over the selected coordinates $T$, of norm
$a_j\sqrt{p}$. If $v_{j,0}\in S$ with $p-1$ standard vertices selected, the
min-norm interpolant puts $a_j$ on each selected coordinate and
$-pa_j/(r_j-p+1)$ on each unselected one, giving the exact value
\begin{equation}\label{eq:bS}
\norm{b_S}^2=a_j^2\Bigl[(p-1)+\frac{p^2}{r_j-p+1}\Bigr]\le
\frac{r_j^2-3}{2}\,a_j^2<r_j^2a_j^2 .
\end{equation}
\end{enumerate}
Hence
$\norm{\hat w_j-w_j^{\star}}\ge M/\sqrt2-d\,a_{\max}
=\sqrt{N((s-k)C+1)}$ by \eqref{eq:Mshift}. The block second-moment matrix is
$H_j=\frac1N(I+\mathbf1\mathbf1^{\top})\succeq\frac1NI$ on $E_j$, so this
block alone contributes excess risk $L_D(\hat w)-\Ls$ of at least
$(s-k)C+1$, strictly more than the candidate optimal excess $(s-k)C$. An
optimal selection therefore uses at least $r_j$ vertices in every block.

\emph{Counting.} The base cost is $\sum_jr_j=d$ points, so at most $k$
blocks can retain all $r_j+1$ vertices. At least $s-k$ blocks lose exactly
one vertex and contribute $C$ each. Conversely, keeping $k$ full blocks with
equal weights inside each block and dropping one vertex elsewhere attains
$1+(s-k)C$ exactly: full blocks output $w_j^{\star}$, and broken blocks
contribute $C$.
\end{proof}

\begin{theorem}\label{thm:gammalb}
For all $1\le k\le d-1$,
\[
\Fw(d,d+k)\;\ge\;1+\Gam_{d,k}.
\]
\end{theorem}

\begin{proof}
Apply Proposition~\ref{prop:blocks} with the maximizing $s$ in
\eqref{eq:gamma} and the balanced partition.
\end{proof}

\paragraph{Explicit instances.}
We verified the following three instances by exact computation. For $(d,n)=(3,4)$:
\[
z_1=\bigl((-1,-1,0),-30+\tfrac{\sqrt5}{3}\bigr),\quad
z_2=\bigl((1,0,0),10+\tfrac{\sqrt5}{3}\bigr),\quad
z_3=\bigl((0,1,0),20+\tfrac{\sqrt5}{3}\bigr),
\]
\[
z_4=\bigl((0,0,-1),-10+\sqrt{\tfrac53}\bigr),\quad
z_5=\bigl((0,0,1),10+\sqrt{\tfrac53}\bigr),
\]
with $w^{\star}=(10,20,10)$, $\Ls=1$ and
$\Ls(4;\Astar)=\frac53=1+\Gam_{3,1}$. The value is attained by
$\{z_1,z_2,z_3,z_5\}$ with weights $(\frac16,\frac16,\frac16,\frac12)$. For
$(d,n)=(4,5)$, two $2$-dimensional blocks with $M=17$ give
$w^{\star}=(17,34,17,34)$, $\Ls=1$ and $\Ls(5;\Astar)=2=1+\Gam_{4,1}$. For
$(d,n)=(5,6)$, blocks of dimensions $2$ and $3$ give
$\Ls(6;\Astar)=\frac{11}5=1+\Gam_{5,1}$. The following table gives small values of
$1+\Gam_{d,k}$:
\[
\begin{array}{c|ccccc}
d\backslash k&1&2&3&4&5\\\hline
2&\frac32&&&&\\
3&\frac53&\frac43&&&\\
4&2&\frac32&\frac54&&\\
5&\frac{11}5&\frac85&\frac75&\frac65&\\
6&\frac52&\frac53&\frac32&\frac43&\frac76
\end{array}
\]

\section{Orthogonal circuit-block systems: an exact minimax value}\label{sec:block}

The instances of Section~\ref{sec:lb} belong to a natural model class. On
this class, the harmonic quantity $\Gam_{d,k}$ is not merely a lower bound:
it is the exact minimax value, as we show in this section.

\begin{definition}\label{def:block}
A whitened system $\{(\xi_i,\eta_i)\}$ in $\R^d$ has \emph{orthogonal
circuit-block structure} with dimension vector $r=(r_1,\dots,r_s)$,
$\sum_jr_j=d$, if
$\R^d=E_1\oplus^{\perp}\dots\oplus^{\perp}E_s$ with $\dim E_j=r_j$, every
nonzero feature lies in some $E_j$, and every $E_j$ contains a positive
circuit of $r_j+1$ gradients spanning $E_j$. Here a positive circuit of
gradients means a set of indices whose gradients $g_i=\eta_i\xi_i$ carry a
strictly positive zero-sum relation, minimal with respect to support. Write
$I_j=\{i:\xi_i\in E_j\setminus\{0\}\}$ and
$R_j=\frac1N\sum_{i\in I_j}\eta_i^2$, so $\sum_jR_j\le1$; points with
$\xi_i=0$ belong to no block. The \emph{model class} $\mathcal D_r$ consists
of the datasets $D$ with full feature rank and $\Ls>0$ whose whitened system
under Lemma~\ref{lem:whiten} has orthogonal circuit-block structure with
dimension vector $r$.
\end{definition}

The class is defined at the level of datasets, not just whitened systems.
This distinction matters. A whitened system does not determine the min-norm
tie rule of the underlying dataset because whitening is not orthogonal.
For the self-realized dataset of a whitened system, the tie rule actually
helps the selector: the full-data optimum is the origin, and unconstrained
directions are resolved to it for free. The statements below address both
sides of this issue. The upper bound uses only strictly convex selections,
which do not depend on the tie rule and are therefore valid for every dataset in
$\mathcal D_r$. The lower bound uses the original-coordinate instances
of Proposition~\ref{prop:blocks}, whose large parameter shift makes the tie
rule costly.

Since the blocks are orthogonal and each block's nonzero features lie inside
it, the covariance identity forces
$\frac1N\sum_{i\in I_j}\xi_i\xi_i^{\top}$ to equal the orthogonal projector
onto $E_j$, and the normal equation projects to zero in each block.

\begin{theorem}[exact budget--energy value]\label{thm:lp}
Let $D\in\mathcal D_r$ and let its whitened system have energies
$(R_j)$; let the budget be $n=d+k$ with $0\le k\le s-1$. Then there is a
strictly convex, strictly positively weighted selection of at most $d+k$
points with
\[
\norm{\hat q}^2\;\le\;\min_{K\subseteq[s],\,|K|=k}\ \sum_{j\notin K}r_jR_j ,
\]
and the supremum of the right-hand side over admissible energies is exactly
\begin{equation}\label{eq:phi}
\Phi(r,k)=\max_{J\subseteq[s],\,|J|\ge k+1}
\frac{|J|-k}{\sum_{j\in J}1/r_j}.
\end{equation}
For $k\ge s$ the budget completes every block and the excess is zero.
\end{theorem}

\begin{proof}
\emph{Selection.} Fix $K$ with $|K|=k$. For $j\in K$, select the full circuit
of block $j$ with its strictly positive zero-sum weights. This uses $r_j+1$ points.
By Lemma~\ref{lem:cert} applied inside $E_j$, the block output is $0$. For
$j\notin K$, view the points indexed by $I_j$ as a dataset in
$E_j\cong\R^{r_j}$. Zero-feature points belong to no block and are never
selected. Its covariance is a positive multiple of the identity and
its normal equation vanishes, so its optimal parameter is $0$ and its optimal
risk is $(N/M_j)R_j$, where $M_j=|I_j|$. Lemma~\ref{lem:L2} provides
$r_j$ independent features whose interpolation point $u_j$ satisfies, after
expanding the risk around $0$,
$\frac{N}{M_j}\bigl(R_j+\norm{u_j}^2\bigr)\le(r_j+1)\frac{N}{M_j}R_j$, or
$\norm{u_j}^2\le r_jR_j$. The union of the selections uses
$\sum_jr_j+|K|=d+k$ points. The block objectives act on orthogonal coordinates
and the joint Hessian is positive definite. The joint minimizer is the direct
sum of the block minimizers, so
$\norm{\hat q}^2=\sum_{j\notin K}\norm{u_j}^2\le\sum_{j\notin K}r_jR_j$.

\emph{Worst-case energies.} Substituting $x_j=r_jR_j$, the adversary solves
\[
W(r,k)=\sup\Bigl\{\min_{|K|=k}\sum_{j\notin K}x_j\ :\
x\ge0,\ \sum_jx_j/r_j\le1\Bigr\},
\]
and we must show $W(r,k)=\Phi(r,k)$. The relaxation $x\ge0$ is deliberate:
Definition~\ref{def:block} forces every block energy to be strictly
positive. A maximizer with some $x_j=0$ is therefore approached within the
class as energies vanish, rather than attained. When the
maximizing $J$ in \eqref{eq:phi} is all of $[s]$, as for the balanced
partitions at the maximizing $s$ of \eqref{eq:gamma} used in
Proposition~\ref{prop:blocks}, the value is attained exactly. For other
balanced partitions, the maximizing $J$ may be a proper subset.

For the lower bound, fix $J$ attaining \eqref{eq:phi}, put
$x_j=C_J=(\sum_{h\in J}1/r_h)^{-1}$ for $j\in J$ and $x_j=0$ otherwise. The
constraint holds with equality, and removing the $k$ largest entries leaves
$(|J|-k)C_J=\Phi(r,k)$.

Every strictly positive feasible energy vector is realized by a member of
$\mathcal D_r$, so the supremum over admissible energies matches the
relaxed program up to the zero-energy limits already discussed. Given
$R_j>0$ with $\sum_jR_j\le1$, take in each $E_j$ the $r_j+1$ unit
directions $u_{j,0},\dots,u_{j,r_j}$ of a regular simplex, which satisfy
$\sum_mu_{j,m}=0$ and
$\sum_mu_{j,m}u_{j,m}^{\top}=\frac{r_j+1}{r_j}P_{E_j}$, and set
\[
\xi_{j,m}=\sqrt{\tfrac{Nr_j}{r_j+1}}\;u_{j,m},\qquad
\eta_{j,m}=\sqrt{\tfrac{NR_j}{r_j+1}},
\]
with one further pure-residual point $\xi_0=0$,
$\eta_0=\sqrt{N(1-\sum_jR_j)}$, and $N=d+s+1$. Each block then has
covariance $P_{E_j}$, energy exactly $R_j$, and an equal-weight gradient
circuit spanning $E_j$. The identities \eqref{eq:threeid} hold, and the
self-realized dataset of this system lies in $\mathcal D_r$.

For the upper bound, note that for any feasible $x$,
$\min_{|K|=k}\sum_{j\notin K}x_j$ equals $\sum_jx_j$ minus the sum of the $k$
largest entries. With $J_t=\{j:x_j>t\}$, the layer-cake identities give
\[
\min_{|K|=k}\sum_{j\notin K}x_j=\int_0^{\infty}\bigl(|J_t|-k\bigr)_{+}\,dt,
\qquad
\sum_j\frac{x_j}{r_j}=\int_0^{\infty}\sum_{j\in J_t}\frac1{r_j}\,dt\le1 .
\]
Whenever $|J_t|\ge k+1$, the definition of $\Phi$ gives
$|J_t|-k\le\Phi(r,k)\sum_{j\in J_t}1/r_j$. The integrand vanishes otherwise.
Integrating, $\min_{|K|=k}\sum_{j\notin K}x_j\le\Phi(r,k)$.
\end{proof}

\begin{corollary}\label{cor:lp-gamma}
$\displaystyle\max_{\substack{k+1\le s\le d\\r_1+\dots+r_s=d}}
\Phi(r,k)=\Gam_{d,k}$. Hence
\[
\sup_{r}\ \sup_{D\in\mathcal D_r}
\Bigl(\frac{\Ls(d+k;\Astar)}{\Ls}-1\Bigr)=\Gam_{d,k},
\]
the outer supremum over dimension vectors with $s\ge k+1$ parts, and the
value is attained by the balanced simplex-block datasets of
Proposition~\ref{prop:blocks}.
\end{corollary}

\begin{proof}
For the identity between the maxima, take any $J$ in \eqref{eq:phi} with
$t=|J|$ and $\sum_{j\in J}r_j\le d$.
Among positive integers with $t$ parts and total at most $d$, the sum
$\sum_{j\in J}1/r_j$ is minimized by taking total exactly $d$ and balancing,
since enlarging a part decreases its reciprocal. Thus
$\sum_{j\in J}1/r_j\ge1/C_{d,t}$ and
$(t-k)/\sum_{j\in J}1/r_j\le(t-k)C_{d,t}\le\Gam_{d,k}$. Conversely, choose
$t$ maximizing \eqref{eq:gamma}, partition all of $d$ into a balanced
$t$-partition, and take $J=[t]$.

For the dataset statement, the upper bound follows from Theorem~\ref{thm:lp}. For
every $D\in\mathcal D_r$ the constructed selection is strictly convex, so
its output does not involve the tie rule, and \eqref{eq:ratio} bounds the
excess by $\Phi(r,k)\le\Gam_{d,k}$. For the lower bound, the balanced
dataset of Proposition~\ref{prop:blocks} lies in $\mathcal D_r$: its
full-data covariance is block-diagonal with respect to the $E_j$, since each
block's simplex has second moment $\frac1N(I+\mathbf1\mathbf1^{\top})$
inside its own subspace. Whitening therefore acts block by block. It maps each block's
features into the same $E_j$ and preserves each simplex gradient circuit
together with its strictly positive dependence. Its optimal excess is
$(s-k)C=\Gam_{d,k}$ for the balanced partition at the maximizing $s$.
\end{proof}

Two remarks are worth noting. First, Theorem~\ref{thm:lp} and
Proposition~\ref{prop:blocks} give matching bounds within the
model class. The conjectured value of $\Fw$ in the whole open regime is the
exact answer for this class. Second, the maximum in \eqref{eq:phi} is in
general attained by a proper subset $J$: for unbalanced dimension vectors the
worst adversary shuts off some blocks entirely. The balanced formula
$\Gam_{d,k}$ reappears only after optimizing over partitions. We verified
Corollary~\ref{cor:lp-gamma} by exhaustive enumeration of all partitions for
$d\le14$ and all $k$.

\section{The budget $4$ in dimension $3$}\label{sec:fw34}

This section proves $\Fw(3,4)=\frac53$. The lower bound is
Theorem~\ref{thm:gammalb} at $(d,k)=(3,1)$. The upper bound requires new
structure: a complete classification of five-vector positive bases of $\R^3$
and a covering theorem for the gradients outside such a basis.

\subsection{Penalty limits and anchors}

\begin{lemma}[multi-anchor penalty limit]\label{lem:penalty}
Let $C$ have full row rank, $\mathcal A=\{q:Cq=b\}$, and let $G\ge0$ be a
convex quadratic whose restriction to $\mathcal A$ has a unique minimizer
$q_{\mathcal A}$. Assume the Hessian nullspace of $G$ meets $\ker C$ only at
$0$. Then for every $\lambda>0$ the function
$F_\lambda(q)=G(q)+\lambda\norm{Cq-b}^2$ is strictly convex, and its unique
minimizer $q_\lambda$ converges to $q_{\mathcal A}$ as $\lambda\to\infty$.
\end{lemma}

\begin{proof}
Strict convexity follows from the nullspace condition. Optimality gives
$F_\lambda(q_\lambda)\le F_\lambda(q_{\mathcal A})=G(q_{\mathcal A})$, so
$G(q_\lambda)\le G(q_{\mathcal A})$ and
$\lambda\norm{Cq_\lambda-b}^2\le G(q_{\mathcal A})$, using $G\ge0$.
The constraint violation therefore tends to $0$. Decompose
$q_\lambda=v_\lambda+w_\lambda$ with $w_\lambda\in\ker C$ and
$v_\lambda\perp\ker C$. The component $v_\lambda$ is bounded because
$Cq_\lambda$ is bounded and $C$ is injective on $(\ker C)^{\perp}$. On the
affine sets $\{Cq=Cq_\lambda\}$ the quadratic $G$ is strictly convex in the
$\ker C$ direction, uniformly in the translate, so the sublevel set
$\{G\le G(q_{\mathcal A})\}$ intersected with bounded constraint violation is
bounded in $w_\lambda$ as well. Any accumulation point of $q_\lambda$ lies in
$\mathcal A$ and minimizes $G$ there, so it equals $q_{\mathcal A}$.
\end{proof}

When each row of $C$ is a feature $\xi_i^{\top}$ and $b_i=\eta_i$, the
penalty $\lambda\norm{Cq-b}^2$ is a positive combination of original losses.
Every finite $\lambda$ therefore yields a legal, strictly positively
weighted, strictly convex objective, and the risk of its output approaches
the risk of $q_{\mathcal A}$. Since $\Ls(n;\Astar)$ is an infimum, limits of
this kind suffice for upper bounds.

\begin{lemma}[anchored reduction in $\R^3$]\label{lem:anchor3}
Let a whitened system in $\R^3$ contain a point with $\eta_0=0$ and
$\xi_0\ne0$. Then for every $\varepsilon>0$ there is a strictly convex,
strictly positively weighted selection of at most $4$ points with
$1+\norm{\hat q}^2\le\frac32+\varepsilon$.
\end{lemma}

\begin{proof}
Let $W=\xi_0^{\perp}$ and $\zeta_i=P_W\xi_i$. The projected family
$\{(\zeta_i,\eta_i)\}$ satisfies the three identities \eqref{eq:threeid} on
$W\cong\R^2$: the covariance projects to $P_WP_W=I_W$, and the normal equation
and energy are unchanged. Apply Theorem~\ref{thm:wend} with $d=2$ to obtain
at most $3$ indices and positive weights whose objective $G$, formed from the
\emph{original} three-dimensional losses, is strictly convex on $W$ with
unique minimizer $q_W\in W$ satisfying $1+\norm{q_W}^2\le\frac32$. The
selected features control only their $W$-components. Strict convexity on $W$
holds because the chosen $\zeta$'s span $W$. Add the anchor constraint
$\ip{q}{\xi_0}=0$ through Lemma~\ref{lem:penalty}. The penalty is the anchor
point's own loss. The nullspace condition holds because a direction killed by
both $G$ and the anchor would lie in $W$, contradicting strict convexity
there. Thus $q_\lambda\to q_W$. A large finite $\lambda$ gives a four-point
selection within $\varepsilon$.
\end{proof}

Consequently, in the branches below we may assume
\begin{equation}\label{eq:nozero3}
g_i=0\ \Longrightarrow\ \xi_i=0,
\end{equation}
because a point with $\eta_i=0\ne\xi_i$ triggers Lemma~\ref{lem:anchor3} and
the bound $\frac32<\frac53$.

\begin{lemma}[plane energy]\label{lem:plane}
Let $C$ be a positive circuit of three gradients of rank $2$ with plane
$P=\spn\{g_i:i\in C\}$ and weights $\alpha_i>0$, $\sum_{i\in C}\alpha_ig_i=0$.
Define $R(P)=\frac1N\sum_{i:\xi_i\in P}\eta_i^2$, where points with
$\xi_i=0$ count as members of $P$. Then under \eqref{eq:nozero3} there is a
strictly convex four-point selection with
$\norm{\hat q}^2\le1-R(P)$.
\end{lemma}

\begin{proof}
Let $u$ be a unit normal of $P$. Every point with $\xi_j\notin P$ has
$\ip{u}{\xi_j}\ne0$ and, by \eqref{eq:nozero3}, $\eta_j\ne0$. The covariance
identity gives $\sum_{j:\xi_j\notin P}\ip{u}{\xi_j}^2=N$, so by weighted
averaging some $j$ has
$\eta_j^2/\ip{u}{\xi_j}^2\le\frac1N\sum_{\xi_j\notin P}\eta_j^2=1-R(P)$.
Select the circuit with its weights and the point $j$, and let
$q_j=\bigl(\eta_j/\ip{u}{\xi_j}\bigr)u$. Circuit residuals at $q_j$ are
$-\eta_i$ since $q_j\perp P$, so the circuit part of the gradient is
$-2\sum\alpha_ig_i=0$. Point $j$ is interpolated at $q_j$. The features span
$\R^3$, so $q_j$ is the unique minimizer, and
$\norm{q_j}^2\le1-R(P)$.
\end{proof}

\subsection{Five-vector positive bases of $\R^3$}

The classification below is the case $\ell=3$ of Theorem~5.3 of
\citet{CKR2024}, which characterizes the positive bases of size $2\ell-1$
of an $\ell$-dimensional space. We include a short proof because the
coupled normal form and the constant $\kappa$ are used explicitly in what
follows.

\begin{theorem}[classification]\label{thm:class5}
Let $B=\{b_1,\dots,b_5\}$ be a positive basis of $\R^3$. Up to renumbering,
independent positive rescaling of each vector, and an invertible linear map,
\[
b_1=e_1,\quad b_2=e_2,\quad b_3=-e_1-e_2,\quad
b_4=e_3-\tfrac\kappa2e_1,\quad b_5=-e_3-\tfrac\kappa2e_1,
\]
for some $\kappa\ge0$; equivalently
\begin{equation}\label{eq:tworel}
b_1+b_2+b_3=0,\qquad \kappa b_1+b_4+b_5=0 .
\end{equation}
For $\kappa=0$ the basis splits into a rank-$2$ triangle circuit and an
antipodal pair; all $\kappa>0$ form one orbit, in which two rank-$2$
triangle circuits share the vector $b_1$.
\end{theorem}

\begin{proof}
Let $K$ be the kernel of the $3\times5$ matrix with columns $b_i$.
$\dim K=2$. The cone $\mathcal C=K\cap\R^5_{+}$ contains a strictly positive
vector by Lemma~\ref{lem:l345}(1), so it is a two-dimensional pointed cone
with exactly two extreme rays. By Lemma~\ref{lem:l345}(3) each extreme ray is
supported on a positive circuit. A circuit of rank $3$ would positively span
$\R^3$ with at most four elements, contradicting minimality of $B$. Thus both
circuits have rank at most $2$ and support size at most $3$. The strictly
positive vector is a positive combination of the two rays, so the two
supports cover $\{1,\dots,5\}$. Two supports of size at most $3$ covering
five coordinates are either disjoint of sizes $3,2$, or of sizes $3,3$
sharing one coordinate. The first case is a triangle circuit plus an
antipodal pair. The second is two triangles sharing one vector. Their planes
differ because the matrix has rank $3$. Normalizing the first triangle to
$e_1,e_2,-e_1-e_2$ and the second relation to $\kappa b_1+b_4+b_5=0$, and
sending $(b_4-b_5)/2$ to $e_3$, yields the stated form. A diagonal map
$\mathrm{diag}(\lambda,\lambda,\mu)$ rescales $\kappa$ by $\lambda/\mu$, so
all $\kappa>0$ are equivalent.
\end{proof}

\subsection{The coupled branch: a three-plane cover}

Fix a coupled basis, positively normalized as
\begin{equation}\label{eq:coupled}
a_1+a_2+a_3=0,\qquad \kappa a_1+a_4+a_5=0,\qquad\kappa>0,
\end{equation}
and write $P_A=\spn(a_1,a_2)$, $P_B=\spn(a_1,a_4)$.

\begin{lemma}\label{lem:cone2}
Assume no four gradients form a positive circuit of rank $3$. Then every
nonzero gradient $g$ outside the basis satisfies $-g=\alpha a_i$ or
$-g=\alpha a_i+\beta a_j$ with $\alpha,\beta>0$.
\end{lemma}

\begin{proof}
$-g\in\pos(B)$, so take a support-minimal nonnegative representation. Its
support is linearly independent, since a dependence could otherwise shrink it. A support
of size $3$ would give, together with $g$, a four-point positive circuit of
rank $3$.
\end{proof}

Call $g$ of \emph{type $ij$} in the second case. Types inside one triangle
lie in its plane: types $12,13,23$ lie in $P_A$ and types $14,15,45$ in
$P_B$. The remaining types $24,25,34,35$ are called \emph{cross} types.

\begin{lemma}\label{lem:onecross}
Assume no four gradients form a positive circuit of rank $3$. Then at most
one cross type occurs.
\end{lemma}

\begin{proof}
Let $p=-\alpha a_u-\beta a_v$ and $q=-\gamma a_{u'}-\delta a_{v'}$ with
$u,u'\in\{2,3\}$, $v,v'\in\{4,5\}$, $(u,v)\ne(u',v')$. In each of the three
cases we exhibit a four-point positive circuit, giving a contradiction. If $u=u'$,
$v\ne v'$, compute using \eqref{eq:coupled}, for $u=2$:
$\delta p+\beta q=-(\delta\alpha+\beta\gamma+\beta\delta\kappa)a_2
-\beta\delta\kappa a_3$, so $\{p,q,a_2,a_3\}$ carries a positive dependence,
and its rank is $3$ because $p,q$ leave the plane $P_A$. If $v=v'$, $u\ne u'$,
for $v=4$:
$\gamma p+\alpha q
=-\bigl(\tfrac{\alpha\gamma}{\kappa}+\gamma\beta+\alpha\delta\bigr)a_4
-\tfrac{\alpha\gamma}{\kappa}a_5$, which gives the same conclusion with $a_4,a_5$. If
$u\ne u'$ and $v\ne v'$:
$\delta p+\beta q=-(\delta\alpha+\beta\delta\kappa)a_u
-(\beta\gamma+\beta\delta\kappa)a_{u'}$, which gives the same conclusion with $a_u,a_{u'}$.
\end{proof}

\begin{lemma}[three-plane cover]\label{lem:cover}
Assume no four gradients form a positive circuit of rank $3$. Then all
nonzero gradients lie in $P_A\cup P_B$, or in
$P_A\cup P_B\cup P_{uv}$ for a single cross type $(u,v)$ with
$P_{uv}=\spn(a_u,a_v)$. Each listed plane contains a rank-$2$ positive
triangle circuit.
\end{lemma}

\begin{proof}
Lemma~\ref{lem:cone2} reduces every extra gradient to a ray or a two-element
cone type. The six in-triangle types lie in $P_A\cup P_B$, and by
Lemma~\ref{lem:onecross} at most one cross type $(u,v)$ occurs. Its members
lie in $P_{uv}$. Rays lie in these planes as well. $P_A$ and $P_B$ contain
the two basis triangles. If the cross plane occurs, $a_u,a_v$ together with
any cross gradient of that type form a positive triangle circuit spanning
$P_{uv}$.
\end{proof}

\begin{proposition}[coupled branch]\label{prop:coupled}
Under \eqref{eq:nozero3}, if the nonzero gradients contain a coupled
five-vector positive basis and no four of them form a rank-$3$ positive
circuit, then some four-point strictly convex selection has
$\norm{\hat q}^2\le\frac23$.
\end{proposition}

\begin{proof}
By Lemma~\ref{lem:cover} the nonzero features, which are collinear with
their gradients, lie in at most three planes, each carrying a positive
triangle circuit. Points with $\xi_i=0$ lie in every plane. The three plane
energies cover total energy $1$, so some plane has $R(P)\ge\frac13$. With
two planes, some plane has $R(P)\ge\frac12$. Lemma~\ref{lem:plane} gives
$\norm{\hat q}^2\le1-R(P)\le\frac23$.
\end{proof}

\subsection{The split branch and the six-axis branch}

\begin{proposition}[split branch]\label{prop:split}
Under \eqref{eq:nozero3}, suppose the nonzero gradients contain a
$\kappa=0$ basis, i.e.\ a rank-$2$ triangle circuit with plane $V$ and an
antipodal pair on a line $L$ with $V\oplus L=\R^3$, and suppose all nonzero
gradients lie in $V\cup L$. Then some four-point strictly convex selection
has $\norm{\hat q}^2\le\frac23$.
\end{proposition}

\begin{proof}
All nonzero features lie in $V\cup L$ by \eqref{eq:nozero3}. Write
$H_V,H_L$ for the feature second moments of the two groups; then
$I=H_V+H_L$ with $\rk H_V=2$, $\rk H_L=1$. Writing
$H_L=\lambda uu^{\top}$, the rank of $H_V=I-\lambda uu^{\top}$ equals $2$
only if $\lambda=1$. Hence $L=\R u$, $V=u^{\perp}$, $H_V=P_V$, $H_L=P_L$.
Assign points with $\xi_i=0$ to the $V$ group and write
$R_V+R_L\le1$ for the group energies. The normal equation splits along the
orthogonal decomposition, so both groups have optimum $0$.

Selection A: the full triangle circuit in $V$ plus the best single point on
$L$. On $L$, $\frac1N\sum t_i^2=1$, and the weighted-average argument of
Theorem~\ref{thm:wend} yields a point with squared offset at most $R_L$. The
circuit fixes the $V$ component at $0$. Hence $\Delta_A\le R_L$ with $4$
points.

Selection B: an antipodal pair on $L$ with canceling weights, plus a
volume-sampling pair in $V$. The $V$ group, viewed as a dataset in
$V\cong\R^2$ of size $N_V$, has covariance $(N/N_V)I_V$ and optimal risk
$(N/N_V)R_V$. Lemma~\ref{lem:L2} with $r=2$ gives two independent features
whose interpolation point $u_V$ satisfies
$(N/N_V)\bigl(R_V+\norm{u_V}^2\bigr)\le3(N/N_V)R_V$, so
$\norm{u_V}^2\le2R_V$. The pair fixes the $L$ component at $0$. The joint
objective is strictly convex on $\R^3$ and separates, so
$\Delta_B\le2R_V$ with $4$ points.

Finally, $\min\{R_L,2R_V\}\le\frac23$ whenever $R_V+R_L\le1$.
\end{proof}

\begin{proposition}[six-axis branch]\label{prop:sixaxis}
Under \eqref{eq:nozero3}, suppose the minimum-cardinality positively
spanning subset of the nonzero gradients has size $6$ and all nonzero
gradients lie on its three lines. Then some four-point strictly convex
selection has $\norm{\hat q}^2\le\frac23$.
\end{proposition}

\begin{proof}
By Lemma~\ref{lem:pb} the basis consists of antipodal pairs on three lines.
All nonzero features lie on the lines by \eqref{eq:nozero3}, so the
covariance identity with three rank-one groups forces, as in
\eqref{eq:orth}, orthonormal directions with unit feature energy each. Let
$R_1\ge R_2\ge R_3$ be the line energies, $\sum R_j\le1$. Select an
antipodal canceling pair on the line of highest energy and the best single point on
each of the other two. On each line the offset is $0$ or at most $R_j$. Then
$\norm{\hat q}^2\le R_2+R_3\le\frac23$, with $2+1+1=4$ points.
\end{proof}

\subsection{The main theorem for $(3,4)$}

\begin{theorem}\label{thm:fw34}
$\Fw(3,4)=\frac53$.
\end{theorem}

\begin{proof}
The lower bound is Theorem~\ref{thm:gammalb}. For the upper bound, let $D$ be
any dataset. If the features span a subspace of dimension $\rho\le2$, budget
$4\ge2\rho$ gives ratio $1$ as in Theorem~\ref{thm:endpoint-ub}. If
$\Ls=0$, three independent features determine the realizable predictor. So
assume full rank and $\Ls>0$, and whiten. If some point has
$\eta_i=0\ne\xi_i$, Lemma~\ref{lem:anchor3} gives ratio at most
$\frac32+\varepsilon$, which suffices because we take an infimum. Thus assume
\eqref{eq:nozero3}. If the nonzero gradients spanned a proper subspace, some
feature would lie outside it with zero gradient, contradicting
\eqref{eq:nozero3} and full feature rank. Thus they span, and as in
Theorem~\ref{thm:wend} they positively span $\R^3$. Let $B$ be a
minimum-cardinality positively spanning subset. Then $|B|\in\{4,5,6\}$ by
Lemma~\ref{lem:pb}.

If $|B|=4$: by Lemma~\ref{lem:l345}(2), $B$ is a full-rank positive circuit.
Lemma~\ref{lem:cert} recovers exactly with $4$ points.

If $|B|=5$: apply Theorem~\ref{thm:class5}. If at any stage four gradients
form a rank-$3$ positive circuit, Lemma~\ref{lem:cert} recovers exactly.
Otherwise, in the coupled case Proposition~\ref{prop:coupled} applies. In
the split case, suppose some nonzero gradient $g$ lies outside $V\cup L$.
Its minimal cone representation uses one triangle vertex $a_u$ and one pair
gradient $\ell$, giving a cross relation $g+\alpha a_u+\beta\ell=0$ with
$\alpha,\beta>0$. The five gradients consisting of the triangle, $g$ and
$\ell$ then positively span $\R^3$: the triangle positively spans $V$, and
the cross relation exhibits $-\ell$ as a positive combination. If some four
of these five positively span, we get exact recovery; otherwise they form a
five-vector positive basis containing two triangles that share $a_u$. This is
a coupled basis, so the previous case applies. If no such gradient exists,
Proposition~\ref{prop:split} applies.

If $|B|=6$: by Lemma~\ref{lem:pb}, $B$ consists of antipodal pairs on
three lines. An extra gradient with components on two or three of these
lines would produce, by the construction inside Lemma~\ref{lem:rig}, a
positively spanning set of five or four gradients, contradicting the
minimality of $B$. Hence all nonzero gradients lie on the three lines and
Proposition~\ref{prop:sixaxis} applies.

Every branch ends with a strictly convex objective of at most $4$ points and
$\norm{\hat q}^2\le\frac23$, or with the anchored branch at
$\frac32+\varepsilon$. By \eqref{eq:ratio} and the infimum in the
definition, $\Ls(4;\Astar)\le\frac53\Ls$.
\end{proof}

\begin{remark}[the refuted interpolation formula]\label{rem:refute}
The three known values $\Fw(d,d)=d+1$, $\Fw(d,2d-1)=1+1/d$ and
$\Fw(d,2d)=1$ are consistent with the guess $\Fw(d,n)=1+(2d-n)^2/d$, which
predicts $\Fw(3,4)=\frac73$. Theorem~\ref{thm:fw34} refutes this formula at
the smallest cell where it differs from $1+\Gam_{d,k}$.
\end{remark}

\begin{proposition}[the coupled class stays below $3/2$]\label{prop:sharp5}
Let the dataset consist of exactly five points whose gradients form a
coupled five-vector positive basis. Then
$\Ls(4;\Astar)<\frac32\Ls$.
\end{proposition}

\begin{proof}
Normalize $\Ls=1$ and let $c_i=\eta_i^2$, so $\sum_ic_i=5$. The dependence
cone of the five gradients has two circuit rays. After renumbering, there is
$\theta\in(0,1)$ with $\theta g_1+g_2+g_3=0$ and $(1-\theta)g_1+g_4+g_5=0$.
Consider the selection consisting of the circuit $\{1,2,3\}$ with its
balancing weights plus the single point $4$, interpolated. The unique output
$q$ satisfies $q\perp\spn(g_1,g_2)$ and $\ip{q}{\xi_4}=\eta_4$, so
$\ip{q}{g_4}=\eta_4\ip{q}{\xi_4}=c_4$. In the basis
$B=[g_1\,g_2\,g_4]$ with $s=B^{\top}q$ we get $s=(0,0,c_4)$. The covariance
identity reads $5I=BRB^{\top}$ with
$R=\sum_iv_iv_i^{\top}/c_i$, where $v_i$ are the coefficient vectors of
$g_i$ in $B$. Hence $\norm{q}^2=\frac15s^{\top}Rs$ and the excess of this
selection is $\Delta_4=\frac15\bigl(c_4+c_4^2/c_5\bigr)$. Symmetric
selections give $\Delta_5,\Delta_2,\Delta_3$ with the roles exchanged. For
$a,b>0$, $\min\{a+a^2/b,\;b+b^2/a\}\le a+b$, so
\[
\min_i\Delta_i\le\tfrac15\min\{c_2+c_3,\;c_4+c_5\}
\le\tfrac1{10}(5-c_1)<\tfrac12 .
\]
\end{proof}

\begin{remark}
We do not know whether $\frac32$ is the exact supremum of
$\Ls(4;\Astar)/\Ls$ over this class. On the path $c_1\downarrow0$,
$c_2=c_3=c_4=c_5\uparrow\frac54$ the upper bound above tends to $\frac12$,
and exact optimization over all full-rank supports approaches ratio
$\frac32$ from below; but a matching lower bound over all selections,
including rank-deficient supports under the min-norm rule, has not been
proved. We record this as a numerical observation only.
\end{remark}

\section{The budget $5$ in dimension $4$}\label{sec:fw45}

This section proves $\Fw(4,5)=2$. The lower bound is
Theorem~\ref{thm:gammalb} at $(d,k)=(4,1)$. The upper bound combines four
ingredients: interface versions of the preceding results, a completion
lemma for any low-rank positive circuit based on quotient volume
sampling, classifications of the minimal positive bases of $\R^4$, and a
dimension-free extremal-basis argument for the one remaining
configuration, the rectangle.

\subsection{Interface lemmas}

\begin{theorem}[whitened $(3,4)$ interface]\label{thm:L13D}
Let $\{(\xi_i,\eta_i)\}$ be a whitened system in $\R^3$ and
$\varepsilon>0$. There are at most $4$ indices and strictly positive weights
whose objective is strictly convex with unique minimizer $\hat q$ satisfying
$1+\norm{\hat q}^2\le\frac53+\varepsilon$.
\end{theorem}

\begin{proof}
If the nonzero gradients span a proper subspace $V$, some feature lies
outside $V$. Its gradient lies in $V$ only if it vanishes, so this point has
$\eta_0=0\ne\xi_0$, and Lemma~\ref{lem:anchor3} gives
$\frac32+\varepsilon$. If they span $\R^3$ and some point has
$\eta_0=0\ne\xi_0$, Lemma~\ref{lem:anchor3} applies again. Otherwise
\eqref{eq:nozero3} holds, and the case analysis in the proof of
Theorem~\ref{thm:fw34} for spanning gradients applies unchanged to the
whitened system. It produces finite strictly convex selections of at most $4$
points with $\norm{\hat q}^2\le\frac23$ in every branch.
\end{proof}

\begin{corollary}[anchored reduction in $\R^4$]\label{cor:anchor4}
Let a whitened system in $\R^4$ contain a point with $\eta_0=0$ and
$\xi_0\ne0$. Then for every $\varepsilon>0$ there is a strictly convex,
strictly positively weighted selection of at most $5$ points with
$1+\norm{\hat q}^2\le\frac53+\varepsilon$.
\end{corollary}

\begin{proof}
Project to $W=\xi_0^{\perp}\cong\R^3$. The projected family is a whitened
system on $W$, as in Lemma~\ref{lem:anchor3}. Fix $\delta_m\downarrow0$. By
Theorem~\ref{thm:L13D} there is a finite objective $G_m$ of at most $4$
original points whose restriction to $W$ is strictly convex with minimizer
$u_m$ and $1+\norm{u_m}^2\le\frac53+\delta_m$. If the construction of $G_m$
itself involves an inner anchor penalty, fix that inner parameter first at a
finite value achieving the bound. Then add the outer penalty
$\lambda\ip{q}{\xi_0}^2$ and apply Lemma~\ref{lem:penalty}. For a large
finite $\lambda_m$ the unique minimizer $q_m$ of the five-point objective
satisfies $\bigl|\norm{q_m}^2-\norm{u_m}^2\bigr|\le\delta_m$. Let
$m\to\infty$.
\end{proof}

\subsection{Completion and circuit-star lemmas}

\begin{lemma}[minimal cone representation]\label{lem:mincone}
Let $B$ positively span $\R^m$ and $g\ne0$. Then $-g=\sum_{t\in T}\alpha_tt$
for some linearly independent $T\subseteq B$ and $\alpha_t>0$, and
$T\cup\{g\}$ is a positive circuit of rank $|T|$.
\end{lemma}

\begin{proof}
Take a support-minimal nonnegative representation of $-g$ over $B$. If its
support carried a dependence $\sum c_tt=0$, shifting by
$\lambda=\min_{c_t>0}\alpha_t/c_t$ would zero out a coefficient without
changing the sum, contradicting minimality. The strictly positive relation
$g+\sum\alpha_tt=0$ on an independent support is a positive circuit.
\end{proof}

\begin{lemma}[circuit-flat completion]\label{lem:completion}
Let a whitened system in $\R^d$ contain a positive circuit $C$ of gradients
with $|C|=r+1$, $P=\spn\{g_i:i\in C\}$, $\dim P=r$, and weights
$\alpha_i>0$, $\sum_{i\in C}\alpha_ig_i=0$. Define
\[
E(P)=\frac1N\sum_{i:\,\Pi_{P^{\perp}}\xi_i\ne0}\eta_i^2 .
\]
Then there is a strictly convex, strictly positively weighted selection of at
most $d+1$ points with $\norm{\hat q}^2\le(d-r)\,E(P)$.
\end{lemma}

\begin{proof}
If $r=d$ then $P=\R^d$, the circuit features span $\R^d$, and
Lemma~\ref{lem:cert} gives $\hat q=0$ with $r+1=d+1$ points. Thus assume
$r<d$. Let $\zeta_i=\Pi_{P^{\perp}}\xi_i$, $A=\{i:\zeta_i\ne0\}$, $M=|A|$ and
$m=d-r$. Points outside $A$ have $\xi_i\in P$, so
$\sum_{i\in A}\zeta_i\zeta_i^{\top}=\Pi\bigl(\sum_i\xi_i\xi_i^{\top}\bigr)\Pi
=N\,I_{P^{\perp}}$ and
$\sum_{i\in A}\eta_i\zeta_i=\Pi\sum_i\eta_i\xi_i=0$. The family
$(\zeta_i,\eta_i)_{i\in A}$ is thus a dataset on $P^{\perp}\cong\R^m$ with
covariance $(N/M)I$, optimum $0$ and optimal risk $(N/M)E(P)$. By
Lemma~\ref{lem:L2} there are $m$ indices $T\subseteq A$ with independent
projected features whose interpolation point $u\in P^{\perp}$ satisfies
$\frac{N}{M}(E(P)+\norm{u}^2)\le(m+1)\frac{N}{M}E(P)$, or
$\norm{u}^2\le m\,E(P)$.

Select $C$ with its circuit weights and $T$ with any positive weights. At
$u$, the circuit residuals are $-\eta_i$ since $\xi_i\in P\perp u$, so the
circuit gradient part is $-2\sum\alpha_ig_i=0$. The points of $T$ satisfy
$\ip{u}{\xi_i}=\ip{u}{\zeta_i}=\eta_i$ and are interpolated. The circuit
features span $P$ and the $T$ features project onto a basis of
$P^{\perp}$, so all selected features span $\R^d$ and $u$ is the unique
minimizer. The count is $(r+1)+m=d+1$.
\end{proof}

\begin{lemma}[circuit-star]\label{lem:star}
Let a whitened system in $\R^4$ contain a rank-$2$ positive triangle circuit
of gradients with plane $P$, and suppose every projected feature
$\Pi_{P^{\perp}}\xi_i$ lies on one of two lines of $P^{\perp}$. Then there
is a strictly convex five-point selection with $\norm{\hat q}^2\le1$.
\end{lemma}

\begin{proof}
Write the nonzero projections as $t_iu_j$ on lines $j=1,2$. The projected
covariance identity $I_{P^{\perp}}=Q_1u_1u_1^{\top}+Q_2u_2u_2^{\top}$ with
two rank-one terms forces, as in \eqref{eq:orth}, $u_1\perp u_2$ and
$Q_1=Q_2=1$. Let $R_j$ be the residual energy of line $j$. As in
Theorem~\ref{thm:wend}, each line contains a point $i_j$ with
$(\eta_{i_j}/t_{i_j})^2\le R_j$. Set
$q=(\eta_{i_1}/t_{i_1})u_1+(\eta_{i_2}/t_{i_2})u_2$. It interpolates both
line points and is orthogonal to $P$. Thus circuit stationarity holds at
$q$ and the five selected features span $\R^4$. Hence $q$ is the unique
minimizer and $\norm{q}^2\le R_1+R_2\le1$.
\end{proof}

\subsection{Branch reductions}

Throughout this subsection the system is a whitened system in $\R^4$
satisfying
\begin{equation}\label{eq:nozero4}
g_i=0\ \Longrightarrow\ \xi_i=0,
\end{equation}
its nonzero gradients positively span $\R^4$, and $B$ denotes a
minimum-cardinality positively spanning subset of the nonzero gradients,
$|B|\in\{5,6,7,8\}$ by Lemma~\ref{lem:pb}. For the structure of $B$
itself, the size-$6$ dichotomy of Proposition~\ref{prop:B6}
and the common-anchor form of Lemma~\ref{lem:sevenlines} are the cases
$\ell=4$ of Theorems~5.4 and~5.3 of \citet{CKR2024}, which characterize the
positive bases of sizes $\ell+2$ and $2\ell-1$ of an $\ell$-dimensional
space. The risk bounds also require control of the gradients outside $B$.
The proofs below supply that control.

\begin{proposition}\label{prop:B5}
If $|B|=5$, exact recovery with $5$ points.
\end{proposition}

\begin{proof}
By Lemma~\ref{lem:l345}(2), $B$ is a rank-$4$ positive circuit.
Apply Lemma~\ref{lem:cert}.
\end{proof}

\begin{proposition}[size-$6$ classification]\label{prop:B6}
Let $|B|=6$. Then either some four gradients of $B$ form a positive circuit
of rank $3$, in which case Lemma~\ref{lem:completion} gives
$\norm{\hat q}^2\le E(P)\le1$ with $5$ points, or $B$ is the disjoint union
of two rank-$2$ triangle circuits spanning complementary planes.
\end{proposition}

\begin{proof}
Let $K$ be the kernel of the matrix with columns $B$. Then $\dim K=2$, and
$\mathcal C=K\cap\R^6_{+}$ is a two-dimensional pointed cone with a strictly
positive interior point. Hence it has two extreme rays with circuit supports $S,T$
covering $[6]$, by Lemma~\ref{lem:l345}. A support of size $5$ would be a
rank-$4$ circuit that positively spans with five elements, contradicting minimality.
Size $1$ is impossible, and equal supports would give a two-dimensional
dependence space on a circuit. Thus $|S|,|T|\in\{2,3,4\}$ with
$|S\cup T|=6$, leaving $(3,3)$ disjoint, $(4,3)$ sharing one, $(4,4)$
sharing two, and $(4,2)$ disjoint. In the last three cases the size-$4$
support is a positive circuit of rank $3$. The case $(3,3)$ disjoint gives
two triangles whose planes are complementary because $B$ has rank $4$.
\end{proof}

\begin{proposition}[two blocks without cross gradients]\label{prop:22}
Let $B=A\cup C$ be two disjoint triangles with complementary planes $U,V$,
and suppose all nonzero gradients lie in $U\cup V$. Then a five-point
strictly convex selection achieves $\norm{\hat q}^2\le1$.
\end{proposition}

\begin{proof}
By \eqref{eq:nozero4} all nonzero features lie in $U\cup V$. Then
$I=H_U+H_V$ with $\rk H_U=\rk H_V=2$. Since $0\preceq H_U,H_V\preceq I$ and
the ranks are complementary, $U\perp V$ and $H_U=P_U$, $H_V=P_V$. The system
has orthogonal circuit-block structure with $r=(2,2)$, and
Theorem~\ref{thm:lp} with $k=1$ gives
$\norm{\hat q}^2\le\min\{2R_U,2R_V\}\le R_U+R_V\le1$ with $d+1=5$ points.
\end{proof}

\begin{definition}\label{def:cross}
In the two-block case, a \emph{cross gradient of type $(i,j)$} is a nonzero
gradient $h$ with minimal cone representation
$h+\alpha a_i+\beta c_j=0$, $\alpha,\beta>0$, where $a_i\in A$ and
$c_j\in C$. The \emph{rectangle branch} is the case where two cross types
$(i,j)$ and $(i',j')$ occur with $i\ne i'$ and $j\ne j'$.
\end{definition}

\begin{proposition}[shared-endpoint crosses]\label{prop:shared}
In the two-block case, suppose cross gradients occur and all cross types
share an endpoint, that is, all use the same $a_i$ or all use the same
$c_j$. Suppose also that no four gradients form a rank-$3$ positive circuit
and no five gradients positively span. Then a five-point strictly convex
selection achieves $\norm{\hat q}^2\le1$.
\end{proposition}

\begin{proof}
First reduce the extra gradients. By Lemma~\ref{lem:mincone} each nonzero
gradient outside $B$ has a minimal cone support in $B$ of size at most $4$.
Size $4$ gives a rank-$4$ five-point positive circuit, which positively
spans and is excluded. Size $3$ gives a rank-$3$ circuit, also excluded. Supports
of size $1$ or $2$ inside one triangle stay in that triangle's plane, and
the mixed supports are exactly the cross gradients. Say all cross types use
$a_1$. Fix one cross circuit $\{h,a_1,c_j\}$ and let
$P=\spn(a_1,c_j)$, a plane carrying this positive triangle circuit. Project
to $P^{\perp}$. The $A$-plane maps to the line spanned by the image of
$a_2$. The $C$-plane maps to the line spanned by the images of the
$c_{j'}$, $j'\ne j$, which is one line since $C$'s plane is
two-dimensional and $c_j$ maps to zero. A cross gradient of type $(1,j')$
maps into the $C$-line since $h'=-\alpha a_1-\beta c_{j'}$ and $a_1$ maps to
zero. Gradients in $P$ map to zero. Thus all projected features lie on two
lines and Lemma~\ref{lem:star} applies.
\end{proof}

\begin{lemma}[size-$7$ two-line reduction]\label{lem:sevenlines}
Let $|B|=7$ and suppose that no four nonzero gradients form a rank-$3$
positive circuit. Then there is a rank-$2$ positive triangle circuit of
gradients with plane $P$ such that every nonzero gradient projects into one
of two lines of $P^{\perp}$.
\end{lemma}

\begin{proof}
All extreme-ray circuits of the dependence cone of $B$ have size at most
$3$: a circuit of size $5$ has rank $4$ and positively spans with five
elements, contradicting minimality, and size $4$ is excluded by hypothesis. They
cannot all be antipodal pairs. Disjoint pairs cannot cover seven
coordinates, while two pairs sharing a vector force two elements onto one ray,
making one of them removable. Thus there is a triangle circuit $C_0=\{p,q,u\}$ with
plane $V=\pos(C_0)$.

By Lemma~\ref{lem:quot} the projections of the other four vectors to
$\R^4/V$ form a positive basis of the two-dimensional quotient of size $4$,
which by Lemma~\ref{lem:pb} consists of two antipodal projected pairs:
\[
v+\lambda r=p_1\in V,\qquad w+\mu s=p_2\in V,\qquad\lambda,\mu>0 .
\]
Represent $-p_1$ over the triangle $C_0$ minimally. A support of two
triangle vectors would give a rank-$3$ four-point circuit, which is excluded. Thus
$p_1=0$ or $p_1$ is a negative multiple of a single triangle vector, and
likewise $p_2$. If the two nonzero shifts anchored at different triangle
vectors $c_a\ne c_b$, then $-c_a\in\pos\{r,v\}$ and $-c_b\in\pos\{s,w\}$,
and the triangle relation would put the third vector $c_c$ in
$\pos\{r,v,s,w\}$, making it removable. This is a contradiction. Thus, after renaming the
shared anchor $q$,
\[
q+p+u=0,\qquad v+\lambda r+\tau q=0,\qquad w+\mu s+\sigma q=0,
\qquad\tau,\sigma\ge0 .
\]
Write $P_0=\spn(p,q)$, $P_1=\spn(q,r)$, $P_2=\spn(q,s)$, and call
$\{p,u\}$, $\{r,v\}$, $\{s,w\}$ the leaf groups. All seven basis vectors
lie in $P_0\cup P_1\cup P_2$.

Next consider any nonzero gradient outside $B$. Its minimal cone support
over $B$ (Lemma~\ref{lem:mincone}) has size at most $2$: a support of size
$4$ would give five positively spanning gradients, contradicting minimality, and
a support of size $3$ would give a rank-$3$ circuit, excluded by hypothesis. Supports
of size $1$, and two-element supports drawn from one leaf group or from a
leaf group together with $q$, give gradients inside $P_0\cup P_1\cup P_2$.
The remaining possibility is a two-element support taking its vectors from
two different leaf groups. We call such a gradient a mixed cross. A mixed
cross forms a positive triangle with its two support vectors. If some
positive triangle of the configuration lies in a plane complementary to
that cross plane, the two triangles together are six positively spanning
gradients, contradicting minimality. Thus such a mixed cross cannot occur.

Suppose $\tau,\sigma>0$. Then each leaf group together with $q$ is a
positive triangle, and for any mixed cross the unused leaf group with $q$
lies in a plane complementary to the cross plane, since $p,q$ span $V$ and
$r,s$ are independent modulo $V$. Thus no mixed cross occurs, and every nonzero
gradient lies in $P_0\cup P_1\cup P_2$. Projecting along $P_0$ sends
$P_1,P_2$ to two lines of $P_0^{\perp}$, so the central triangle with $P=P_0$
has the required property.

Suppose $\tau=\sigma=0$, so the leaves are detached antipodal pairs. A
mixed cross between $\{r,v\}$ and $\{s,w\}$ lies in a plane complementary
to that of the central triangle and is excluded. Every other support is
contained in $\{p,q,u\}\cup\{r,v\}$ or in $\{p,q,u\}\cup\{s,w\}$. The
resulting gradient need not itself lie in $P_0$ or on a pair line. But
under the projection along $P_0$ its central component vanishes, so its
image lies on the projected line of the pair it uses. Hence every nonzero
gradient projects into two lines of $P_0^{\perp}$, so again $P=P_0$ has the required property.

Suppose exactly one shift vanishes, say $\sigma=0$ and $\tau>0$. A mixed
cross between $\{s,w\}$ and $\{r,v\}$ is complementary to the central
triangle $\{q,p,u\}$, and one between $\{p,u\}$ and $\{s,w\}$ is
complementary to the leaf triangle $\{q,r,v\}$. Both are excluded. If a
mixed cross between $\{p,u\}$ and $\{r,v\}$ exists, fix one such cross
triangle $\{g_0,x,z\}$ with $x\in\{p,u\}$, $z\in\{r,v\}$, and take
$P=\spn(x,z)$. A support $\{q,s\}$ or $\{q,w\}$ would then give a triangle
complementary to this cross triangle and is excluded as well. Modulo $P$,
both nondegenerate triangles project to the line of $\pi(q)$, further
crosses of the same mixed type and every support inside $\{p,q,u,r,v\}$
project into that same line, and the detached pair projects to a second
line. Thus all nonzero gradients project into two lines of $P^{\perp}$. If
no such cross exists, every remaining support is contained in
$\{p,q,u,r,v\}$ or in $\{p,q,u\}\cup\{s,w\}$. In each case the projection
along $P_0$ kills the central component, and the image lies on the
projected $r$-line or the projected $s$-line. This again gives two lines, with
$P=P_0$.
\end{proof}

\begin{proposition}[size $7$]\label{prop:B7}
If $|B|=7$, a five-point strictly convex selection achieves
$\norm{\hat q}^2\le1$.
\end{proposition}

\begin{proof}
If some four nonzero gradients form a rank-$3$ positive circuit,
Lemma~\ref{lem:completion} gives $\norm{\hat q}^2\le E(P)\le1$ with five
points. Otherwise Lemma~\ref{lem:sevenlines} provides a triangle circuit
with plane $P$ such that every nonzero gradient, hence by
\eqref{eq:nozero4} every nonzero feature, projects into one of two lines
of $P^{\perp}$. Lemma~\ref{lem:star} completes the proof.
\end{proof}

\begin{proposition}[size $8$]\label{prop:B8}
If $|B|=8$, a five-point strictly convex selection achieves
$\norm{\hat q}^2\le\frac34$.
\end{proposition}

\begin{proof}
Minimum cardinality $8$ means no seven nonzero gradients positively span, so
Lemma~\ref{lem:rig} puts all nonzero gradients on four lines, and by
\eqref{eq:nozero4} all nonzero features as well. The covariance identity
with four rank-one groups forces orthonormal directions and unit energy per
line as in \eqref{eq:orth}. Select a canceling antipodal pair on the line of
largest residual energy and the best single point on each other line. The
offsets are $0$ and at most $R_j$. The three smaller energies sum to at most
$\frac34$.
\end{proof}

\subsection{The rectangle branch}

We now treat the rectangle branch of Definition~\ref{def:cross}. Relabel so
that the two given cross types are attached to the rays of
$a_1,a_2\in A$ and $c_1,c_2\in C$, with $A=\{a_0,a_1,a_2\}$,
$C=\{c_0,c_1,c_2\}$, normalized by
\[
a_0+a_1+a_2=0,\qquad c_0+c_1+c_2=0,
\]
and cross witnesses
\[
h+p\,a_1+q\,c_1=0,\qquad k+r\,a_2+s\,c_2=0,\qquad p,q,r,s>0 .
\]
Call $L_1=\R a_1$, $L_2=\R a_2$, $L_3=\R c_1$, $L_4=\R c_2$ the \emph{side
lines}. We say a nonzero vector $g$ is \emph{strictly $4$-generated} if
$g=\sum_{j=1}^{4}\kappa_jw_j$ with $\kappa_j>0$ and linearly independent
$w_1,\dots,w_4$ drawn from the gradients of other data points.

\begin{lemma}[off-side points]\label{lem:offside}
In the rectangle branch, assume no four gradients form a rank-$3$ positive
circuit and no five gradients form a rank-$4$ positive circuit. Then every
nonzero gradient not on a side line is strictly $4$-generated.
\end{lemma}

\begin{proof}
First consider the basis vectors and cross points. Substituting the triangle
relations into the cross relations gives
\[
a_0=\tfrac1p\,h+\tfrac1r\,k+\tfrac qp\,c_1+\tfrac sr\,c_2,
\qquad
c_0=\tfrac1q\,h+\tfrac1s\,k+\tfrac pq\,a_1+\tfrac rs\,a_2,
\]
and the generators are independent: in the basis
$(a_1,a_2,c_1,c_2)$ the matrix of $(h,k,c_1,c_2)$ has determinant
$pr\ne0$, and symmetrically for $c_0$. More generally, any cross gradient
$g$ with $g+x\,a_i+y\,c_j=0$, $x,y>0$, satisfies
\[
g=x\,a_{i'}+x\,a_{i''}+y\,c_{j'}+y\,c_{j''},
\]
where $\{i,i',i''\}=\{0,1,2\}$ and $\{j,j',j''\}=\{0,1,2\}$. The four
generators are independent: each complementary plane contains an
independent pair. This covers every cross type, including any beyond the two
witnesses.

Next consider the in-plane points. We give the $A$-plane case; the $C$-plane case is
symmetric. Let $g$ be a nonzero gradient in the $A$-plane and off the side
lines. Take its minimal cone support over $B$
using Lemma~\ref{lem:mincone}, of size at most $4$. Size $4$ would give a
rank-$4$ five-point positive circuit and size $3$ a rank-$3$ circuit, both
excluded by hypothesis, so the support has at most two elements. It cannot
involve an element of $C$: the planes are complementary, so if
$-g=\alpha a_i+\beta c_j$ with $\beta>0$, then
$\beta c_j=-g-\alpha a_i$ lies in the $A$-plane, forcing $c_j=0$. A support
inside $C$ alone would put $g$ in both planes, giving $g=0$. Thus the support
is a one- or two-element subset of $\{a_0,a_1,a_2\}$. One-element
supports on $a_1$ or $a_2$ are on side lines, leaving the following cases.

If $g=-x\,a_1-y\,a_2$ with $x,y>0$:
\[
g=\tfrac xp\,h+\tfrac yr\,k+\tfrac{xq}p\,c_1+\tfrac{ys}r\,c_2,
\]
strictly $4$-generated with the same independent generators.

If $g=-x\,a_0-y\,a_1=(x-y)a_1+x\,a_2$ with $x,y>0$: for $x<y$,
\[
r\,g+x\,k+r(y-x)\,a_1+xs\,c_2=0
\]
is a rank-$3$ positive circuit, excluded by hypothesis. For $x=y$, $g$ lies
on $L_2$. For $x>y$,
\[
g+\tfrac{x-y}{p}\,h+\tfrac xr\,k+\tfrac{(x-y)q}{p}\,c_1+\tfrac{xs}{r}\,c_2=0
\]
is a five-point positive circuit of rank $4$, excluded. The support
$\{a_0\}$ alone, $g=x\,a_1+x\,a_2$, is the boundary case $y=0$ of the last
identity and is likewise excluded. The case $g=-x\,a_0-y\,a_2$ is
symmetric, with $L_1$ as the ray boundary. Hence in the remaining
configuration every off-side nonzero gradient is strictly $4$-generated.
\end{proof}

The excluded cases in Lemma~\ref{lem:offside} already give the desired bound: a rank-$3$
circuit yields $\norm{\hat q}^2\le1$ by Lemma~\ref{lem:completion}, and a
rank-$4$ five-point circuit yields exact recovery by
Lemma~\ref{lem:cert}. It remains to handle the configuration in which every
off-side point is strictly $4$-generated.

We pass to projective coordinates. In the branches below every point with
$\xi_i\ne0$ has $\eta_i\ne0$ by \eqref{eq:nozero4}. Write
$v_i=\xi_i/\eta_i$, and $v_i=0$ for pure-residual points. The constraint
$\ip{q}{\xi_i}=\eta_i$ reads $\ip{q}{v_i}=1$, and $g_i=\eta_i^2v_i$, so
$v_i$ spans the same ray as $g_i$.

\begin{lemma}[sign-cone realization]\label{lem:signcone}
Let $v_0=\sum_{j=1}^{d}\kappa_jv_j$ with $\kappa_j>0$ and
$v_1,\dots,v_d\in\R^d$ linearly independent, all being projective points of
the system. Let $q$ satisfy
\[
\ip{q}{v_0}>1,\qquad \ip{q}{v_j}<1\quad(j=1,\dots,d).
\]
Then $q$ is the unique minimizer of a strictly convex, strictly positively
weighted objective on these $d+1$ points.
\end{lemma}

\begin{proof}
Write $e_i(q)=\ip{q}{v_i}-1$. The stationarity condition
$\sum_i\alpha_i\eta_i^2\,e_i(q)\,v_i=0$ holds iff the coefficient vector
$(\alpha_i\eta_i^2e_i)_i$ is proportional to the signed dependence
$(1,-\kappa_1,\dots,-\kappa_d)$ of $v_0-\sum\kappa_jv_j=0$. Setting
$\alpha_0=t/(\eta_0^2e_0)$ and $\alpha_j=t\kappa_j/(\eta_j^2(-e_j))$ with
any $t>0$ gives strictly positive weights, since $e_0>0>e_j$. The Hessian
$\sum\alpha_i\eta_i^2v_iv_i^{\top}$ is positive definite because
$v_1,\dots,v_d$ are independent.
\end{proof}

\begin{lemma}[extremal basis]\label{lem:extremal}
Let a whitened system in $\R^d$ satisfy $g_i=0\Rightarrow\xi_i=0$ and admit
$d$ linearly independent projective points $p_1,\dots,p_d$ with the
following property. Writing $q_B$ for the unique solution of
$\ip{q_B}{p_j}=1$ and $z_i=\ip{q_B}{v_i}$ for all points, every point with
$|z_i|>1$ admits a representation $v_i=\sum_{j=1}^{d}\kappa_jw_j$ with
$\kappa_j>0$ and $w_1,\dots,w_d$ linearly independent projective points of
other data points. Then the infimum of $\norm{\hat q}^2$ over strictly
convex selections of at most $d+1$ points is at most $1$.
\end{lemma}

\begin{proof}
For a pure-residual point $\ip{q_B}{\xi_i}=0=\eta_iz_i$ defines $z_i=0$. The
covariance identity gives
\[
\tau:=\norm{q_B}^2
=\frac1N\sum_i\ip{q_B}{\xi_i}^2
=\frac1N\sum_i\eta_i^2z_i^2
\le L^2,\qquad L=\max_i|z_i| .
\]
If $\tau\le1$, select $p_1,\dots,p_d$ with any positive weights: the
objective interpolates all $d$ constraints at $q_B$, is strictly convex, and
its unique minimizer $q_B$ has $\norm{q_B}^2=\tau\le1$. This uses $d$
points.

If $\tau>1$ then $L\ge\sqrt\tau>1$. Pick $i$ with $|z_i|=L$ and let
$\sigma=\mathrm{sgn}(z_i)$. By hypothesis,
$v_i=\sum_j\kappa_jw_j$, $\kappa_j>0$, with independent generators
$w_1,\dots,w_d$ among the other points. Set $q_\ast=\sigma q_B/L$. Then
$\ip{q_\ast}{v_i}=1$ and, for each generator,
$\ip{q_\ast}{w_j}\le|\ip{q_B}{w_j}|/L\le1$, since $L$ bounds every
evaluation. Let $K=\sum_j\kappa_j$. Pairing the generation identity with
$q_\ast$ gives $1=\sum_j\kappa_j\ip{q_\ast}{w_j}\le K$.

If $K=1$, all $\ip{q_\ast}{w_j}=1$, so $q_\ast$ interpolates the anchor and
all generators. Any positive weights on these $d+1$ points give a strictly
convex objective with unique minimizer $q_\ast$, and
$\norm{q_\ast}^2=\tau/L^2\le1$.

If $K>1$, the strict sign cone
$\{q:\ip{q}{v_i}>1,\ \ip{q}{w_j}<1\ \forall j\}$ is nonempty: the point
$q_c$ with $\ip{q_c}{w_j}=c$ for all $j$ and $c\in(1/K,1)$ has anchor value
$Kc>1$. Every point of the segment $(1-t)q_\ast+t\,q_c$, $t\in(0,1]$, lies
in the strict cone, and by Lemma~\ref{lem:signcone} each is the unique
output of a strictly convex $(d+1)$-point objective. Letting $t\downarrow0$,
the achievable risks converge to $1+\norm{q_\ast}^2\le2$. The definition of
$\Ls(n;\Astar)$ takes an infimum, which proves the claim.
\end{proof}

\begin{remark}[finite attainment]\label{rem:attained}
In the case $\tau>1$ the bound is attained, with room to spare, by a
finite selection. The calibrated basis points have $z_j=1<L$ and nonzero
residual, since their gradients are nonzero, so the inequality
$\tau=\frac1N\sum_i\eta_i^2z_i^2\le L^2$ is strict, $\norm{q_\ast}^2=\tau/L^2<1$,
and by continuity $\norm{(1-t)q_\ast+tq_c}^2<1$ for all small $t>0$. Hence
in Lemma~\ref{lem:extremal} some strictly convex selection of at most $d+1$
points has $\norm{\hat q}^2\le1$ outright, and every branch of
Theorem~\ref{thm:fw45} other than the anchored one is closed by a finite
selection.
\end{remark}

\begin{proposition}[rectangle closure]\label{prop:rect}
In the rectangle branch, under \eqref{eq:nozero4}, either some selection of
at most $5$ points achieves $\norm{\hat q}^2\le1$ outright, or the extremal
basis hypothesis holds and the infimum over five-point selections of the
excess is at most $1$.
\end{proposition}

\begin{proof}
If some four gradients form a rank-$3$ positive circuit or some five form a
rank-$4$ circuit, Lemma~\ref{lem:completion} or Lemma~\ref{lem:cert}
gives the result. Otherwise Lemma~\ref{lem:offside} applies: every off-side nonzero
point is strictly $4$-generated. On each side line $L_l$ fix a reference
direction and pick the data point $p_l$ on $L_l$ of maximal $|\theta|$,
where $v=\theta\cdot(\text{reference})$. The four side lines carry the basis
rays $a_1,a_2,c_1,c_2$, so each holds at least one point. The four
chosen points are independent because $a_1,a_2$ span the $A$-plane and
$c_1,c_2$ the complementary $C$-plane. Any point on a side line satisfies
$|z|=|\theta/\theta_{\max}|\le1$, and pure-residual points have $z=0$. Thus
every point with $|z|>1$ is off-side, hence strictly $4$-generated, and
Lemma~\ref{lem:extremal} with $d=4$ applies.
\end{proof}

\subsection{The main theorem for $(4,5)$}

\begin{theorem}\label{thm:fw45}
$\Fw(4,5)=2$.
\end{theorem}

\begin{proof}
The lower bound is Theorem~\ref{thm:gammalb}. For the upper bound, let $D$
be arbitrary. If the feature rank is $\rho\le2$, budget $5\ge2\rho$ gives
ratio $1$. If $\rho=3$, the problem is isometric to dimension $3$ with
budget $5=2\cdot3-1$, and Theorem~\ref{thm:endpoint-ub} gives
$\frac43<2$. If $\Ls=0$, four independent features suffice. Otherwise
whiten. If some point has $\eta_i=0\ne\xi_i$,
Corollary~\ref{cor:anchor4} gives $\frac53+\varepsilon<2$. So assume
\eqref{eq:nozero4}. As before, the nonzero gradients span, hence positively
span, $\R^4$. Take a minimum-cardinality positively spanning subset $B$.

$|B|=5$: Proposition~\ref{prop:B5} gives excess $0$. $|B|=6$:
Proposition~\ref{prop:B6} either finishes with excess at most $1$ or leaves
two complementary triangles $B=A\cup C$ with planes $U,V$. In the latter
case, first reduce the gradients outside $B$ by their minimal cone supports
over $B$ (Lemma~\ref{lem:mincone}): a support of size $4$ gives a rank-$4$
five-point positive circuit and exact recovery by Lemma~\ref{lem:cert}. A
support of size $3$ gives a rank-$3$ circuit and
Lemma~\ref{lem:completion} gives excess at most $1$. If neither
occurs, every extra gradient has support of size at most $2$. Supports of
size $1$, or of size $2$ inside one triangle, lie in $U\cup V$. The
mixed two-element supports are exactly the cross gradients of
Definition~\ref{def:cross}. We now have an exhaustive trichotomy. With no cross
gradients, all nonzero gradients lie in $U\cup V$ and
Proposition~\ref{prop:22} applies. With cross gradients all sharing an
endpoint, first note that any four gradients forming a rank-$3$ positive
circuit would give the result by Lemma~\ref{lem:completion}. Also, no five
gradients positively span because $B$ has minimum cardinality $6$. Thus
Proposition~\ref{prop:shared} applies. With two non-sharing cross types,
apply Proposition~\ref{prop:rect}. $|B|=7$: apply Proposition~\ref{prop:B7}. $|B|=8$:
apply Proposition~\ref{prop:B8}.

Every branch gives excess at most $1$ by a finite strictly convex selection
(Remark~\ref{rem:attained} for the rectangle branch), except the anchored
branch, which gives $\frac53+\varepsilon$ for every $\varepsilon>0$. Thus
$\Ls(5;\Astar)\le2\Ls$.
\end{proof}

\section{A finite toolkit: Gale coordinates, adjacent bases, composition}\label{sec:toolkit}

This section develops four tools of independent use. The first removes the
ambient geometry from selector problems. The second eliminates weight
optimization at budget $d+1$. The third computes direct sums exactly, and
the fourth settles bases of size $d+2$ under a coverage hypothesis.

\subsection{Gale coordinates}

\begin{lemma}[Gale--metric reduction]\label{lem:gale}
Let a whitened system have all points with $\xi_i\ne0,\eta_i\ne0$ or
$\xi_i=0$, write $c_i=\eta_i^2$, $G=[g_1\cdots g_N]$ and
$C=\mathrm{diag}(c_i)$, where pure-residual points contribute zero columns;
points with $\xi_i=0=\eta_i$ carry nothing and are discarded, so all
$c_i>0$, while $N$ keeps its meaning as the normalization constant of the
original system and may exceed the number of remaining points. The three
identities \eqref{eq:threeid} are equivalent to
\[
\textstyle\sum_ic_i=N,\qquad G\mathbf1=0,\qquad GC^{-1}G^{\top}=N\,I_d .
\]
Let $K=\ker G$, so $\mathbf1\in K$. For any $q$ the evaluation vector
$s=G^{\top}q$ lies in $K^{\perp}$, satisfies
$N\norm{q}^2=s^{\top}C^{-1}s$, and inverts as $q=\frac1NGC^{-1}s$. A point
$q$ is a minimizer of a strictly positive weighting $(\alpha_i)_{i\in S}$
on a support $S$ if and only if the vector $h$ with
$h_i=\frac{\alpha_i}{c_i}(s_i-c_i)$ on $S$ and $0$ elsewhere lies in $K$;
the objective is strictly convex, so that $q$ is then its unique minimizer
and the learner's output, if and only if $\mathrm{rank}\,G_S=d$. On a
rank-deficient support the minimizers form an affine subspace and the
min-norm rule picks one of them, which need not be $q$. Consequently the
full-rank selector profile of a whitened system depends only on the pair
$(K,C)$, not on the ambient coordinates of the gradients.
\end{lemma}

\begin{proof}
The three identities translate directly: $\sum c_i=N$ is the energy
identity, $G\mathbf1=\sum_i\eta_i\xi_i=0$ is the normal
equation, and $GC^{-1}G^{\top}=\sum_i\xi_i\xi_i^{\top}=NI_d$ is the
covariance. For $s=G^{\top}q$, membership in $K^{\perp}$ is immediate.
Moreover,
$s^{\top}C^{-1}s=q^{\top}GC^{-1}G^{\top}q=N\norm{q}^2$, and
$\frac1NGC^{-1}s=\frac1NGC^{-1}G^{\top}q=q$. For a point with
$\xi_i\ne0$ we have $\xi_i=g_i/\eta_i$ and $\ip{q}{\xi_i}=s_i/\eta_i$, so
$(\ip{q}{\xi_i}-\eta_i)\,\xi_i=\frac{s_i-c_i}{c_i}\,g_i$. Stationarity of
$\sum_{i\in S}\alpha_i(\ip{q}{\xi_i}-\eta_i)^2$ therefore reads
$\sum_{i\in S}\frac{\alpha_i}{c_i}(s_i-c_i)\,g_i=0$, which is exactly
$Gh=0$. The Hessian is
$2\sum_{i\in S}\alpha_i\xi_i\xi_i^{\top}=2\,G_S\,\mathrm{diag}(\alpha_i/c_i)\,G_S^{\top}$.
\end{proof}

\subsection{Adjacent bases}

\begin{lemma}[adjacent-basis simplex]\label{lem:adjacent}
Let $T$ be $d+1$ points whose feature matrix $X_T\in\R^{(d+1)\times d}$ has
full column rank, and let $h$ span $\ker X_T^{\top}$. For each $i$ with
$h_i\ne0$, deleting point $i$ leaves $d$ independent features; let $q_i$ be
their unique interpolation point. Then the closure of the set of outputs of
strictly positively weighted least squares objectives supported on $T$ is
exactly
\[
\mathrm{conv}\{q_i:h_i\ne0\}.
\]
\end{lemma}

\begin{proof}
Write $r(q)=X_Tq-y_T$, an affine map. At $q_i$ all residuals vanish except
the $i$th. The identity $h^{\top}r(q_i)=-h^{\top}y_T=:C_0$ gives
$r(q_i)=(C_0/h_i)e_i$. For a convex combination $q=\sum_i\lambda_iq_i$,
affinity gives $r_i(q)=C_0\lambda_i/h_i$. Stationarity under weights
$\alpha$ means $\mathrm{diag}(\alpha)r(q)\in\ker X_T^{\top}=\R h$. That is,
$\alpha_iC_0\lambda_i/h_i=\mu h_i$ for all $i$, or equivalently
$\lambda_i\propto h_i^2/\alpha_i$. Every choice of strictly positive
weights produces strictly positive $\lambda$. Conversely, any strictly
positive $\lambda$ is realized by $\alpha_i\propto h_i^2/\lambda_i$.
Boundary points of the simplex arise as limits $\lambda_i\downarrow0$.
Points with $h_i=0$ are interpolated by every $q_j$ and impose no
constraint. If $C_0=0$ all $q_i$ coincide and the statement is trivial.
\end{proof}

At budget $d+1$ the continuous weight optimization on a fixed support
therefore reduces to a finite computation: enumerate the adjacent
interpolation points and minimize the risk over their convex hull.

\subsection{Circuit--quotient lifting}

\begin{lemma}[lifting]\label{lem:lift}
Let $C$ be a positive circuit of gradients of rank $r$ with
$P=\spn\{g_i:i\in C\}$. Suppose the projected system on $P^{\perp}$ admits
a selection of $b$ points whose objective is strictly convex on
$P^{\perp}$ with unique minimizer $u\in P^{\perp}$. Then the original
system admits a strictly convex, strictly positively weighted selection of
at most $(r+1)+b$ points with unique minimizer $u$.
\end{lemma}

\begin{proof}
Let the quotient selection have weights $\beta_j$ on the points $j\in T$.
At $u$, the gradient of the quotient objective vanishes along $P^{\perp}$,
so the vector $h=\sum_{j\in T}\beta_j(\ip{u}{\xi_j}-\eta_j)\,\xi_j$, the
full gradient of the quotient part with the factor $2$ omitted, lies in
$P$. The circuit positively spans $P$, so $h=\sum_{i\in C}\gamma_ig_i$ with
$\gamma_i\ge0$. Adding a large multiple of the circuit's strictly positive
zero-sum relation makes all circuit weights $\gamma_i'$ strictly positive
without changing the sum. At $u$ the circuit residuals are $-\eta_i$, since
$\xi_i\in P\perp u$, so the circuit contributes
$-\sum_{i\in C}\gamma_i'g_i=-h$ to the gradient. This cancels the
quotient part and establishes total stationarity. The circuit features span $P$ and the quotient features span
$P^{\perp}$ after projection, so the joint Hessian is positive definite.
\end{proof}

Lemma~\ref{lem:lift} can be combined with an interface bound on the quotient.
Suppose the cell $(d-r,\,d-r+k-1)$ admits a proved whitened interface with
excess at most $\Gam_{d-r,k-1}$, as it does for $k-1=0$ by
Lemma~\ref{lem:L2} and for the cells proved in this paper. Completing a
circuit of rank $r$ and applying that interface on the energy-normalized
quotient yields $\norm{\hat q}^2\le\Gam_{d-r,k-1}\,E(P)$ with
$(r+1)+(d-r)+(k-1)=d+k$ points. This generalizes Lemma~\ref{lem:completion}.
For cells whose value is only conjectured, it remains a conditional
statement.

\subsection{Composition}

\begin{theorem}[orthogonal composition allocation law]\label{thm:alloc}
Let a whitened system split as an orthogonal direct sum of components
$S_1,\dots,S_m$ with $\dim S_j=d_j$, $\sum_jd_j=d$, meaning that every
nonzero feature lies in some $S_j$ and the $S_j$ are mutually orthogonal.
Let $R_j$ be the residual energy of component $j$ and let
$\delta_j(d_j+\ell)$ denote the optimal excess, an infimum over strictly
convex selections, of the energy-normalized component using $d_j+\ell$
points; a component with $R_j=0$ has all labels zero, is recovered exactly
by any $d_j$ independent features, and gets $\delta_j\equiv0$. Then the
optimal strictly convex excess of the whole system at budget $d+k$ is
\[
\Delta(d+k)=\min_{\substack{\ell_j\ge0\\ \sum_j\ell_j\le k}}
\ \sum_{j=1}^mR_j\,\delta_j(d_j+\ell_j).
\]
Moreover, realizing the components as datasets with generic large
parameter shifts turns this value into the exact min-over-selections: any
selection using fewer than $d_j$ independent features in some component
leaves a shifted direction to the min-norm rule and its risk diverges with
the shift.
\end{theorem}

\begin{proof}
Orthogonality separates the objectives and outputs. A strictly convex
selection restricted to a component is a strictly convex component
selection. Outputs add orthogonally, and excesses add after energy
normalization. Component weights may be scaled independently and merged by
a global normalization, so any tuple of $\varepsilon$-optimal component
selections with $\sum_j(d_j+\ell_j)\le d+k$ is jointly feasible and the
infima add. Conversely, any strictly convex selection induces such a tuple.
This proves the displayed identity. For the second claim, take the
self-realized dataset and shift its labels by $y_i\mapsto y_i+M\ip{h}{x_i}$
with a unit vector $h$ outside the finitely many feature spans of
rank-deficient supports, and let $\sigma>0$ be the least distance from $h$
to these spans. The full-data optimum becomes $Mh$ and the covariance is
the identity, so the output of a rank-deficient support, which lies in its
feature span, has excess at least $M^2\sigma^2$. The excess of every
full-rank selection is unchanged by the shift, and for budget at least $d$
some full-rank selection of $d$ points has excess at most $d$ by
Lemma~\ref{lem:L2}. Hence for $M>\sqrt d/\sigma$ no rank-deficient
selection can be optimal, and the identity computes the true optimum, an
infimum over strictly convex selections.
\end{proof}

\subsection{Bases of size $d+2$}

\begin{theorem}[two-circuit bases]\label{thm:twocircuit}
Let the minimum-cardinality positively spanning set $B$ of the nonzero
gradients have $d+2$ elements, and let $C_1,C_2$ be the two extreme-ray
circuits of its dependence cone, with flats $P_\ell=\spn C_\ell$ of ranks
$r_\ell$. If every nonzero feature lies in $P_1\cup P_2$, then at budget
$d+1$,
\[
\norm{\hat q}^2\ \le\ C_{d,2}\ =\ \Gam_{d,1}.
\]
\end{theorem}

\begin{proof}
The two circuit supports cover $B$, since the strictly positive dependence
of $B$ is a positive combination of the two rays. Let $t=|C_1\cap C_2|$,
$p=d-r_1$, $q=d-r_2$; counting $|C_1|+|C_2|-t=d+2$ with
$|C_\ell|=r_\ell+1$ gives $p+q=d-t$. Lemma~\ref{lem:completion} applied to
each circuit gives $\norm{\hat q}^2\le p\,E(P_1)$ and
$\norm{\hat q}^2\le q\,E(P_2)$ with $d+1$ points. Under the coverage
hypothesis every nonzero feature contributes to at most one of the outside
energies, and features in $P_1\cap P_2$ and pure-residual points to
neither, so $E(P_1)+E(P_2)\le1$. Hence
\[
\norm{\hat q}^2\le\min\{pE(P_1),\,qE(P_2)\}
\le\frac{pq}{p+q}\le C_{p+q,2}=C_{d-t,2}\le C_{d,2},
\]
using that $ab/(a+b)$ is maximized by balanced splits and that
$C_{d,2}=\lfloor d/2\rfloor\lceil d/2\rceil/d$ is nondecreasing in $d$.
Finally $C_{d,2}=\Gam_{d,1}$: the term $s=2$ of \eqref{eq:gamma} is
$C_{d,2}$ itself, and for $3\le s\le d$ the AM--HM inequality
$\sum_{j\le s}1/r_j\ge s^2/d$ gives
$(s-1)C_{d,s}\le(s-1)d/s^2\le 2d/9\le(d^2-1)/(4d)\le C_{d,2}$, the third
step because $d\ge3$.
\end{proof}

The coverage hypothesis excludes only points whose features lie outside both
flats: cross ears and zero-residual anchors outside the flats. The
treatment of cross ears in general dimension remains open, and
Section~\ref{sec:obstructions} shows that some such hypothesis is
necessary for any bound by $1$ alone.

\section{The budget $6$ in dimension $4$}\label{sec:fw46}

This section establishes a third exact cell beyond the endpoint line. It is
the smallest cell with $k\ge2$.

\begin{theorem}\label{thm:fw46}
$\Fw(4,6)=\frac32$.
\end{theorem}

\begin{proof}
The lower bound is Theorem~\ref{thm:gammalb} at $(d,k)=(4,2)$, where
$\Gam_{4,2}=\max\{C_{4,3},\,2C_{4,4}\}=\max\{\frac25,\frac12\}=\frac12$.

For the upper bound, take any $D$. If the feature rank is
$\rho\le3$, the budget satisfies $6\ge2\rho$ and the official
$n\ge2\rho$ regime gives ratio $1$. If $\Ls=0$, four independent features
give exact recovery. Otherwise whiten. If some point has $\eta_0=0\ne\xi_0$, project to
$\xi_0^{\perp}\cong\R^3$ and apply the whitened endpoint interface
(Theorem~\ref{thm:wend} with $d=3$), which uses at most $5$ points and has
excess at most $\frac13$. Penalizing the anchor back
(Lemma~\ref{lem:penalty}) gives, for every $0<\varepsilon<\frac16$, a
six-point strictly convex selection with excess at most
$\frac13+\varepsilon<\frac12$.
We may therefore assume $g_i=0\Rightarrow\xi_i=0$. As before, the nonzero gradients
positively span $\R^4$. Let $B$ be a minimum-cardinality positively
spanning subset. If $|B|\le6$, selecting all of $B$ recovers exactly
(Lemma~\ref{lem:l345}(1) and Lemma~\ref{lem:cert}). By
Lemma~\ref{lem:pb} this leaves $|B|\in\{7,8\}$.

\emph{Case $|B|=7$ with a rank-$3$ positive circuit $C$.} Let
$P=\spn\{g_i:i\in C\}$ and project to the line $P^{\perp}$. The projected
gradients sum to zero. They cannot all vanish: otherwise all nonzero
gradients, and hence all nonzero features by the anchor-free assumption, would
lie in $P$, contradicting full feature rank. Thus both signs occur among the
nonzero projections. Choose one point of each sign with weights canceling
their projections; the quotient objective of this pair is strictly convex
on the line $P^{\perp}$ with unique minimizer $0$. Lemma~\ref{lem:lift}
lifts this to a strictly convex selection on $(3+1)+2=6$ points with
unique output $0$, giving exact recovery. These six gradients positively span
$\R^4$, so this case cannot occur when $|B|=7$ is minimal. We retain it as a
harmless certificate.

\emph{Case $|B|=7$ without rank-$3$ circuits.} Lemma~\ref{lem:sevenlines}
provides a rank-$2$ positive triangle circuit with plane $P$ such that all
nonzero gradients, and hence all nonzero
features by the anchor-free assumption,
project into two lines of $P^{\perp}$. As in \eqref{eq:orth},
the projected covariance forces the two lines to be orthogonal, each with unit feature
energy. Let $R_1\ge R_2$ be their residual energies, so
$R_1+R_2\le1$. Select the triangle ($3$ points), a projection-canceling
pair on the heavy line ($2$ points),
and the best interpolated single point on the light line ($1$ point),
where the usual pigeonhole argument gives a single point with squared offset at most
$R_2$. The pair's residual gradient sum at the candidate output lies in
$P$ and is absorbed by the circuit weights exactly as in
Lemma~\ref{lem:lift}. The single point is interpolated, and the features span
$\R^4$. The unique output lies on the light line, so
$\norm{\hat q}^2\le R_2\le\frac{R_1+R_2}2\le\frac12$, with $6$ points.

\emph{Case $|B|=8$.} Minimum cardinality $8$ means no seven gradients
positively span, so Lemma~\ref{lem:rig} puts all nonzero gradients, and by
the anchor-free assumption all nonzero features, on four lines. The
covariance identity forces them to be orthonormal, each with unit energy, as in
\eqref{eq:orth}. Order the residual energies $R_1\ge\dots\ge R_4$. Select
canceling pairs on the two heaviest lines and best single points on the
two lightest. This uses $2+2+1+1=6$ points and gives
$\norm{\hat q}^2\le R_3+R_4\le\frac12\sum_jR_j\le\frac12$.

Every branch yields a strictly convex selection of at most $6$ points with
excess at most $\frac12$, or the anchored branch with
$\frac13+\varepsilon$; by \eqref{eq:ratio} and the infimum,
$\Ls(6;\Astar)\le\frac32\Ls$.
\end{proof}

\begin{remark}
Together with Theorems~\ref{thm:fw34} and~\ref{thm:fw45} and
Corollary~\ref{cor:endpoint}, the proved cells of
Conjecture~\ref{conj:gamma} are now the full line $k=d-1$, the cells
$(3,4)$ and $(4,5)$ at $k=1$, and the cell $(4,6)$ at $k=2$. The smallest
open cell is $(d,n)=(5,6)$.
\end{remark}

\section{Three obstructions}\label{sec:obstructions}

The proof of Theorem~\ref{thm:fw45} uses sign-cone geometry rather
than any of three shorter routes that might seem promising at first. This
section gives explicit counterexamples to each route.
Every numerical claim in this section is a finite exact rational
computation. The proofs provide the derivations, and the
supplementary material includes scripts that re-verify
each claimed fraction in exact arithmetic.

\subsection{Low-dimensional circuit covers fail}

Lemma~\ref{lem:mincone} controls every extra gradient by a low-rank positive
circuit. One might hope to choose these circuit flats with
small total dimension and then finish with a partition-style argument as in
Section~\ref{sec:block}. This already fails at the minimal cell.

\begin{proposition}[rectangle obstruction]\label{prop:rectobs}
In $\R^4$ let
\[
\begin{aligned}
&a_1=e_1,\quad a_2=e_2,\quad a_3=-e_1-e_2,\qquad
b_1=e_3,\quad b_2=e_4,\quad b_3=-e_3-e_4,\\
&h_1=-e_1-e_3,\qquad h_2=-e_2-e_4,
\end{aligned}
\]
and $G=\{a_1,a_2,a_3,b_1,b_2,b_3,h_1,h_2\}$. Then $G$ positively spans
$\R^4$; no subset of at most five vectors positively spans; the positive
circuits of $G$ are exactly
$\{a_1,a_2,a_3\}$, $\{b_1,b_2,b_3\}$, $\{a_1,b_1,h_1\}$,
$\{a_2,b_2,h_2\}$; and any cover of $G$ by circuit flats needs four planes,
of total dimension $8$.
\end{proposition}

\begin{proof}
A nonnegative dependence $\sum\lambda_ia_i+\sum\mu_ib_i+\nu_1h_1+\nu_2h_2=0$
forces, coordinate by coordinate,
$\lambda_1=\lambda_3+\nu_1$, $\lambda_2=\lambda_3+\nu_2$,
$\mu_1=\mu_3+\nu_1$, $\mu_2=\mu_3+\nu_2$. The dependence cone is therefore
generated by four extreme rays, obtained by setting one of
$\lambda_3,\mu_3,\nu_1,\nu_2$ to one and the others to zero. These give the
four listed triangles, with no other positive circuits. A
positively spanning subset must support a strictly positive dependence whose
support spans $\R^4$. Single circuits have rank $2$, two circuits sharing a
direction span at most rank $3$, and the two complementary unions
$\{a_1,a_2,a_3,b_1,b_2,b_3\}$ and $\{a_1,b_1,h_1,a_2,b_2,h_2\}$ have six
elements each. Finally $a_3,b_3,h_1,h_2$ each belong to exactly one circuit
plane, so all four planes are needed.
\end{proof}

The associated regression instance is harmless for the risk question: with
features $2e_1,2e_2,-e_1-e_2,2e_3,2e_4,-e_3-e_4,-e_1-e_3,-e_2-e_4$ and all
labels $1$, the rows sum to zero, $w^{\star}=0$, $\Ls=1$, and the
five-point selection of the first five rows with weights
$(\frac3{16},\frac3{16},\frac12,\frac1{16},\frac1{16})$ outputs
$(-\frac1{14},-\frac1{14},\frac12,\frac12)$ with full risk $\frac{10}7<2$.
The obstruction concerns proof technique, not the value of $\Fw(4,5)$.

\subsection{Anchor disjunctions fail}

A second route interpolates a few low-cost points exactly and recurses in
the orthogonal complement. The next two lemmas establish valid steps along
this route. Neither is used in
Section~\ref{sec:fw45}, for the reason explained after the counterexample.

\begin{proposition}[single anchor]\label{prop:anchor1}
For a whitened system in $\R^4$ and a point $i$ with $\xi_i\ne0$, let
$t_i=\eta_i^2/\norm{\xi_i}^2$. If $t_i\le\frac15$, then for every
$\varepsilon>0$ some five-point strictly convex selection has
$1+\norm{\hat q}^2\le\frac53(1+t_i)+\varepsilon\le2+\varepsilon$.
\end{proposition}

\begin{proposition}[double anchor]\label{prop:anchor2}
For independent features $\xi_i,\xi_j$, let $q_{ij}$ be the minimum-norm
solution of $\ip{q}{\xi_i}=\eta_i$, $\ip{q}{\xi_j}=\eta_j$ and
$t_{ij}=\norm{q_{ij}}^2$. If $t_{ij}\le\frac13$, then for every
$\varepsilon>0$ some five-point strictly convex selection has
$1+\norm{\hat q}^2\le\frac32(1+t_{ij})+\varepsilon\le2+\varepsilon$.
\end{proposition}

\begin{proof}[Proof of both]
Write $W$ for the orthogonal complement of the anchored features and
$\zeta_k=P_W\xi_k$, $e_k=\eta_k-\ip{q_{\mathrm a}}{\xi_k}$, where
$q_{\mathrm a}$ is the min-norm interpolant of the anchors. Direct expansion
from \eqref{eq:threeid} gives
$\frac1N\sum_k\zeta_k\zeta_k^{\top}=I_W$,
$\frac1N\sum_ke_k\zeta_k=0$ and
$\frac1N\sum_ke_k^2=1+t$, where $t$ is the anchor cost. The cross term
uses the fact that $q_{\mathrm a}$ lies in the span of the anchored features.
Rescaling $e_k$ by $(1+t)^{-1/2}$ yields a whitened system on $W$. Apply
Theorem~\ref{thm:L13D} on $W\cong\R^3$, or Theorem~\ref{thm:wend} with
$d=2$ on $W\cong\R^2$, and penalize the anchors back with
Lemma~\ref{lem:penalty}. The counts are $4+1$ and $3+2$.
\end{proof}

Given these lemmas and Lemma~\ref{lem:completion}, the statement
\begin{equation}\label{eq:R}
\max_{P}R(P)\ge\tfrac12
\qquad\text{or}\qquad
\min_{i,j}t_{ij}\le\tfrac13,
\tag{R}
\end{equation}
quantified over all rectangle systems, with $P$ ranging over the circuit
planes and $i,j$ over independent feature pairs, would close
$\Fw(4,5)=2$ in a few lines. It is false.

Both counterexamples below are specified by directions and masses. The
following standard construction converts these data into a whitened system.

\begin{lemma}[projective realization]\label{lem:realize}
Let $s_i>0$ with $\sum_is_i=1$ and let nonzero vectors $d_i\in\R^d$ satisfy
$\sum_is_id_i=0$ and $M=\sum_is_id_id_i^{\top}\succ0$. Set
\[
v_i=M^{-1/2}d_i,\qquad \eta_i=\sqrt{Ns_i},\qquad \xi_i=\eta_iv_i .
\]
Then $\{(\xi_i,\eta_i)\}$ satisfies the three identities
\eqref{eq:threeid}; moreover $\eta_i^2=Ns_i$, each $v_i$ spans the image of
the ray of $d_i$ under the invertible map $M^{-1/2}$, and
$\ip{v_i}{v_j}=d_i^{\top}M^{-1}d_j$. Linear dependences among the $d_i$
transfer verbatim to the $v_i$; since $g_i=Ns_iv_i$ with $s_i>0$, a
dependence $\sum_ic_id_i=0$ becomes $\sum_i(c_i/s_i)\,g_i=0$, so circuit
supports, coefficient signs and positive spanning properties transfer to
the gradients after this positive rescaling of coefficients.
\end{lemma}

\begin{proof}
Direct substitution:
$\frac1N\sum_i\xi_i\xi_i^{\top}=\sum_is_iM^{-1/2}d_id_i^{\top}M^{-1/2}=I$,
then $\frac1N\sum_i\eta_i\xi_i=\sum_is_iM^{-1/2}d_i=0$, and
$\frac1N\sum_i\eta_i^2=\sum_is_i=1$.
\end{proof}

Self-realization also makes each system below a legal dataset. In
its original coordinates, the min-norm rule is the whitened-coordinate
Euclidean rule, so single-point selector outputs can be computed directly.

\begin{proposition}[the $(\mathrm R)$ counterexample]\label{prop:Rfail}
Let $\varepsilon=\frac1{100}$ and take the projective directions and masses
\[
\begin{aligned}
&d_0=(-1,-1,0,0),\qquad d_1=e_1,\qquad d_2=e_2,\\
&d_3=(0,0,-1,-1),\qquad d_4=e_3,\qquad d_5=e_4,\\
&d_6=-\varepsilon(e_1+e_3),\qquad d_7=-\varepsilon(e_2+e_4),
\end{aligned}
\]
\[
(s_0,\dots,s_7)=\tfrac1{2700}\,(399,400,401,399,400,401,100,200).
\]
Then $\sum_is_id_i=0$ and $M=\sum_is_id_id_i^{\top}\succ0$, so
Lemma~\ref{lem:realize} realizes a whitened system in $\R^4$. It is a
rectangle system: its positive circuits are exactly
$\{0,1,2\}$, $\{3,4,5\}$, $\{1,4,6\}$, $\{2,5,7\}$, and no subset of at
most five directions positively spans. Its four circuit-plane energies are
$\frac49,\frac49,\frac13,\frac{167}{450}$, all below $\frac12$; and its
anchor costs satisfy
\[
\min_it_i=\tfrac{32003064046811}{144192302726400}\approx0.2219>\tfrac15,
\quad
\min_{i<j}t_{ij}\approx0.4439>\tfrac13,
\quad
\min_{i<j<k}t_{ijk}\approx1.1097>1 .
\]
Hence both branches of \eqref{eq:R} fail, and so do the single- and
triple-anchor variants. The system is nevertheless cheap for the selector:
the single point $0$ gives ratio about $1.2219$, and an explicit five-point
strictly convex objective on the support $\{0,4,5,6,7\}$ gives exact ratio
$\frac{2452861}{2000000}=1.2264\ldots<2$.
\end{proposition}

\begin{proof}
The mass identities $-399+400-1=0$ and $-399+401-2=0$, once per coordinate
pair, give $\sum_is_id_i=0$. Positive definiteness of the rational matrix
$M$ is a finite check. To determine the circuit structure, note that a nonnegative dependence
$\sum\lambda_id_i=0$ forces, coordinate by coordinate,
\[
\lambda_1=\lambda_0+\tfrac1{100}\lambda_6,\quad
\lambda_2=\lambda_0+\tfrac1{100}\lambda_7,\quad
\lambda_4=\lambda_3+\tfrac1{100}\lambda_6,\quad
\lambda_5=\lambda_3+\tfrac1{100}\lambda_7,
\]
so the dependence cone is generated by the four rays obtained by setting one
of $\lambda_0,\lambda_3,\lambda_6,\lambda_7$ to one and the others to zero.
These give the four listed circuits. As in
Proposition~\ref{prop:rectobs}, no five directions positively span. In the self-realized system $\eta_i^2=Ns_i$, so each plane energy is the
sum of the masses of its member indices:
$R(P_A)=\frac{399+400+401}{2700}=\frac49$, likewise $R(P_C)=\frac49$,
$R(P_{14})=\frac{400+400+100}{2700}=\frac13$ and
$R(P_{25})=\frac{401+401+200}{2700}=\frac{167}{450}$. The
anchor costs are the rational numbers
$t_S=\mathbf1^{\top}K_S^{-1}\mathbf1$ with
$K_S=(d_i^{\top}M^{-1}d_j)_{i,j\in S}$, evaluated in exact arithmetic over
all $8$ singletons, all $28$ pairs, and all $52$ triples with independent
features. The four circuit triples are linearly dependent, and their
interpolation constraints are infeasible: pairing
$\ip{q}{v_i}=1$ with a strictly positive zero-sum relation would equate a
positive sum with zero. Thus $t_S$ is defined as $+\infty$ there. The minima and
the comparisons with $\frac15,\frac13,1$ are exact comparisons of
fractions, reproduced by the verification scripts. The five-point
objective is the sign-cone realization of Lemma~\ref{lem:signcone} applied
to the signed relation $d_0=d_4+d_5+100\,d_6+100\,d_7$. Its stationarity,
positive weights, Hessian determinant and full-data risk are again exact
rational identities checked by the same scripts.
\end{proof}

The failure mode is instructive. The counterexample lies in a
near-degenerate boundary layer. The cross directions are smaller than the
block directions by a factor of $100$, yet carry constant residual mass.
Standard log-random parameter sweeps miss this layer. Our sweeps missed it over
hundreds of thousands of samples; analyzing the boundary led us to the
construction above. The cheap five-point selection through the
signed relation $d_0=d_4+d_5+100\,d_6+100\,d_7$ is exactly an instance of
Lemma~\ref{lem:signcone}. This observation motivated the extremal-basis
mechanism of Section~\ref{sec:fw45}.

\subsection{Cheap points need not be generated}

A final tempting shortcut is to hope that every rectangle system has a point of
interpolation cost $t_i\le1$ whose direction is strictly $4$-generated by
other rays. Lemma~\ref{lem:signcone} would then apply directly. This
hope also fails robustly.

\begin{proposition}[threshold family]\label{prop:c2fail}
For $\varepsilon>0$ define directions in $\R^4$
\[
\begin{aligned}
d_0&=-\varepsilon(e_1+e_2), & d_1&=e_1,\ d_2=e_2,\\
d_3&=-\varepsilon(e_3+e_4), & d_4&=e_3,\ d_5=e_4,\\
d_6&=-\varepsilon(e_1+e_3), & d_7&=-\varepsilon(e_2+e_4),
\end{aligned}
\]
with masses $s_0=s_3=s_6=s_7=\frac1{4(1+2\varepsilon)}$ and
$s_1=s_2=s_4=s_5=\frac{\varepsilon}{2(1+2\varepsilon)}$, realized as a
whitened system through Lemma~\ref{lem:realize}. Exactly the four points
$d_0,d_3,d_6,d_7$ are strictly $4$-generated, and the four side points
$d_1,d_2,d_4,d_5$ are not generated by the other rays at all. The exact
Mahalanobis norms are
\[
\norm{v}^2=\frac{2\varepsilon(3\varepsilon+2)}{\varepsilon+1}
\ \text{(generated points)},\qquad
\norm{v}^2=\frac{\varepsilon^2+4\varepsilon+2}{\varepsilon(\varepsilon+1)}
\ \text{(side points)} .
\]
For $0<\varepsilon<\frac{\sqrt{33}-3}{12}\approx0.2287$, all generated
points have $t=1/\norm{v}^2>1$ and all side points have $t<1$: no point is
simultaneously cheap and generated. At $\varepsilon=\frac1{100}$ the values
are exactly $t=\frac{5050}{203}$ and $t=\frac{101}{20401}$.
\end{proposition}

\begin{proof}
The masses are positive, sum to $1$, and satisfy $\sum_is_id_i=0$
coordinatewise. A nonnegative dependence $\sum\lambda_id_i=0$ forces
$\lambda_1=\varepsilon(\lambda_0+\lambda_6)$,
$\lambda_2=\varepsilon(\lambda_0+\lambda_7)$,
$\lambda_4=\varepsilon(\lambda_3+\lambda_6)$,
$\lambda_5=\varepsilon(\lambda_3+\lambda_7)$,
with $\lambda_0,\lambda_3,\lambda_6,\lambda_7$ free. Thus the dependence cone
is generated by the four triangle circuits
$\{0,1,2\},\{3,4,5\},\{6,1,4\},\{7,2,5\}$. The four strict generations are
the exact identities
\[
\begin{aligned}
d_0&=d_6+d_7+\varepsilon d_4+\varepsilon d_5, &
d_3&=d_6+d_7+\varepsilon d_1+\varepsilon d_2,\\
d_6&=d_0+d_3+\varepsilon d_2+\varepsilon d_5, &
d_7&=d_0+d_3+\varepsilon d_1+\varepsilon d_4,
\end{aligned}
\]
with independent generators. For the first,
$\det[d_6,d_7,d_4,d_5]=\varepsilon^2\ne0$, and similarly for the others. A
side point is not generated: if $d_1=\sum_{j\ne1}\kappa_jd_j$ with
$\kappa_j\ge0$, the dependence with $\lambda_1=1$ and
$\lambda_j=-\kappa_j$ contradicts
$\lambda_1=\varepsilon(\lambda_0+\lambda_6)\le0$. The Mahalanobis norms
$\norm{v_i}^2=d_i^{\top}M^{-1}d_i$ evaluate in exact arithmetic to the two
displayed closed forms. The threshold is the positive root of
$6\varepsilon^2+3\varepsilon-1=0$, namely
$\varepsilon=\frac{\sqrt{33}-3}{12}$. Below it, the generated points have
$\norm{v}^2<1$ and the side points $\norm{v}^2>1$. The exact values at
$\varepsilon=\frac1{100}$ follow by substitution. The verification scripts
also check all of these identities.
\end{proof}

The extremal-basis lemma sidesteps this obstruction because it does not fix
the anchor in advance. It first calibrates a basis on the side lines, then
anchors at the maximal evaluation, which is generated precisely when it
matters. The family above does not need this mechanism at all: the
triangle $\{0,1,2\}$ with weights $(\frac12,\frac14,\frac14)$ has gradient
sum zero, and the min-norm output over its rank-$2$ objective is $q=0$
exactly. Thus the family satisfies
$\Ls(5;\Astar)=\Ls$ and is no threat to any risk bound.

\subsection{Connectivity alone does not bound the excess}

The two exact values in dimensions $3$ and $4$ may suggest a clean
dichotomy. If the gradient circuits decompose orthogonally, the
block program of Section~\ref{sec:block} applies. If they are coupled, one
might hope for a coupling mechanism that yields excess at most $1$ with $d+1$ points. The
following rational system refutes the second half in dimension $6$.

\begin{proposition}[connected system with expensive budget
$d+1$]\label{prop:connected}
Let $d=6$ and take the eight gradients given by the columns of
\[
G=
\begin{pmatrix}
1&-1&0&0&0&0&0&0\\
0&1&-1&0&0&0&0&0\\
0&0&0&0&1&-1&0&0\\
0&0&0&0&0&1&-1&0\\
0&0&0&0&0&0&1&-1\\
1&1&1&-6&1&1&1&0
\end{pmatrix},
\]
\[
c=\tfrac1{10001}\,(16000,16000,16000,8,8000,8000,8000,8000),
\]
realized as the dataset $x_i=g_i/\sqrt{c_i}$, $y_i=\sqrt{c_i}$, which has
$w^{\star}=0$ and $\Ls=1$. The dependence cone of the eight gradients has
exactly two extreme circuits, $\{0,1,2,3\}$ of rank $3$ and
$\{3,4,5,6,7\}$ of rank $4$, sharing the point $3$; in particular the
circuit structure is connected and the system is not an orthogonal
circuit-block system. Its exact full-rank optima are
\[
\widetilde\Delta_6=\frac{192048003}{80008000}\approx2.4004,
\qquad
\widetilde\Delta_7=\frac{288030000}{240094007}\approx1.19966>1 .
\]
The value $\widetilde\Delta_7$ is attained by a seven-point strictly
positively weighted, strictly convex objective with rational output and
weights, verified by exact arithmetic.
\end{proposition}

The two values refer to strictly convex selections, the only kind that the
mechanisms of this paper produce. On the unshifted dataset written above the
tie rule would recover $w^{\star}=0$ for free. The circuit $\{0,1,2,3\}$
with weights $(1,1,1,\frac12)$ has gradient sum zero and a rank-$3$
Hessian, so its minimum-norm minimizer is $0$. As in
Proposition~\ref{prop:blocks}, a generic large shift
$y_i\mapsto y_i+M\ip{h}{x_i}$ leaves the excess of every full-rank
selection unchanged and makes the excess of every rank-deficient selection
grow like $M^2$, so for large $M$ the two values above are the true
optima $\Ls(6;\Astar)-\Ls$ and $\Ls(7;\Astar)-\Ls$.

This example has two implications. First, no mechanism that certifies excess at most $1$ with
$d+1$ points can cover every connected coupled configuration. A general
argument must overcome this obstruction as well as the earlier ones.
Block allocation, sign-cone calibration and circuit
completion are the three mechanisms known to us that do so. Second, the
system poses no threat to
Conjecture~\ref{conj:gamma}. Its features lie in the union of the two
circuit flats, so Theorem~\ref{thm:twocircuit} bounds its budget-$7$
excess by $C_{6,2}=\frac32=\Gam_{6,1}$. Indeed,
$\widetilde\Delta_7<\frac65\le\frac32$, where the constant $\frac65$ is the
two-flat bound $pq/(p+q)$ with $p=3,q=2$. A one-parameter mass family
that degenerates the shared point shows that this two-flat bound is asymptotically
sharp. Replacing the unnormalized squared residuals by
$(2,2,2,\varepsilon,1,1,1,1)$ gives, exactly,
\[
\widetilde\Delta_7(\varepsilon)
=\frac{6(5\varepsilon+48)}{(\varepsilon+10)(7\varepsilon+24)}
\ \longrightarrow\ \frac65,
\qquad
\widetilde\Delta_6(\varepsilon)
=\frac{3(\varepsilon+8)^2}{8(\varepsilon+10)}
\ \longrightarrow\ \frac{12}5
\qquad(\varepsilon\downarrow0).
\]
The budget-$7$ formula holds for every $\varepsilon>0$. The full-face
stationary point of the seven-point support $\{0,\dots,6\}$ has positive
barycentric coordinates for all $\varepsilon>0$, and no other seven-point
support is ever cheaper. The budget-$6$ formula holds for
$0<\varepsilon\le\varepsilon_0=\frac43(\sqrt{21}-3)\approx2.11$. Beyond
$\varepsilon_0$, the optimal six-point support changes and
$\widetilde\Delta_6(\varepsilon)=8(\varepsilon+8)/(\varepsilon(\varepsilon+10))$.
Both statements are verified symbolically over all $\varepsilon>0$ in the
supplementary material.

\section{Constructive selection algorithms}\label{sec:algo}

All upper bounds in this paper are constructive in the following sense.
Each procedure below works in exact real arithmetic over the input. Given
$\varepsilon>0$, it outputs a selection whose risk ratio
is at most the stated bound plus $\varepsilon$. The number of arithmetic
operations is polynomial in $N$ for the fixed dimensions concerned. The
$\varepsilon$ enters only in branches that use penalty limits or
sign-cone closures. These branches require a finite penalty parameter or a finite
interior point chosen as a function of $\varepsilon$. All other
branches output selections that attain their bounds exactly. We describe the
procedures below; none has been optimized for running time.

\paragraph{Budget $2d-1$.}
Given $D$, compute the feature span and $\Ls$. In the rank-deficient or
realizable cases, return the corresponding exact certificates. Otherwise
whiten and form the gradients. If the nonzero gradients span a proper
subspace, return the Steinitz certificate plus rank-completing zero-gradient
points. Otherwise repeatedly delete gradients while the remainder still
positively spans, using one linear program per deletion to test positive
spanning. The terminal set is a positive basis of size at most $2d$. If its
size is at most $2d-1$, solve for the strictly positive dependence and
return it. If its size is $2d$, scan the remaining gradients. Any gradient
off the two-per-line structure yields, by the construction inside
Lemma~\ref{lem:rig}, a positively spanning set of at most $2d-1$ gradients,
and we return its certificate. Next check for a point with
$\eta_i=0\ne\xi_i$. If one exists, return the certificate of the
zero-gradient branch of Theorem~\ref{thm:wend}: a canceling pair on each of
$d-1$ lines plus this point. Only then do all features lie on the $d$
lines. Compute the per-line residual energies, select the best single point
on the lightest line and canceling pairs elsewhere.

\paragraph{Budget $4$, dimension $3$.}
After the same preprocessing, check for an anchor point with $\eta_i=0\ne\xi_i$.
If one exists, run the $d=2$ procedure on the orthogonal complement and penalize
the anchor with a large finite weight. Otherwise extract a
minimum-cardinality positively spanning set. For size $4$, return the circuit.
For size $5$, identify the two extremal circuits of the dependence cone. In the
coupled case, classify the extra gradients by their minimal cone supports
and detect a four-point rank-$3$ circuit if present. Otherwise compute the at
most three circuit-plane energies, take the plane maximizing $R(P)$, and
pair its triangle with the point minimizing
$\eta_j^2/\ip{u}{\xi_j}^2$ off the plane. In the split case, first test every nonzero gradient for membership in
$V\cup L$. A gradient outside triggers the reduction of
Theorem~\ref{thm:fw34} to an exact four-point certificate or to the coupled
case. Only when all gradients lie in $V\cup L$ compare
$R_L$ against $2R_V$ and return the better of the two selections of
Proposition~\ref{prop:split}. For size $6$, all nonzero gradients lie
on the three lines of the basis, since a gradient off the lines would
yield a smaller positively spanning set. Select a canceling pair on the
heaviest line and best single points elsewhere.

\paragraph{Budget $5$, dimension $4$.}
The branch structure of Theorem~\ref{thm:fw45} is algorithmic in the same
way. The new ingredients are the following. A minimum-cardinality
positively spanning subset is found by enumerating gradient subsets of size
at most $2d$ in increasing size and testing positive spanning by one linear
program each. For fixed dimension, this takes polynomially many checks in $N$.
A support-minimal cone representation is obtained from any feasible
nonnegative representation by repeatedly deleting an element whose removal
preserves feasibility. For fixed dimension, one can instead enumerate candidate
supports of size at most $4$ directly. An arbitrary feasible solution of
the linear program need not have minimal support, yet the choice of branch
depends on support size. In the two-triangle case, every
extra gradient is first classified by its support size, returning the
five-point exact certificate at size $4$ and the completion selection at
size $3$ before any cross analysis. The size-$6$ and size-$7$ bases are
classified through the extreme rays of their dependence cones. In the
rectangle branch, the extremal basis is computed by fixing a reference
direction on each side line, writing each point on the line as
$v=\theta\cdot(\text{reference})$, taking the point of maximal absolute
scalar $|\theta|$ on each of the four side lines, and solving one
$4\times4$ linear system for $q_B$. If $\norm{q_B}\le1$, return the
four-point interpolation. Otherwise anchor at the maximal evaluation
$|z_i|$ and read off its strict generation $v_i=\sum_j\kappa_jw_j$. If
$K=\sum_j\kappa_j=1$, return the five-point interpolation at
$q_\ast$. If $K>1$, pick the interior point $q_c$ with side values
$c\in(1/K,1)$, choose $t>0$ small enough that the risk of
$(1-t)q_\ast+t\,q_c$ is within the prescribed $\varepsilon$, and return the
sign-cone weights of Lemma~\ref{lem:signcone} at that point.

\paragraph{Block systems.}
For orthogonal circuit-block systems with known blocks, the allocation
minimizing the bound $\min_{|K|=k}\sum_{j\notin K}r_jR_j$ of
Theorem~\ref{thm:lp} is computed by sorting the products $r_jR_j$ and
completing the $k$ largest blocks. This optimizes the guarantee, not the
excess of a given instance. A block may contain a point far cheaper than
its bound $r_jR_j$, in which case completing a different block is better.
Instance-level optimality requires the actual component budget--cost
profiles rather than the energies $r_jR_j$.

\section{Extended discussion}\label{sec:discussion}

This appendix expands the discussion of Conjecture~\ref{conj:gamma}.

\paragraph{Status of the conjecture.}
The conjecture holds at $k=d-1$ (Corollary~\ref{cor:endpoint}), at
$(d,k)=(3,1)$ (Theorem~\ref{thm:fw34}), at $(d,k)=(4,1)$
(Theorem~\ref{thm:fw45}) and at $(d,k)=(4,2)$ (Theorem~\ref{thm:fw46}). Its value is also the exact minimax answer on the
class of orthogonal circuit-block systems (Corollary~\ref{cor:lp-gamma}).
For auxiliary evidence, we ran randomized stress tests in dimensions two and
three. We also exhaustively enumerated supports and used numerical and randomized
weight search on structured instances in dimensions four and five.
These tests, provided as seeded scripts in the supplementary material, produced no instance
exceeding $1+\Gam_{d,k}$. Further exploratory searches in higher
dimensions found none either, but are not packaged reproducibly. Per-weight solves
are floating-point least squares and the weight search is not a global
optimization. We therefore treat all these searches as diagnostics, not verification. The smallest open cell is $(d,n)=(5,6)$, with conjectured
value $\frac{11}5$.

\paragraph{What remains for a general upper bound.}
The proofs of the four exact cells share a common structure. The failures
recorded in Section~\ref{sec:obstructions} clarify its limits. Away
from degenerate branches, the nonzero gradients contain a minimal positive
basis of size $d+m$. Budgets with $d+k\ge|B|$ recover exactly, so the hard
range is $m\ge k+1$. This mirrors the range of $s$ in $\Gam_{d,k}$. Within
the hard range three mechanisms operate: orthogonal block allocation
(Theorem~\ref{thm:lp}), sign-cone calibration
(Lemma~\ref{lem:extremal}), and circuit--quotient completion
(Lemmas~\ref{lem:completion} and~\ref{lem:lift},
Theorem~\ref{thm:twocircuit}). Proposition~\ref{prop:connected} shows the
second mechanism alone cannot cover all coupled configurations, so a
general argument must overcome all of these obstructions. The three
mechanisms are the ones known to us that do so.

\paragraph{A profile program.}
Theorem~\ref{thm:alloc} reduces direct sums to an allocation program over
component budget--cost profiles
$\delta_j(d_j),\delta_j(d_j+1),\dots$. This identifies the precise
quantitative statement still needed. For example, for the cell
$(5,6)$, composing a two-dimensional simplex block with an arbitrary
three-dimensional component of profile $(x,y)=(\widetilde\Delta_3,
\widetilde\Delta_4)$ yields the adversarial value $2x/(x+2-y)$. On this
two-plus-three composition subclass, the bound by $\Gam_{5,1}=\frac65$ is
equivalent to the facet inequality
\[
2\widetilde\Delta_3+3\widetilde\Delta_4\ \le\ 6
\]
for every energy-normalized three-dimensional whitened system. Simplex
circuits attain equality. Our searches found nothing above
$4.27$ among coupled components, and Lemma~\ref{lem:adjacent} makes the
left side a finite geometric quantity. In our view, proving this inequality and its
higher-dimensional analogues, together with quantifying the completion mechanism on
cross ears, forms the remaining core of
Conjecture~\ref{conj:gamma}. Settling the cell $(5,6)$ itself additionally
requires a reduction of arbitrary five-dimensional systems to such
compositions, which we do not have. Whether profile domination by circuit-block
mixtures holds for every system is precisely the sufficient form of the
conjecture suggested by all evidence so far.

\paragraph{A heuristic on the extremal-basis hypothesis.}
The extremal-basis lemma (Lemma~\ref{lem:extremal}) is dimension-free. As a
heuristic, its hypothesis is obstructed on simplex-block extremals with a
block of dimension at least two. Calibrating a basis inside the blocks
leaves a missed vertex whose evaluation has magnitude equal to its block
dimension, and there are no cross ears to generate it. On all-one-dimensional
blocks the hypothesis can hold vacuously, and the lemma then gives the
harmless bound of excess $1$. Making this interplay precise is part of the
program above.

\paragraph{Further questions.}
Three directions seem natural. First, the unweighted and vector-valued variants
of the problem (Questions~3 and~4 of \citealp{HMSYopen}) interact with the
present geometry through the same gradient structure, and it would be
interesting to determine which parts transfer. Second, our selections are
constructive but not optimized. The complexity of computing an optimal
weighted selection for a given dataset is open. Third, the interplay between the normal forms of small positive bases
\citep{CKR2024} and the control of the vectors outside the basis developed
here may be of use in derivative-free optimization, where positive spanning
sets are a basic primitive.

\section{Discussion continued}\label{sec:extended-outlook}

The endpoint line and the three exact cells reveal a common pattern. Away
from degenerate branches, the nonzero gradients contain a minimal positive
basis of size $d+m$. Budgets with $d+k\ge|B|$ give exact recovery. The hard
range is therefore $m\ge k+1$, mirroring the range of $s$ in $\Gam_{d,k}$.
Three mechanisms operate in this range:
orthogonal block allocation (Theorem~\ref{thm:main-block}), sign-cone
calibration (Lemma~\ref{lem:extremal-main}) and circuit--quotient completion
(Lemmas~\ref{lem:completion-main} and~\ref{lem:lift},
Theorem~\ref{thm:twocircuit}). Proposition~\ref{prop:connected} shows
that the second alone cannot cover all coupled configurations. A
domination principle stating that every gradient system is dominated by a
circuit-block system that is no better for the selector would prove
Conjecture~\ref{conj:gamma} in full. The composition law
identifies the missing quantitative statement: for the cell $(5,6)$, a
two-dimensional simplex block composed with a three-dimensional component of
profile $(\widetilde\Delta_3,\widetilde\Delta_4)$ has adversarial value
$2\widetilde\Delta_3/(\widetilde\Delta_3+2-\widetilde\Delta_4)$. On this
two-plus-three composition subclass, the bound by $\Gam_{5,1}=\frac65$ is
equivalent to the facet inequality
$2\widetilde\Delta_3+3\widetilde\Delta_4\le6$ for every energy-normalized
three-dimensional whitened system, with equality on simplex circuits.
Lemma~\ref{lem:adjacent} makes the left side a finite geometric quantity.
Proving this inequality and its higher-dimensional analogues, and
quantifying completion on cross ears, is in our view the remaining core of
the conjecture. Further directions are listed in Appendix~\ref{sec:discussion}.

\end{document}